\documentclass{article}

\usepackage[left=1in, right=1in, top=1in, bottom=1in]{geometry}

\usepackage[utf8]{inputenc}                     % allow utf-8 input
\usepackage{beton}  % removed: causes T1/ccr/b/n font shape undefined warnings with \bm{}

\usepackage{amsmath}
\usepackage{amsfonts}
\usepackage{amssymb}
\usepackage{bm}
\usepackage{nicefrac}       % compact symbols for 1/2, etc.
\usepackage{microtype}      % microtypography
\usepackage[T1]{fontenc}
\usepackage[dvipsnames]{xcolor}
\usepackage[
  pagebackref=true,
  colorlinks=true,
  citecolor=MidnightBlue,
  linkcolor=PineGreen,
  urlcolor=Sepia,
]{hyperref}
\hypersetup{
  pdftitle={Leaner Transformers Can Easily Learn to Cluster},
  pdfauthor={Charlotte Park, Kenneth L. Clarkson, Lior Horesh, Takuya Ito, Parikshit Ram},
  pdfsubject={},
  pdfkeywords={transformer circuits, clustering, in-context learning}
}

\usepackage{mathtools}% http://ctan.org/pkg/mathtools
\usepackage[ruled,vlined,linesnumbered]{algorithm2e}
\usepackage{xspace}
\usepackage{booktabs}
\usepackage[capitalize,noabbrev]{cleveref}
\usepackage{natbib}
\usepackage{amsthm}
\usepackage{mdframed}
\usepackage[many]{tcolorbox}
\usepackage{xkcdcolors}

\newcounter{hyperrefeqctr}

\AtBeginDocument{%
  \let\OldRefstepcounterEquation\refstepcounter
  \renewcommand{\refstepcounter}[1]{%
    \OldRefstepcounterEquation{#1}%
    \ifstrequal{#1}{equation}{\stepcounter{hyperrefeqctr}}{}%
  }%
}

\theoremstyle{plain}
\newtheorem{theorem}{Theorem}[section]
\newtheorem{lemma}{Lemma}[section]

\newtheorem{proposition}{Proposition}[section]
\theoremstyle{definition}
\newtheorem{definition}{Definition}[section]
\newtheorem{assumption}{Assumption}[section]
\theoremstyle{remark}

\usepackage{enumitem}
\setlist[itemize]{noitemsep, topsep=0pt, leftmargin=*}
\setlist[enumerate]{noitemsep, topsep=0pt, leftmargin=*}

\usepackage{graphicx}
\usepackage{subcaption}
\usepackage{wrapfig}
\usepackage[normalem]{ulem}

\newcommand{\R}{\mathbb{R}}
\newcommand{\E}{\mathbb{E}}
\newcommand{\N}{\mathbb{N}}
\newcommand{\Ls}{\mathcal{L}}
\def\ceil#1{\lceil #1 \rceil}
\newcommand{\iset}[1]{[ #1 ]}

\def\gD{{\mathcal{D}}}
\def\gS{{\mathcal{S}}}
\def\gT{{\mathcal{T}}}
\def\gU{{\mathcal{U}}}
\def\gX{{\mathcal{X}}}
\def\gY{{\mathcal{Y}}}

\def\vb{{\bm{b}}}
\def\vc{{\bm{c}}}
\def\ve{{\bm{e}}}
\def\vl{{\bm{l}}}
\def\vlambda{{\bm{\lambda}}}
\def\vp{{\bm{p}}}
\def\vr{{\bm{r}}}
\def\vs{{\bm{s}}}
\def\vv{{\bm{v}}}
\def\vx{{\bm{x}}}
\def\vz{{\bm{z}}}
\def\rvx{{\mathbf{x}}}

\def\mA{{\bm{A}}}
\def\mB{{\bm{B}}}
\def\mC{{\bm{C}}}
\def\mI{{\bm{I}}}
\def\mK{{\bm{K}}}
\def\mQ{{\bm{Q}}}
\def\mV{{\bm{V}}}
\def\mW{{\bm{W}}}
\def\mX{{\bm{X}}}
\def\mY{{\bm{Y}}}
\def\mZ{{\bm{Z}}}
\def\tZ{{\mathcal{Z}}}

\newcommand{\softmax}{\mathrm{softmax}}
\newcommand{\hardmax}{\mathrm{hardmax}}
\newcommand{\ahmax}{\mathrm{avghardmax}}
\newcommand{\sqdist}{\mathrm{sqdist}}

\DeclareMathOperator*{\argmin}{arg\,min}

\newcommand{\norm}[1]{\left\Vert#1\right\Vert}
\newcommand{\inpr}[2]{\left \langle #1, #2 \right\rangle}
\newcommand{\eudi}[2]{\left\| #1 - #2 \right\|^2}

\newcommand{\demb}{{d_\mathsf{emb}}}

\def\attned#1{{{\mathcal{A}_2^{#1}}}}

\def\cmodel{F^{\gamma}_{\!\!d:k}}

\def\eohc{{{E_C^{\mathtt{OH}}}}}
\def\eohx{{{E_X^{\mathtt{OH}}}}}
\def\ebnc{{{E_C^{\mathtt{BN}}}}}
\def\ebnx{{{E_X^{\mathtt{BN}}}}}

\newcommand{\oh}{{\texttt{OH}}\xspace}
\newcommand{\bn}{{\texttt{BN}}\xspace}
\newcommand{\na}{{\texttt{NA}}\xspace}
\newcommand{\up}{$\blacktriangle$}
\newcommand{\down}{$\blacktriangledown$}

\newcommand{\qa}[2]{\left[ \begin{matrix} #1 \\ #2 \end{matrix} \right]}

\newcommand{\highlightmath}[2]{\colorbox{#1!50}{$\displaystyle#2$}}

\usepackage[toc,page,header]{appendix}
\usepackage{minitoc}

\title{Leaner Transformers Can Easily Learn to Cluster}

\author{%
Charlotte Park${}^\ddag$,
Kenneth L. Clarkson${}^\dag$,
Lior Horesh${}^\dag$,
Takuya Ito${}^\dag$,
Parikshit Ram${}^\dag$.
\\
{\normalsize
${}^\dag${\sf IBM Research},
${}^\ddag${\sf MIT}.
}
\qquad \qquad
{\small
Email: \texttt{parikshit.ram@ibm.com}
}}
\date{}

\begin{document}

\begin{mdframed}[
  linecolor=ForestGreen!30!black,
  backgroundcolor=ForestGreen!10,
  topline=false,
  rightline=false,
  bottomline=false,
  leftline=false,
  innertopmargin=-20pt,
  splittopskip=0pt,
  innerbottommargin=20pt,
]
\maketitle

\begin{abstract}
Transformers have in-context learning capabilities, where some known learning algorithms can be executed in the forward pass through the model. Recent work shows that transformers can exactly perform Lloyd's algorithm for $k$-means clustering with $n$ points in $d$ dimensions with an embedding size $\demb = d+k$ (thus, requiring attention projection matrices of size $(d+k)^2$). In this work, we build upon this result in the following ways: First, we present an equally expressive but smaller transformer that executes Lloyd's algorithm with embedding size $\demb = (d + \ceil{\log_2 k})$.
Next, we train these transformers to learn the clustering algorithms given a distribution of clustering tasks, and theoretically characterize and empirically validate the factors affecting the convergence and in-distribution generalization of learning algorithms based on stochastic gradients.
Finally, we probe the general clustering abilities of these learned algorithms (in the form of transformers), and try to understand situations where they succeed and fail.
\end{abstract}
\end{mdframed}
\doparttoc % Tell to minitoc to generate a toc for the parts
\faketableofcontents % Run a fake tableofcontents command for the partocs
%
%\part{} % Start the document part
%\parttoc % Insert the document TOC
%
\section{Transformers Can Execute Clustering Algorithms In-Context}
\label{sec:intro}
Clustering is a well-studied problem in computer science, having various applications across many fields, and possessing a nontrivial level of hardness~\citep{hartigan1975clustering, aggarwal2013data}. Within machine learning, clustering can help with tasks such as exploratory data analysis, unsupervised labeling, data compression, and quantization.
Given a set of items, such as points in Euclidean space, the goal of $k$-clustering is to partition the items into~$k$ groups.
Clustering is categorized through various axes such as discrete vs fuzzy or flat vs hierarchical, and have application-dependent clustering objectives such as
$k$-means, discrete~\citep{macqueen1967some, lloyd1982least} or soft~\citep{dunn1974fuzzy, bezdek2013pattern}; $k$-medoids~\citep{kaufman1990partitioning, ng2002clarans}; spectral~\citep{von2007tutorial}; kernelized~\citep{scholkopf1998nonlinear, girolami2002mercer, dhillon2004kernel}; and variants.

Here we focus on discrete (flat) Euclidean $k$-means clustering. Given a set of points $\gS = \{ \vx_1, \ldots, \vx_n \}$ in a $d$-dimensional Euclidean space, with $\vx_i \in \R^d$ for $i \in \iset{n} \triangleq \{1, 2, \ldots, n\}$, the goal is to find $k$ centroids/prototypes $\vc_j \in \R^d, j \in \iset{k}$ by solving the following problem:
\begin{equation} \label{eq:kmeans}
\min_{\{\vc_j \in \R^d,  j \in \iset{k} \}}
\sum_{i = 1}^n \min_{j \in \iset{k} } \eudi{\vx_i}{ \vc_j }.
\end{equation}
This has a discrete non-differentiable objective function, and so is inherently combinatorial~\citep{macqueen1967some, lloyd1982least}. This clustering problem is known to be NP-hard even with $k=2$, which highlights its non-trivial hardness~\citep{dasgupta2008hardness}, and prompts the development of various clustering algorithms~\citep{hartigan1975clustering, aggarwal2013data}.

Beyond learning predictive models from data (like a decision tree ensemble to predict credit default or a deep convolutional neural network to label an image), an area of interest in machine learning is the ability to learn algorithms from data.
In this line of work, attention-based transformers~\citep{vaswani2017attention, phuong2022formal} and recurrent neural networks~\citep{graves2014neural, andrychowicz2016learning} have garnered much interest.
With transformers, the ability to execute an algorithm is often studied as {\em in-context learning}, where a transformer is shown to be able to execute a specific learning algorithm within its forward pass~\citep{garg2022can, akyurek2023what, von2023transformers, li2023transformers}.
Given this ability to express a specific (learning) algorithm in-context --- thus demonstrating {\em algorithmic expressivity} --- an important question is how and when such algorithms can be learned~\citep{von2023transformers, li2023transformers, ahn2023transformers, huang2023context, zhang2023trained, cheng2024transformers}, and what kind of {\em algorithmic generalization} can be expected~\citep{li2023transformers, ito2025quantifying}.
Recently, \citet{clarkson2026transformer} established algorithmic expressivity of a transformer with a model embedding dimensionality $\demb = d + k$ for Euclidean clustering, demonstrating that this architecture can exactly execute Lloyd's algorithm for $k$-means~\citep{lloyd1982least}, and some variants for fuzzy~\citep{dunn1974fuzzy, bezdek2013pattern} and robust clustering~\citep{cuesta1997trimmed}.
\paragraph{Contributions.}
We build upon \citet{clarkson2026transformer} to make the following contributions:
\begin{itemize}
\item We present a novel transformer architecture that exactly executes Lloyd's algorithm for $k$-means with a reduced model embedding size $\demb = d+\ceil{\log_2 k}$, an up to 75\% reduction\footnotemark in the number of transformer parameters when $k \approx d$, and much larger if $k \gg d$.
\footnotetext{The number of parameters in a single-headed attention-only transformer is often $\demb^2$ for each of the query, key, value projection matrices. When $k \approx d$ and $\demb = (d+k)$, these matrices are of size $(2d)^2$, while $\demb = (d + \log_2 k)$ leads to size $(d+ \log_2 d)^2$. As $d$ grows, $\nicefrac{(d+\log_2 d)^2}{4 d^2} \to 1/4 = 0.25$.}
\item We study a computationally cheap algorithm to train these transformers to learn to cluster in an end-to-end differentiable manner given a distribution of clustering tasks, and theoretically characterize the various factors affecting its convergence and in-distribution generalization. Specific factors of interest here are the token embedding scheme in the transformer, and the smoothing procedure to optimize the discrete objective in \cref{eq:kmeans}.
\end{itemize}
One insight from this analysis is that, for a fixed number of clusters $k$, and a relatively small data dimensionality $d$, our proposed smaller transformer with $\demb = d + \ceil{ \log_2 k }$ takes longer to converge and generalize compared to the previous transformer with $\demb = d+k$. However, both eventually converge to the same performance, and this difference in convergence rate vanishes as $d$ grows.
\paragraph{Notation.}
We denote the index set as $\iset{ m } \triangleq \{1, 2, \ldots, m \}$ for a positive integer $m \in \N$. We use lowercase letters (such as $x$) to denote scalars and uppercase letters (such as $L$) to denote functions, boldface lowercase letters ($\vx$) to denote vectors, boldface uppercase ($\mX$) to denote matrices, and calligraphic letters ($\gX$) to denote sets. The $i^{\text{th}}$ row of a matrix $\mX$ is $\mX_{i,:}$, the $j^{\text{th}}$ column is $\mX_{:,j}$, the first $i$ rows of $\mX$ is $\mX_{:i,:}$. $\mI_k$ is the $(k \times k)$ identity matrix.
$\norm{\cdot}$ denotes the $\ell_2$ norm for vectors and the spectral norm for matrices.
For a set of matrices $\theta = \{ \mW_1, \mW_2, \ldots, \mW_m \}$ (such as model parameter sets), $\norm{\theta} \triangleq \max_{i \in \iset{ m }} \norm{ \mW_i }$. For any function $F: \gX \to \gY$ with learnable parameters $\theta$, we write it as $F(\rvx; \theta)$ for any $\rvx \in \gX$ to explicitly note the dependence on the parameters $\theta$.
\section{Expressing a Clustering Algorithm with a Leaner Transformer}
\label{sec:expressivity}
The existing algorithmic expressivity of transformers exactly emulating Lloyd's algorithm~\citep{clarkson2026transformer} uses three critical components: (i)~a procedure to embed the points (to be clustered) and the initial centers (to be refined), (ii)~a limiting version of softmax attention with the softmax temperature tending to zero, making the softmax behave like argmax, and (iii)~an update based on cross-attention and self-attention, first for the token embeddings of the points, and then for the token embeddings for the cluster centers. For a set of points $\gS = \{\vx_i, i \in \iset{ n } \} \subset \R^d$, represented as $\mX \in \R^{d \times n}$, and $k$ initial cluster centers $\{\vc_j, j \in \iset{ k } \}$, represented as $\mC \in \R^{d \times k}$, $\mX$ and $\mC$ are embedded in $(d+k)$ dimensions with $\eohx: \R^{d \times n} \to \R^{(d+k) \times n}$ and $\eohc: \R^{d \times k} \to \R^{(d+k) \times k}$ as follows:
\begin{equation} \label{eq:tk-emb-v1} \tag{\oh}
\bar{\mC}^{(0)}
\gets \eohc(\mC) \triangleq
\qa{\mC}{\mI_k},
\quad
\bar{\mX}^{(0)}
\gets  \eohx(\mX) \triangleq
\qa{\mX}{\mY^{(0)}},
\quad
\mY^{(0)} = \mathbf{0}_{k \times n} \in \{0, 1\}^{k \times n},
\end{equation}
where the \oh denotes the one-hot encoding of the cluster indices with $\mI_k$. This makes the transformer embedding dimension $\demb = d + k$.
Beginning with $\mY^{(0)} = \mathbf{0}_{k \times n}$, the all zero $(k \times n )$ matrix, the first set of cross-attention and self-attention updates $\bar{\mX}^{(t)} \to \bar{\mX}^{(t+1)}$ by modifying the $\mY^{(t)}$ part of the embeddings to match the current cluster assignments via Lloyd's algorithm, that is, each point's embedding is updated to contain the one-hot encoded index of the closest center. The next set of cross and self-attention updates $\bar{\mC}^{(t)} \to \bar{\mC}^{(t+1)}$ by modifying the $\mC^{(t)}$ part of the embeddings:
\begin{equation} \label{eq:tk-emb-up-v1}
\bar{\mX}^{(t)} = \qa{\mX}{\mY^{(t)}} \to \bar{\mX}^{(t+1)} = \qa{\mX}{\mY^{(t+1)}},
\quad
\bar{\mC}^{(t)} = \qa{\mC^{(t)}}{\mI_k} \to \bar{\mC}^{(t+1)} = \qa{\mC^{(t+1)}}{\mI_k}.
\quad
\end{equation}
With the $\demb=d+k$, the query/key/value projection matrices for the cross and self-attention are of $O(\demb^2) \sim O((d+k)^2)$ size. To mitigate the quadratic dependence on $k$, we propose the following:
\begin{definition}[Token embeddings] \label{def:tk-emb-v2}
For points $\mX \in \R^{d\times n}$ and initial centers $\mC \in \R^{d \times k}$, we use the embedding scheme
$\ebnx: \R^{d \times n} \to \R^{(d+\ceil{\log_2 k}) \times n}$ and $\ebnc: \R^{d \times k} \to \R^{(d+\ceil{\log_2 k}) \times k}$:
\begin{equation} \label{eq:tk-emb-v2} \tag{\bn}
\bar{\mC}^{(0)}
\gets \ebnc (\mC) \triangleq
\qa{\mC}{\mB_k},
\quad
\bar{\mX}^{(0)}
\gets \ebnx (\mX) \triangleq
\qa{\mX}{\mY^{(0)}},
\quad \mY^{(0)} = \mathbf{0}_{\ceil{\log_2 k} \times n},
\end{equation}
where $\mB_k = [\vb_1; \vb_2; \ldots; \vb_k] \in \{0, 1\}^{\ceil{\log_2 k} \times k}$ is a binary matrix with the $i^{\text{th}}$ column $\vb_i \in \{0, 1\}^{\ceil{\log_2 k}}$ as the binary representation of the integer $(i-1)$ with $\ceil{\log_2 k}$ bits for each $ i \in \iset{ k }$.
\end{definition}
The main intuition behind this \bn (binary) embedding scheme is that, while the \eqref{eq:tk-emb-v1} token embedding~\citep{clarkson2026transformer} uses one-hot vectors to represent a cluster index (thus requiring $k$ dimensions), we represent that index more compactly with a binary representation using only $\ceil{\log_2 k}$ dimensions. As an example, when $k = 6$, $\mB_6 = $
{\tiny $\qa{0, 1,  0,  1,  0,  1 \\ 0,  0,  1,  1,  0, 0}{0,  0,  0,  0,  1,  1 }$}
$\in \{0, 1\}^{3 \times 6}$.
Then we use the following softmax attention with negative squared Euclidean distances:
\begin{definition}[Attention] \label{def:attn-op}
Given queries $\mZ \in \R^{\demb \times q}$ and keys $\mZ' \in \R^{\demb \times r}$, we define
attention as:
\begin{equation} \label{eq:attn-op}
\attned{\gamma} (\mZ, \mZ'; (\mQ, \mK, \mV)) =  \mV \mZ' \, \softmax \left( -\frac{\gamma}{\sqrt{\demb}} \, \sqdist( \mQ \mZ, \mK \mZ') \right),
\end{equation}
where $\mQ, \mK, \mV \in \R^{\demb \times \demb}$ are the query, key and value projection matrices respectively, $\gamma$ is the inverse temperature of the columnwise $\softmax$ operation and $\sqdist: \R^{\demb \times q} \times \R^{\demb \times r} \to \R^{r \times q}$ computes the pairwise squared Euclidean distances
with the $(i,j)$-th entry of its output given by $\sqdist(\tZ, \tZ')_{ij} = \eudi{\tZ_{:,j}}{\tZ'_{:,i}}$ for some matrices $\tZ \in \R^{\demb \times q}, \tZ' \in \R^{\demb \times r}$.
\end{definition}
This distance-based attention~\citep{tsai2019transformer} was used to establish algorithmic expressivity for clustering, though this can be replaced with dot-product attention if we also utilize a single-hidden-layer feedforward network in the transformer~\citep{clarkson2026transformer}. Based on the above, we consider the following transformer updates:
\begin{definition}[Per-layer updates] \label{def:updates}
In the $(t+1)^\text{th}$ transformer layer, with parameters $\Theta = \{ (\mQ_\mu, \mK_\mu, \mV_\mu), \mu = 1, \ldots, 4 \}$ shared across all layers, the layer input token embeddings $( \bar{\mX}^{(t)}, \bar{\mC}^{(t)} )$ are updated with the attention operation in \Cref{def:attn-op} as follows:
\begin{equation} \label{eq:updates}
\begin{split}
&
\bar{\mX}^{(t+1)} =
\colorbox{gray!40}{\!$
\bar{\mX}^{(t)} +
$\!}\ \
\colorbox{RawSienna!40}{\!$
\attned{\gamma} (\bar{\mX}^{(t)}, \bar{\mC}^{(t)}; (\mQ_1, \mK_1, \mV_1))
$\!}
+
\colorbox{MidnightBlue!40}{\!$
\attned{\gamma} (\bar{\mX}^{(t)}, \bar{\mX}^{(t)}; (\mQ_2, \mK_2, \mV_2))
$\!},
\\
&
\bar{\mC}^{(t+1)} =
\colorbox{gray!40}{\!$
\bar{\mC}^{(t)} +
$\!}\ \
\colorbox{RawSienna!40}{\!$
\attned{\gamma} (\bar{\mC}^{(t)}, \bar{\mX}^{(t+1)}; (\mQ_3, \mK_3, \mV_3))
$\!}
+
\colorbox{MidnightBlue!40}{\!$
\attned{\gamma} (\bar{\mC}^{(t)}, \bar{\mC}^{(t)}; (\mQ_4, \mK_4, \mV_4))
$\!},
\end{split}
\end{equation}
where the first operation on the right-hand-side after the \colorbox{gray!30}{\!residual connection\!} in each of the update steps corresponds to \colorbox{RawSienna!30}{\!cross-attention\!}, while the second operation is \colorbox{MidnightBlue!30}{\!self-attention\!}.
\end{definition}
Given this architecture, we can show the following algorithmic expressivity result:
\begin{theorem}\label{thm:lloyds-exp}
Given a set of points $\mX \in \R^{d \times n}$ to be clustered into $k$ disjoint sets and initial cluster centers $\mC \in \R^{d \times k}$, both embedded with $(\ebnc, \ebnx)$ (\Cref{def:tk-emb-v2}) to give us a latent dimension $\demb = d + \ceil{\log_2 k}$, there exists parameters $\Theta = \{(\mQ_\mu, \mK_\mu, \mV_\mu), \mu = 1, \ldots, 4\}$ for the transformer (\Cref{def:updates}) with the inverse softmax temperature $\gamma = \infty$ such that the $t^{\text{th}}$ transformer layer output exactly matches the output of Lloyd's algorithm after $t$ iterations for any $t \geq 1$.
\end{theorem}

\begin{figure}[t]
\centering
\includegraphics[clip, trim=0in 0in 0in 0.8in, width=\textwidth]{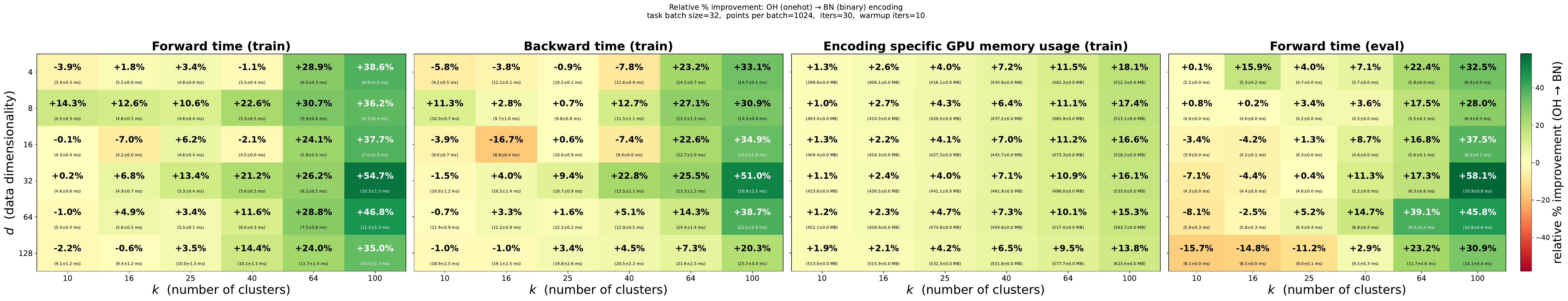}
\caption{Comparison of the proposed \eqref{eq:tk-emb-v2} embedding against the existing \eqref{eq:tk-emb-v1} embedding with varying values of $d \in \{4, 8, 16, 32, 64, 128\}$ and $k \in \{10, 16, 25, 40, 64, 100 \}$. Each heatmap reports the relative improvement (in \%) of \eqref{eq:tk-emb-v2} over \eqref{eq:tk-emb-v1} for each of following metrics: (i)~{\em Forward time (train)}: forward pass time during training, (ii)~{\em Backward time (train)}: backward pass time, (iv)~{\em Encoding specific GPU memory usage}: peak GPU memory usage at training (less the usage from the self-attention matrices which does not depend on $\demb$), and (iv)~{\em Forward pass (eval)}: forward pass time with the computation graph disabled (via \texttt{model.eval()}). Each cell in the $(6 \times 6)$ grid in the heatmap lists the relative gain (or degradation), with actual metrics for \eqref{eq:tk-emb-v1} also listed in the cell as (mean$\pm$std) for completeness; darker greens denote larger gain values. All runs use 32 clustering problems with $n = 1024$ points, and numbers are generated with 30 repetitions (10 warm-up rounds). In terms of memory usage, \eqref{eq:tk-emb-v2} shows gains across the board, ranging from +1-15\%. For the computation times (forward (train), backward, and forward (eval)), \eqref{eq:tk-emb-v2} does not always show a positive gain for smaller values of $k$, especially when the baseline \eqref{eq:tk-emb-v1} runtimes is already quite small ($<10$ms). However, once $k$ is large enough, the gains can be quite significant, reaching over 50\% in some cases.}
\label{fig:enc-profile}
\end{figure}

\begin{table}[!!!b]
\centering
\caption{
  Comparison between proposed \eqref{eq:tk-emb-v2} and existing \eqref{eq:tk-emb-v1} under the same measurement setting as in \cref{fig:enc-profile} but where $k$ is varied in the range $[6,1000]$ (column 1) for a fixed data dimensionality $d=32$. We note the actual forward (column 3) and backward pass (column 4) times (and the speedup of \eqref{eq:tk-emb-v2} over \eqref{eq:tk-emb-v1}), and the peak GPU memory usage (column 5) as described in \cref{fig:enc-profile} (and the reduction from using \eqref{eq:tk-emb-v2}). We also note the additional embedding dimensions $(\demb - d)$ for each value of $k$ for the two embeddings (column 2), and the significant reduction in the additional embedding dimensions lead to the significant reductions in the runtimes once $k$ is sufficiently larger than $d$ with up to 12$\times$/9$\times$ speedup in forward/backward pass, and almost 70\% reduction in memory usage.
}
\label{tab:enc-profile}
\begin{tabular}{r c c c c}
\toprule
$k$
  & $(\demb - d)$
  & Forward pass (ms)
  & Backward pass (ms)
  & GPU memory usage (MB) \\
  & OH\,/\,BN & OH\,/\,BN [speedup] & OH\,/\,BN [speedup] & OH\,/\,BN [reduction] \\
\midrule
6
  & 6\,/\,3
  & 4.25\,/\,4.83     [0.88$\times$]
  & 9.53\,/\,10.66    [0.89$\times$]
  & 419.93\,/\,416.29 [0.9\%] \\
25
  & 25\,/\,5
  & 5.24\,/\,4.99     [1.05$\times$]
  & 11.21\,/\,11.01   [1.02$\times$]
  & 441.08\,/\,423.23 [4.0\%] \\
100
  & 100\,/\,7
  & 9.76\,/\,4.52     [2.16$\times$]
  & 20.21\,/\,9.93    [2.03$\times$]
  & 532.96\,/\,447.04 [16.1\%] \\
400
  & 400\,/\,9
  & 38.68\,/\,7.89     [4.90$\times$]
  & 66.13\,/\,16.47    [4.02$\times$]
  & 1095.60\,/\,531.02 [51.5\%] \\
1000
  & 1000\,/\,10
  & 193.07\,/\,15.90    [12.14$\times$]
  & 308.28\,/\,32.33    [9.54$\times$]
  & 3930.34\,/\,1207.75 [69.3\%] \\
\bottomrule
\end{tabular}
\end{table}

The precise Lloyd's algorithm is presented in \Cref{alg:lloyds}. The proof follows a construction similar to that of \citet{clarkson2026transformer} with modifications to handle the binary version of the cluster indices instead of the one-hot version. See \Cref{asec:expressivity} for details and the proof.
This architecture now requires $O( (d+ \log_2 k)^2 )$ parameters reducing the dependence on $k$ from $\Omega(k^2)$ to $O((\log_2 k)^2)$. While the number of parameters is not really a factor in the expressivity results, it can affect computational cost of the inference (which in turn affects the training time).
\Cref{fig:enc-profile} compares the computational costs (in terms of forward pass time, backward pass time, GPU memory usage) of the models with $\demb = d+k$ for \eqref{eq:tk-emb-v1} and our proposed $\demb = d + \ceil{\log_2 k}$ for \eqref{eq:tk-emb-v2} for varying values of $(d,k)$. While the advantage of the proposed model is clear across the board in terms of memory consumption, the gain in terms of forward/backward pass time is only significant when values of $d$ and $k$ are large enough. \Cref{tab:enc-profile} profiles a different setup where we fix $d$ and vary $k$, and see that, once $k$ is sufficiently larger than $d$, the gains from the proposed \eqref{eq:tk-emb-v2} can be very significant even in terms of the forward and backward pass runtimes.
\section{Understanding Learning Convergence and Generalization}
\label{sec:learning}

\begin{wrapfigure}[21]{R}{0.37\textwidth}
\begin{center}
\begin{tcolorbox}[
  enhanced, sharp corners, before=-22pt, before skip=-22pt,
  left skip=0pt, right skip=0pt,
  left=0pt, right=0pt, top=0pt, bottom=0pt,
  colframe=Sepia!80!black,
  colback=Sepia!15!white,
  boxrule=0pt,
]
{\small
\begin{algorithm}[H]
\SetAlgoCaptionSeparator{}% no separator, default colon
\caption{Learning to $k$-cluster in $d$ dimensions using a transformer $\cmodel$ with learnable parameters $\Theta$, given a distribution $\gD$ of clustering tasks, an $\min$-regularizer $\Omega: \Delta^k \to \R$ with penalty $\tau > 0$, and optimization hyperparameters of learning rate $\eta$, task batch size $B$ and $M$ training steps.}
\label{alg:training}
\DontPrintSemicolon
{\footnotesize
Initialize $\cmodel$ parameters $\Theta$\;
\For{$m = 1, \ldots, M$} {
  \tcp{Compute batch loss}
  Initialize batch loss $l(\Theta) \texttt{ = } 0$\;
  \For{$b = 1, \ldots, B$} {
    Sample task $(\mX, \mC) \sim \gD$ \;
    $l(\Theta) \texttt{ += } \tilde{\Ls}_\Omega^\tau ( \mX, \cmodel(\mX, \mC; \Theta)  )$
  }
  \tcp{update model params}
  $\Theta \gets \texttt{Update}(\Theta, \nabla_\Theta l(\Theta), \eta)$
}
\Return{$\Theta$}
}
\end{algorithm}
}
\end{tcolorbox}
\end{center}
\end{wrapfigure}

Given the above architecture, we will train the transformers to learn a clustering algorithm given a distribution of clustering tasks.
In any learning setup, we need training data, a model defining the ``forward pass'' or inference, and a loss that quantifies the difference between the (current) model output and the desired one.
As we aim to learn a clustering {\em algorithm}, the training data will be a set of $k$-clustering tasks in $d$ dimensions, each with a pair $(\mX, \mC)$ of points to be clustered $\mX \in \R^{d \times n}$ and initial cluster centers $\mC \in \R^{d \times k}$, generated from a distribution $\gD$ of clustering tasks (for example, each $\mX$ is randomly sampled from a different mixture of Gaussians, and $\mC$ is random $k$-subset of $\mX$).
The model $\cmodel: \R^{d \times n} \times \R^{d \times k} \to \R^{d \times k}$ is a single-layer version of the one presented in \Cref{def:tk-emb-v2}, \ref{def:attn-op} and \ref{def:updates} with a finite $\gamma$ (in all our experiments, we set $\gamma = 1$),
and parameterized with $\Theta = \{ (\mQ_\mu, \mK_\mu, \mV_\mu), \mu = 1, \ldots, 4 \}$.

For a clustering task $(\mX, \mC)$, the initial token embeddings $(\bar{\mX}^{(0)}, \bar{\mC}^{(0)})$ will be generated using embedding functions $(E_C, E_X)$ -- either with $(\eohc, \eohx)$ from \eqref{eq:tk-emb-v1} or $(\ebnc, \ebnx)$ from \eqref{eq:tk-emb-v2} -- and passed through a single transformer layer to get $(\bar{\mX}^{(1)}, \bar{\mC}^{(1)} )$, and the model output $\cmodel(\mX, \mC; \Theta)$ is the first $d$ rows $\bar{\mC}^{(1)}_{:d, :}$ of $\bar{\mC}^{(1)}$
--- this choice implies that we are training the model to perform {\bf single-step clustering}.
Note that we need to learn a separate model for each value of $d$ and $k$ (denoted by the $d\!\!:\!\!k$ subscript in $\cmodel$); the same model can be applied to clustering tasks with different values of $n$.

Given the model output $\bar{\mC}(\Theta) \triangleq \cmodel(\mX, \mC; \Theta)$, we can compute the per-task clustering objective $\Ls: \R^{d \times n} \times \R^{d \times k} \to \R$ using the per-point clustering objective $L: \R^d \times \R^{d \times k} \to \R$ as
\begin{equation} \label{eq:loss}
\Ls(\mX, \cmodel(\mX, \mC; \Theta)) = \sum_{i = 1}^n L(\mX_{:,i}, \cmodel(\mX, \mC; \Theta)).
\end{equation}
As a concrete example, we will use the discrete $k$-means objective $L(\mX_{:,i}, \bar{\mC}(\Theta)) = \min_{j \in \iset{ k }} \eudi{\mX_{:,i}}{\bar{\mC}(\Theta)_{:,j}}$. With in-context supervised learning, the learning loss, such as mean-squared-error or negative cross-entropy, is continuous and thus differentiable, while the $k$-means objective in \cref{eq:loss} is not. Hence, we will consider a smoothed surrogate objective. While $\min$ is not a differentiable function, various regularized forms of $\min$, with a convex regularizer $\Omega: \Delta^k \to \R$ and a penalty $\tau > 0$, are differentiable, and give us an upper-bound for the per-point $k$-means objective $\tilde{L}_\Omega^\tau: \R^d \times \R^{d \times k} \to \R$ of the following form (see \Cref{asec:reg-argmax}) for any $\vx \in \R^d$ and cluster centers $\mC \in \R^{d \times k}$ with simplex weights $\{ \lambda^{\Omega_\tau}_j \geq 0 \}_{j \in \iset{k}}$ where $\sum_{j = 1}^k \lambda^{\Omega_\tau}_j = 1$:%
\begin{equation*} \label{eq:loss-ub}
L(\vx, \mC)
\triangleq
\min_{j \in \iset{ k }} \eudi{\vx}{\mC_{:,j}}
\leq \sum_{j = 1}^k \lambda_j^{\Omega_\tau} \eudi{\vx}{\mC_{:,j}}
\triangleq
\tilde{L}_\Omega^\tau(\vx, \mC),
\end{equation*}
Here,  $\tau = 0$ implies no regularization (and thus the exact minimum). While a common choice for the weights $\{\lambda^{\Omega_\tau}_j \}_{ j \in \iset{ k }}$ is the softmax function (obtained by a negative Shannon entropy or NE-regularized $\min$, with $\tau$ serving as the softmax temperature), a tighter bound can be obtained with the sparsemax function which uses the L2-norm regularization leading to sparse weights~\citep{martins2016softmax, blondel2020learning}.

With this smoothed loss surrogate $\tilde{L}_\Omega^\tau$, we have $L(\mX_{:,i}, \bar{\mC}(\Theta)) \leq \tilde{L}_\Omega^\tau(\mX_{:,i}, \bar{\mC}(\Theta))$; equality holds at $\tau = 0$. However, note that as $\tau \to 0$, $\tilde{L}_\Omega^\tau$ becomes progressively less smooth (\Cref{prop:smooth-obj-lip} in \Cref{asec:reg-argmax}). Given this continuous objective, we train the transformer by solving the following:
\begin{equation} \label{eq:learn-obj}
\min_{\Theta} \E_{(\mX, \mC) \sim \gD } \, \tilde{\Ls}_\Omega^\tau (\mX, \cmodel(\mX, \mC; \Theta)),
\quad
\text{where}
\quad
\tilde{\Ls}_\Omega^\tau (\mX, \bar{\mC}(\Theta))
\triangleq \sum_{i = 1}^n \tilde{L}_\Omega^\tau( \mX_{:i}, \bar{\mC}(\Theta) ).
\end{equation}
In practice we will utilize a finite approximation of the above expectation $\E_{(\mX, \mC) \sim \gD}$ by sampling clustering tasks from $\gD$. The training procedure is detailed in \Cref{alg:training}.
Note that while we focus on the discrete $k$-means objective, this learning scheme is applicable to any discrete clustering objective $L$ for which we can find a reasonably smoothed surrogate objective $\tilde{L}$ such that $L \leq \tilde{L}$.
In each training step, we sample a batch of clustering tasks, compute the surrogate loss $\tilde{\Ls}_\Omega^\tau$ of the model outputs $\cmodel(\mX, \mC; \Theta)$, and use the batch-loss-gradient to update the model.

We build upon existing analysis to discuss the effect of different choices in the learning procedure (namely the regularizer $\Omega$, the penalty $\tau$, and the token embedding) on the training convergence.
For any $\alpha$-Lipschitz and $\beta$-smooth learning loss, various results establish guarantees for SGD (stochastic gradient descent) based learning~\citep{bousquet2000algorithmic, hardt2016train}.
For a finite-sum nonconvex objective, such as the one in \cref{eq:learn-obj}, $M$ steps of SGD with a time-dependent learning rate $\eta_m, m \in \iset{ M }$ converges to $\epsilon$-stationarity with
\begin{equation}\label{eq:sgd-eps-convergence}
\epsilon \sim O \left( \beta \alpha^2 \frac{\sum_{m = 1}^M \eta_m^2 }{ \sum_{m = 1}^M \eta_m } \right),
\quad \text{with }
\eta_m =
\begin{cases}
\eta / m & \Rightarrow \epsilon \sim O(\beta \alpha^2 / \log M) \\
\eta / \sqrt{m} & \Rightarrow \epsilon \sim O( \beta \alpha^2 \log M / \sqrt{M} )
\end{cases},
\end{equation}
for some initial learning rate $\eta > 0$. The $(\alpha, \beta)$ properties of the loss function, along with the learning rate scheduling, drive the convergence rate guarantees. However, $(\alpha, \beta)$ also affect generalization. Given a set $\gS$ of clustering tasks sampled from $\gD$, model parameters $\Theta$ learned with $M$ steps of SGD using a time dependent learning rate $\eta_m = \eta / m$, the generalization error (the gap between the true and estimated loss) can be bounded by combining \citet[Theorems 2.2 and 3.12]{hardt2016train}:
\begin{equation} \label{eq:sgd-stab-gen}
\left| \E_\gS \Big[
\E_{(\mX, \mC) \sim \gD}
\tilde{\Ls}_\Omega^\tau (\mX,\bar{\mC}(\Theta))
- \frac{1}{|\gS|} \!\! \sum_{(\mX, \mC) \in \gS}
\!\!\!\!\!\! \tilde{\Ls}_\Omega^\tau (\mX, \bar{\mC}(\Theta))
\Big] \right| \leq
\underbrace{
\frac{1 + \frac{1}{\beta \eta}}{|\gS| - 1}
(2 \eta \alpha^2)^{\frac{1}{\beta \eta + 1}}
M^{\frac{\beta \eta}{\beta \eta + 1}}
}_{\varepsilon}
.
\end{equation}
The above provides a bound on the smoothed objective $\tilde{\Ls}_\Omega^\tau$. With $\varepsilon$ denoting the above gap bound, we have the following {\em in-distribution} generalization bound for the non-differentiable $k$-means objective:
\begin{equation} \label{eq:kmeans-obj-gen-ub}
\E_\gS \Big[ \E_{(\mX, \mC) \sim D}
\Ls( \mX, \cmodel( \mX, \mC; \Theta ) )
\Big]
\leq
\E_\gS \Big[
\frac{1}{|\gS|} \!\! \sum_{(\mX, \mC) \in \gS}
\!\!\!\! \Ls( \mX, \cmodel( \mX, \mC; \Theta ) )
\Big]
+ C_\Omega^\tau + \varepsilon,
\end{equation}
where $0 \leq C_\Omega^\tau \triangleq \max_{(\mX, \hat{\mC})} \tilde{\Ls}_\Omega^\tau( \mX, \hat{\mC} ) - \Ls( \mX, \hat{\mC} )$ quantifies the gap between the true non-differentiable objective $\Ls$ and smoothed $\tilde{\Ls}_\Omega^\tau$, with smaller $C_\Omega^\tau$ implying tighter generalization bounds, and can be controlled with $\tau$.
Note that, under certain conditions, some regularized forms of $\min$ can ensure that this gap $C_\Omega^\tau = 0$; as an example, if there is usually a sufficiently large margin between the inter-cluster distances and intra-cluster distances, the sparsemax based upperbound is tight.
If this gap $C_\Omega^\tau$ is large (relative to the other terms), the guarantee is weak.
Thus, it is important to ensure that this gap $C_\Omega^\tau$ is sufficiently small by using a tighter smoothed upper-bound of the clustering loss.
We will evaluate its effect empirically.
However, the penalty $\tau$ also plays a role in the Lipschitz constant $\alpha$ of the smoothed $\tilde{\Ls}_\Omega^\tau$ as we show in the following:
\begin{theorem}[informal, see \Cref{thm:loss-lip-full} in \Cref{asec:smoothed-obj-lip} for a detailed version] \label{thm:loss-lip-short}
Assume bounded points $\mX \in [0,1]^{d \times n}$ and initial cluster centers $\mC \in [0,1]^{d \times k}$ for all clustering tasks $(\mX, \mC) \sim \gD$, and bounded learnable model parameters $\Theta, \Theta'$. Let $d_E$ denote the maximum squared norm of the initial center embeddings, with $\norm{ \bar{\mC}^{(0)}_{:,j} }^2 \leq d_E$ for all $j \in \iset{ k }$ for any $\mC$, and let the $\Omega$-regularized $\arg\!\max$ with penalty $\tau > 0$ be $(\nicefrac{\omega_\tau}{\tau})$-input-stable. Then the smoothed clustering loss $\tilde{\Ls}_\Omega^\tau$ is $\alpha$-Lipschitz as follows:
\begin{equation} \label{eq:loss-lip-short}
\begin{split}
& \left|
\tilde{\Ls}_\Omega^\tau( \mX, \cmodel( \mX, \mC; \Theta) )
- \tilde{\Ls}_\Omega^\tau( \mX, \cmodel( \mX, \mC; \Theta' ) )
\right| \leq
\ \alpha \, \norm{ \Theta - \Theta' },
\\
& \quad \text{where} \ \alpha \sim O\left(
\frac{k n \omega_\tau}{\tau} \cdot
\max\{ 1, \gamma^2 \} \cdot
\sqrt{d_E} \cdot
\max\left\{
 1, \left(\frac{d_E}{\sqrt{\demb}} \right)^2
\right\}
\cdot
\exp \left( \gamma \frac{d_E}{\sqrt{\demb}} \right)
\right)
\end{split}
\end{equation}
\end{theorem}
This result highlights how the Lipschitz constant $\alpha$ of the smoothed loss $\tilde{\Ls}_\Omega^\tau$ depends on (i)~the regularized $\min$ penalty $\tau$ as $1/\tau$,
(ii)~the regularizer $\Omega$ through $\omega_\tau$, (iii)~the attention inverse temperature $\gamma$, and (iv)~the token embedding scheme through the $d_E$ and $\demb$ terms, where
$d_E = d + 1, \demb = d + k$ with \eqref{eq:tk-emb-v1}, and $d_E = \demb = d + \ceil{\log_2 k}$ with \eqref{eq:tk-emb-v2}.
Thus, the \eqref{eq:tk-emb-v1} scheme of \citet{clarkson2026transformer} has a better Lipschitz constant than the proposed \eqref{eq:tk-emb-v2}, thereby presenting a tradeoff between fewer parameters and improved convergence. However, for large $d$ (relative to $k$), this difference is limited as $\demb$ and $d_E$ scale as $d$ for both token embeddings.
This result in conjunction with those in \cref{sec:expressivity} (\cref{fig:enc-profile} and \cref{tab:enc-profile}) highlight a tradeoff between the architecture size and the actual training times --- {\em the faster forward/backward passes of the leaner architecture with \eqref{eq:tk-emb-v2} token embeddings do not necessarily translate to faster model training to convergence}.

\begin{figure}[!!b]
\centering
\begin{subfigure}{0.49\textwidth}
\centering
\includegraphics[width=\textwidth]{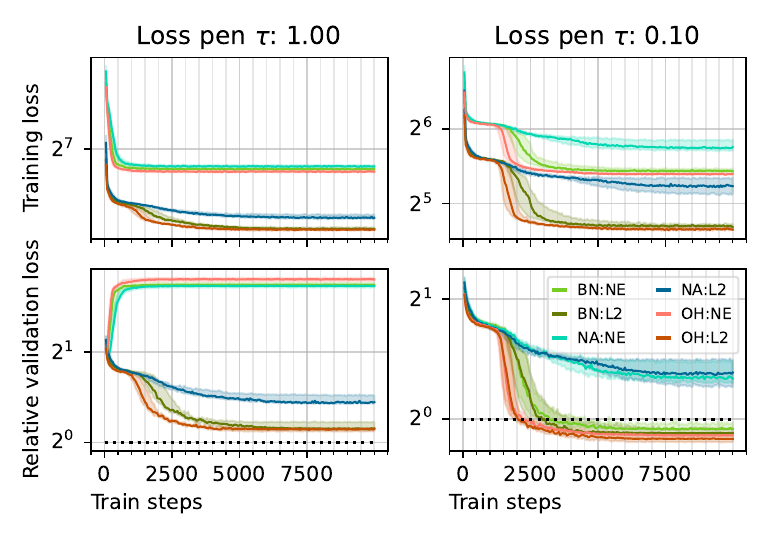}
\caption{Varying $\tau \in \{1, 0.1\}$ with $d = 4$ dimensions.}
\label{fig:tau-var}
\end{subfigure}
~
\begin{subfigure}{0.49\textwidth}
\centering
\includegraphics[width=\textwidth]{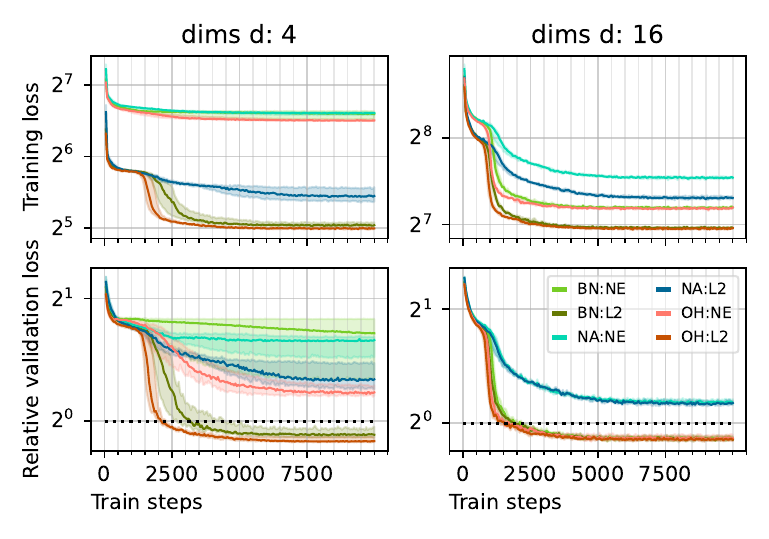}
\caption{Varying $d \in \{4,16\}$ with $\tau = 0.25$.}
\label{fig:dim-var}
\end{subfigure}
\caption{
{\bf Effect of data dimensionality $d$, $\min$-regularizer $\Omega$ with penalty $\tau$, token embedding $(E_C, E_X)$ on training convergence and generalization.}
We show the convergence of (i)~the smoothed training loss $\tilde{\Ls}_\Omega^\tau$ (top row in each figure) and (ii)~the {\em relative} non-differentiable clustering loss $\Ls$ on held-out validation clustering tasks (bottom row) for multiple configurations; the dotted black line at $y = 1$ (bottom row) corresponds to the performance of a single Lloyd's iteration on the validation tasks. See \Cref{asec:ablations} for details and ablations. {\em In all cases, lower values on the vertical axes are better.}
Here, we use clustering tasks with $n = 512$ and $k = 6$. The initial learning rate for Adam is 0.01 (ablation in \cref{fig:ablations:lrate}) with no gradient clipping. In each plot, all 6 curves correspond to the same set of clustering tasks used for training, and we study the behaviour of the NE-regularized $\min$ (softmax) vs the L2-regularized $\min$ (sparsemax), jointly with the effect of $(E_C, E_X)$ --- \eqref{eq:tk-emb-v1} vs \eqref{eq:tk-emb-v2} vs no token embedding (\na).
In \cref{fig:tau-var}, we fix the data dimensionality $d = 4$ and vary the $\min$-regularization penalty $\tau$ (ablation in \cref{fig:ablations:loss-itemp}), and in \cref{fig:dim-var}, we fix $\tau = 0.25$ and vary $d$ (ablation in \cref{fig:ablations:ndims}).
}
\label{fig:tau-omega-demb-dim-analysis}
\end{figure}

To understand how this theoretical characterization aligns with empirical behaviour, we train this model with various configurations. The training clustering tasks contains points sampled from a mixture of isotropic normal distributions (data generation detailed in \Cref{asec:ctasks}).
All experiments are repeated with 10 different random seeds (which affect model initialization, and the tasks sampled), and we aggregate the results over these 10 trials, and present the median (the curve) and the inter-quartile ranges (the translucent ribbon around the curve) of the reported metrics. We use the Adam optimizer with a learning rate of $\eta = 0.01$ for $M = 10000$ steps with a task batch size $B = 32$. Note that we {\em do not utilize the standard practice of gradient clipping} in our experiments to highlight the ease of training the transformer for clustering. Extended ablation of experimental configurations are presented in \Cref{asec:ablations}.

\begin{figure}[t]
\centering
\begin{subfigure}{0.235\textwidth}
\centering
\includegraphics[width=\textwidth]{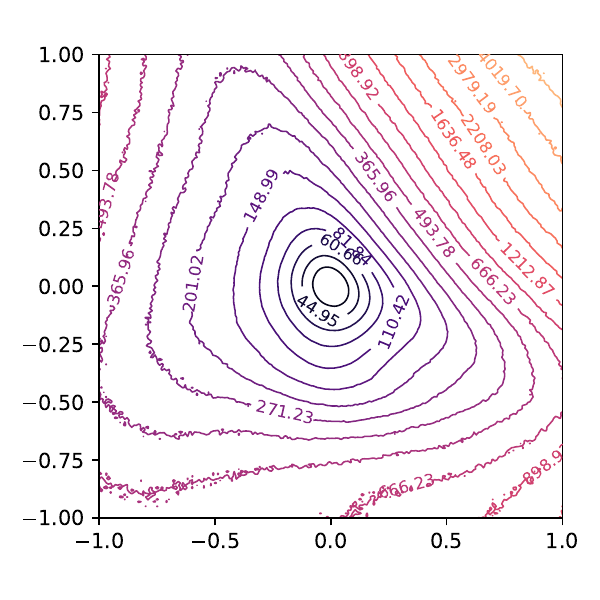}
\caption{$d=4$, \oh}
\label{fig:lsurfs:normal:k6L2d4OH}
\end{subfigure}
~
\begin{subfigure}{0.235\textwidth}
\centering
\includegraphics[width=\textwidth]{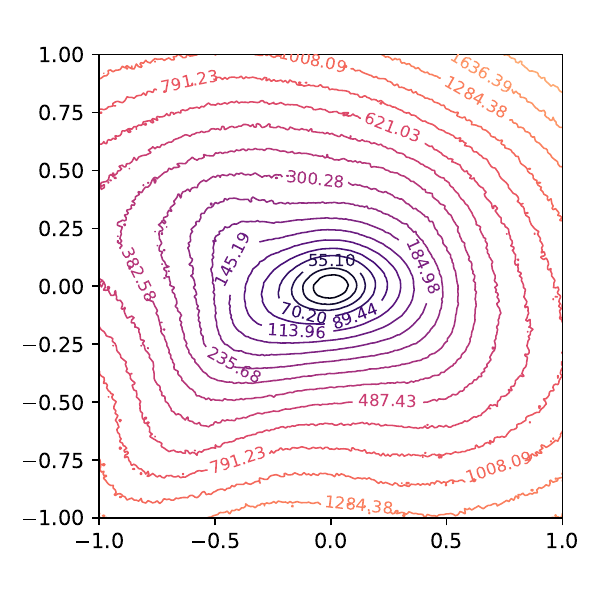}
\caption{$d=4$, \bn}
\label{fig:lsurfs:normal:k6L2d4BN}
\end{subfigure}
~
\begin{subfigure}{0.235\textwidth}
\centering
\includegraphics[width=\textwidth]{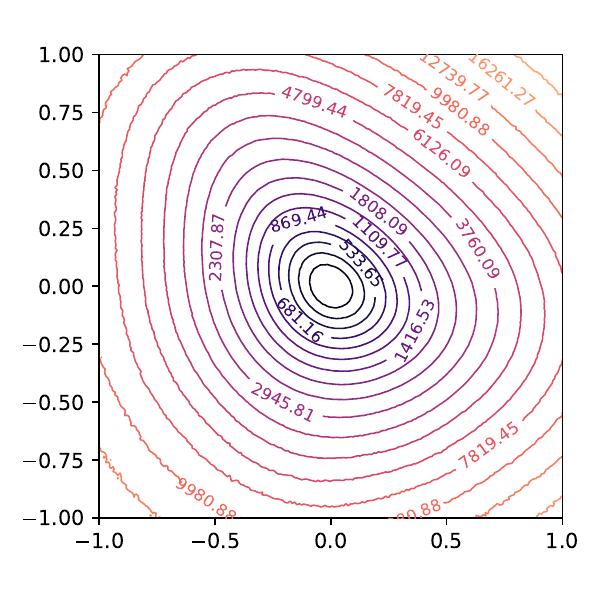}
\caption{$d=32$, \oh}
\label{fig:lsurfs:normal:k6L2d32OH}
\end{subfigure}
~
\begin{subfigure}{0.235\textwidth}
\centering
\includegraphics[width=\textwidth]{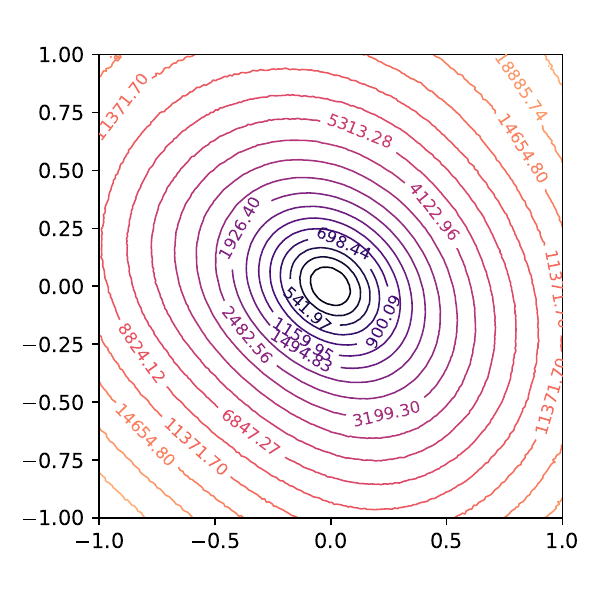}
\caption{$d=32$, \bn}
\label{fig:lsurfs:normal:k6L2d32BN}
\end{subfigure}
\caption{
{\bf Loss surfaces for different dimensionalities ($d=4$ vs $d=32$), token embeddings \eqref{eq:tk-emb-v1} vs \eqref{eq:tk-emb-v2} with the L2-regularized $\min$.}
We fix $k=6$ and vary $d$ to visualize its effect.
We use the procedure in \citet{li2018visualizing} with filterwise normalization and 2 random vectors in the parameter space.
}
\label{fig:lsurfs:normal:k6L2}
\end{figure}

\begin{figure}[!!b]
\centering
\begin{subfigure}{0.49\textwidth}
\centering
\includegraphics[width=\textwidth]{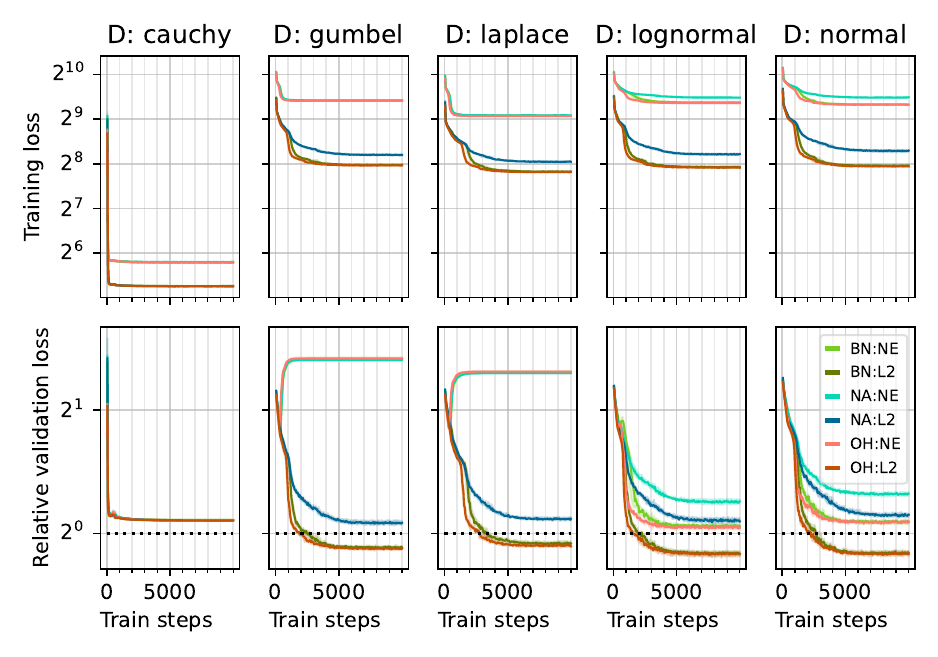}
\caption{$\tau = 1$.}
\label{fig:dist-k10-itemp1}
\end{subfigure}
~
\begin{subfigure}{0.49\textwidth}
\centering
\includegraphics[width=\textwidth]{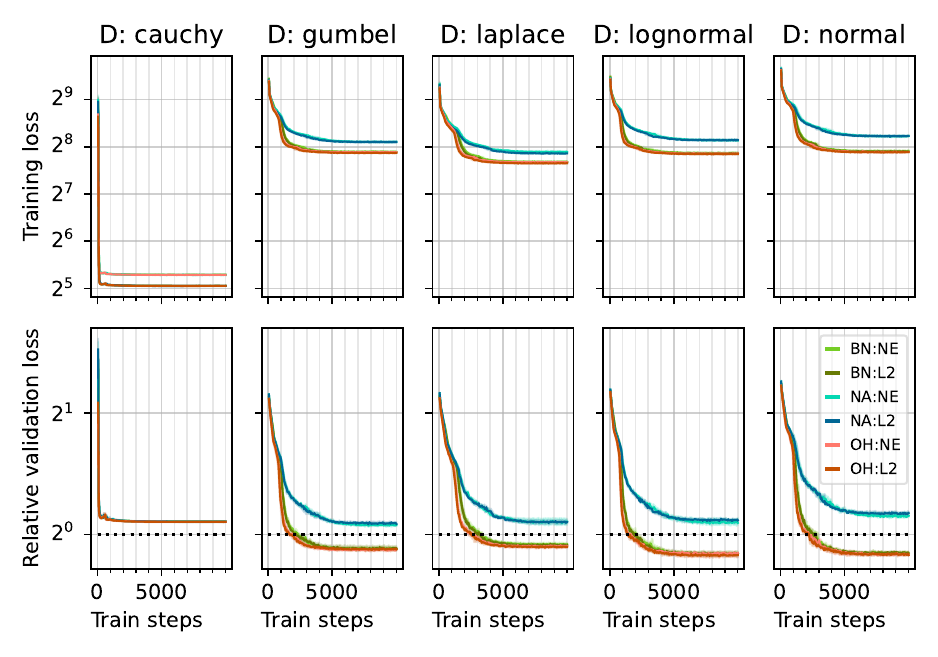}
\caption{$\tau = 0.1$.}
\label{fig:dist-k10-itemp10}
\end{subfigure}
\caption{
{\bf Effect of data distribution.}
The above figures are as described in \cref{fig:tau-omega-demb-dim-analysis}.
We train on tasks with $n = 512, d=32, k=10$ generated from the noted distributions. The validation loss is computed on clustering tasks generated from the training distribution (see ablation in \Cref{fig:ablations:dist}).
}
\label{fig:dists-var}
\end{figure}

\Cref{fig:tau-omega-demb-dim-analysis} shows the training convergence and validation performance.
In almost all cases, the training loss converges quite monotonically, and the validation performance usually tracks the training.
In \cref{fig:tau-var}, we vary the temperature $\tau$ of the smoothed loss $\tilde{\Ls}_\Omega^\tau$, and observe that for any value of $\tau$, the {\em L2-regularized $\min$ always produces lower training losses, and thus tighter upper-bounds (and lower values of $C_\Omega^\tau$)}. The worst upper-bounds with the NE-regularizer are reflected in the validation performance at $\tau = 1$, where {\em the learned model does not perform well on the unseen validation clustering tasks for any of the token embedding schemes} (the validation loss increases during training). In contrast, the L2-regularization
is able to reduce the clustering loss on unseen tasks as training progresses, even though it is unable to match Lloyd's at $\tau = 1$. This highlights the effect of the tightness of the smoothed upper-bound discussed in \cref{eq:kmeans-obj-gen-ub}.
At $\tau = 0.1$, both forms of regularized $\min$ (NE and L2) produce tighter bounds, and {\em the learned model is able to significantly outperform Lloyd's for single-step clustering on the unseen validation tasks}, except with the \na embedding scheme (which does show reduction in validation loss though is unable to match Lloyd's). However, training converges faster for higher values of $\tau$ as implied by \cref{eq:loss-lip-short} though it appears to converge to a worse solution. Thus {\em lowering $\tau$ to tighten the upper-bound can slow down training convergence}.
Finally, for all values of $\tau$, {\em the training converges faster with \eqref{eq:tk-emb-v1} compared to \eqref{eq:tk-emb-v2}}, though {\em both finally converge to a similar loss and are able to outperform single-step Lloyd's}. This effect is visible for both L2-regularized and NE-regularized $\min$ for $\tau = 0.1$. This aligns with \cref{eq:loss-lip-short}, where \eqref{eq:tk-emb-v1} has a smaller Lipschitz constant than \eqref{eq:tk-emb-v2}.
In \cref{fig:dim-var}, we vary the data dimensionality $d$, and see that the results continue to highlight the expressivity of \eqref{eq:tk-emb-v1} and \eqref{eq:tk-emb-v2} (and the lack thereof with \na), and the tighter upper-bounds of the L2-regularized $\min$ as discussed for \cref{fig:tau-var}. We also see that, for a fixed $\tau = 0.25$ (and controlling for other factors), {\em increasing $d$ removes the difference between the convergence of \eqref{eq:tk-emb-v1} vs \eqref{eq:tk-emb-v2}}. We compare the convergence in terms of actual training times in \cref{fig:train-time} in \cref{asec:ablations} to understand the interplay of the faster forward/backward pass and slower convergence of the proposed \eqref{eq:tk-emb-v2}.

\begin{wrapfigure}[13]{r}{0.42\textwidth}
\vskip -0.1in
\centering
\begin{subfigure}{0.2\textwidth}
\centering
\includegraphics[width=\textwidth]{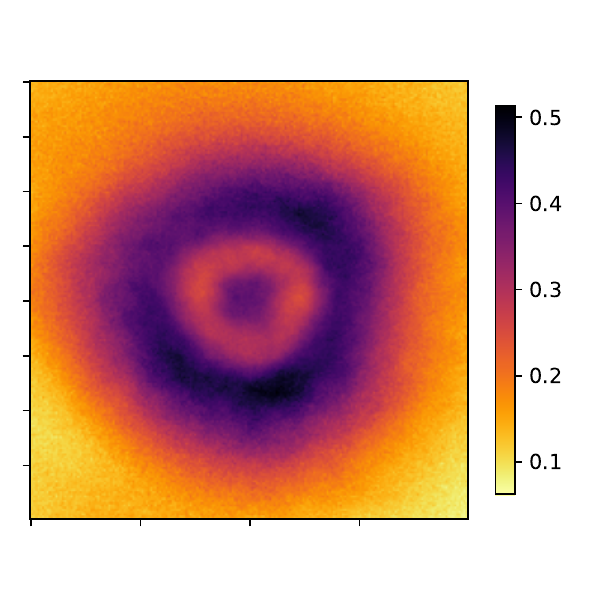}
\vskip -0.15in
\caption{\texttt{normal}}
\label{fig:hm:ubdiff:normal}
\end{subfigure}
~
\begin{subfigure}{0.2\textwidth}
\centering
\includegraphics[width=\textwidth]{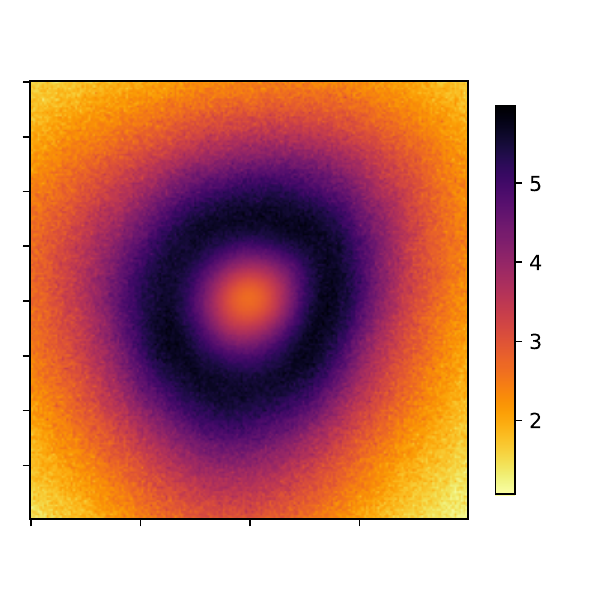}
\vskip -0.15in
\caption{\texttt{cauchy}}
\label{fig:hm:ubdiff:cauchy}
\end{subfigure}
\vskip -0.1in
\caption{
{\bf Training loss gap surfaces.}
We visualize the landscape of the loss gap  $(\tilde{\Ls}_\Omega^\tau - \Ls)$ between the smoothed $\tilde{\Ls}_\Omega^\tau$ and the $k$-means loss $\Ls$ at $\tau = 0.1$.
}
\label{fig:ubdiff}
\end{wrapfigure}

We visualize the loss surfaces~\citep{li2018visualizing} for the L2-regularized smoothed loss in \cref{fig:lsurfs:normal:k6L2}, and see that in all cases, {\em the loss surfaces appear fairly unimodal without any complicated landscapes}, explaining why these models train so easily with apparent monotonic convergence. For the smaller $d=4$, {\em the loss surfaces of \eqref{eq:tk-emb-v1} and \eqref{eq:tk-emb-v2} are markedly different} aligning with the significant difference in their $d_E, \demb$ terms in \Cref{thm:loss-lip-short}. For the larger $d=32$, {\em the loss surfaces of \oh and \bn are quite similar}, as \Cref{thm:loss-lip-short} suggests given the smaller relative differences between their $d_E, \demb$ terms.

In \cref{fig:dists-var}, we use clustering tasks with the points $\mX$ generated from a mixture of isotropic distributions, and consider 5 different distribution families: \texttt{cauchy}, \texttt{gumbel}, \texttt{laplace}, \texttt{lognormal}, \texttt{normal}. We evaluate two values of $\tau$. In all cases, the models are trained and evaluated on the same family. First, we see that our proposed training scheme is unable to generalize when trained with data from the \texttt{cauchy} distribution --- the training seems to converge very quickly, but the performance on unseen clustering tasks are significantly worse than that of the single Lloyd's iteration. For the remaining families with the larger $\tau = 1$, we see that the L2-regularized $\min$ with the \eqref{eq:tk-emb-v1} and \eqref{eq:tk-emb-v2} token embeddings generalize well, and significantly outperform a single Lloyd's iteration; the NE-regularization is unable to match a single Lloyd's iteration, highlighting the benefit of the tighter loss upper-bound. When $\tau = 0.1$, both NE-regularized and L2-regularized $\min$ generalize equally well. Overall, we see that we learn generalizing transformers for all considered distributions except with \texttt{cauchy}; this failure remains an open question. We hypothesize that it is related to the larger gap between the smoothed upper-bound $\tilde{\Ls}$ and the discrete $k$-means objective $\Ls$ over the parameter space, which we visualize for \texttt{cauchy} and \texttt{normal} in \cref{fig:ubdiff} --- the gap is at most 0.5 for \texttt{normal}, but has an order-of-magnitude larger range with \texttt{cauchy}, with the largest gap being greater than 5.
\section{Probing Algorithmic Generalization Beyond Training Distributions}
\label{sec:emp}
\begin{figure}[tb]
\centering
\includegraphics[width=0.75\textwidth]{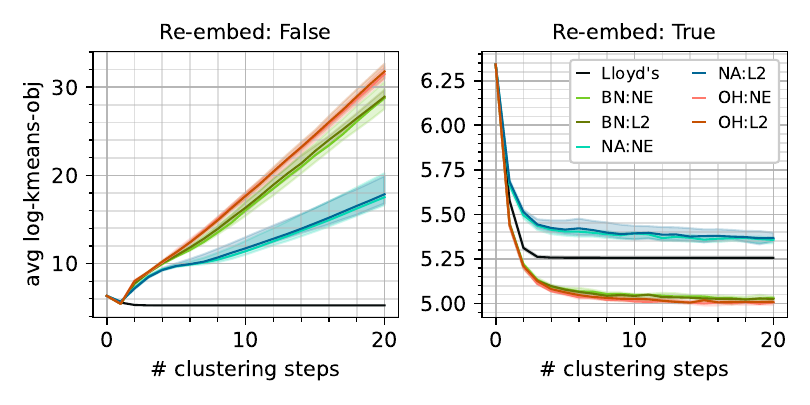}
\caption{{\bf Multi-step clustering performance.} Models trained with $n=512, d = 32, k=10, \tau = 0.1$ and used for $T=20$ step clustering. We consider all configurations (NE/L2 regularizers and \oh/\bn/\na embeddings).
{\bf [Left]}~Multi-step clustering by applying \cref{eq:updates} iteratively for $T$ steps.
{\bf [Right]}~Multi-step clustering where the model input is re-embedded at each of the $T$ steps, treating $T$-steps as $T$ single-steps of clustering. See \cref{asec:multi-step} for further details.}
\label{fig:multi-step}
\end{figure}

In this section, we evaluate the trained transformers (which have implicitly learned a single-step clustering algorithm) on varied distributions of clustering tasks. First, in \cref{fig:multi-step}, we present the performance of the models (trained for single-step clustering) for multiple clustering steps.
We consider two versions of multi-step clustering, and compare them to Lloyd's iterations with the same initial centers. See \Cref{asec:multi-step} for details.
The results clearly show that naively applying the attention updates in \cref{eq:updates} iteratively does not work beyond a single-step, as should be expected since the models are trained for single-step clustering.
After the first-step, the $k$-means objective of the centers generated by the models diverge for all forms of token embeddings (\oh/\bn/\na) and $\min$-regularizer (NE/L2).
However, treating multi-step clustering as multiple single-steps of clustering (thus, re-embedding the inputs at each step) does show improved clustering performance with multiple steps, and models with \oh/\bn token embeddings significantly outperform Lloyd's iterations. Models with no token embeddings also show improvements with multiple steps but underperform Lloyd's. This highlights that {\em even though we train these models for single-step clustering, allowing for efficient training, they are viable for multi-step clustering as the performance improves when used iteratively in a careful manner}.

\begin{figure}[!!b]
\centering
\includegraphics[width=0.95\textwidth]{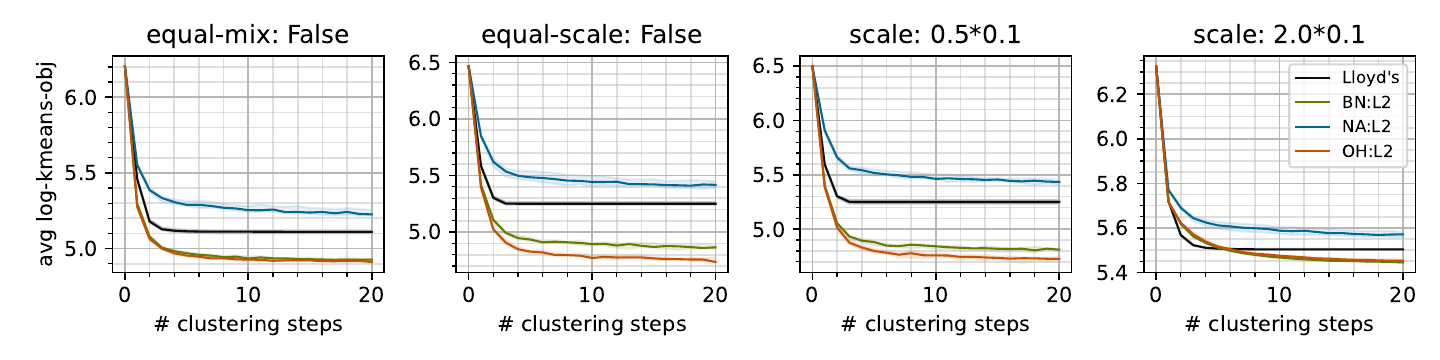}
\caption{
{\bf Generalization to varying distributional properties.} Models are trained with $n=512, d = 32, k=10, \tau = 0.1$, L2-regularized $\min$ with \oh/\bn/\na token embeddings, and evaluated on $T=20$ step clustering with tasks from mixtures of \texttt{normal} distributions. Training tasks are from equally weighted mixtures of isotropic \texttt{normal} distributions with a per-dimensional scale of 0.1. We evaluate generalization in 4 scenarios: {\bf [Left]}~{\em unequally weighted} mixture of isotropic distributions, {\bf [Left center]}~equally weighted mixture of {\em anisotropic} distributions, {\bf [Right center]}~equally weighted mixture of isotropic distributions with a {\em per-dimension scale smaller (half) than that during training}, and {\bf [Right]}~same as {\bf Right center} but {\em with a scale larger (double) than that during training}.
}
\label{fig:vardist-gen}
\end{figure}

Next, we evaluate the ability of the model (trained on clustering tasks from one distribution) to generalize to tasks from different distributions. In \cref{fig:vardist-gen}, we consider models trained on tasks with data sampled from an equally-weighted mixtures of isotropic \texttt{normal} distributions,
and study their generalization to various changes to the mixture distributions.
The learned models can perform multi-step clustering quite well, with \oh and \bn token embedding based models outperforming Lloyd's iterations significantly in all but one case --- for data sampled from a mixture of \texttt{normal} distributions with a larger scale (than the one used during training), the models slightly under-perform Lloyd's initially, but eventually outperform Lloyd's with additional clustering steps.

\begin{wrapfigure}[24]{r}{0.45\textwidth}
\vskip -0.1in
\centering
\includegraphics[width=0.45\textwidth]{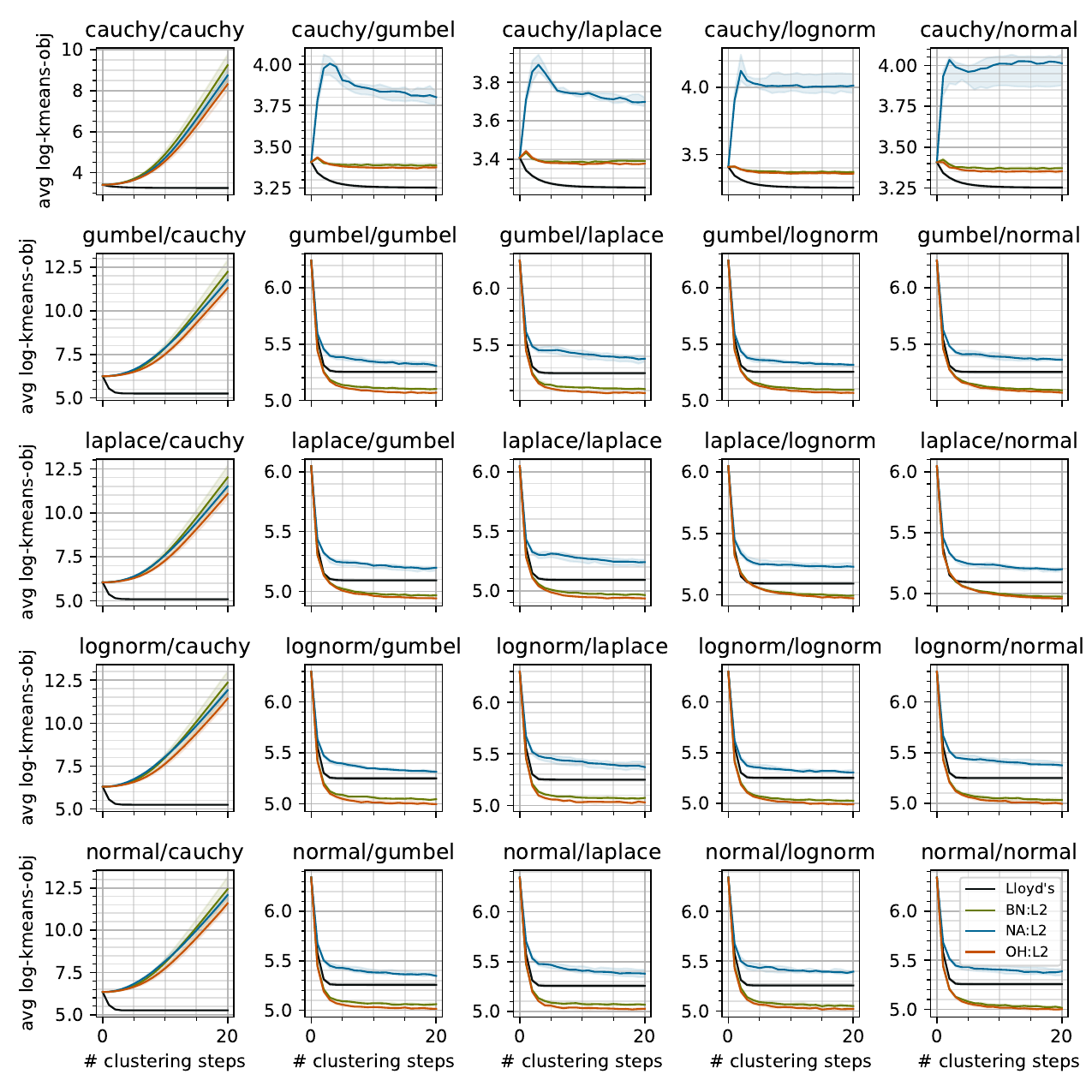}
\caption{{\bf Generalization to different distribution families.}
The models are trained on one distribution family (each column maps to a training family), and tested for multi-step ($T=20$) clustering from a different distribution family (each row maps to a test family).
}
\label{fig:varfam-gen}
\end{wrapfigure}

As a further evaluation of algorithmic generalization, we study the effect of changing the {\em distribution family} in \cref{fig:varfam-gen} (larger version in \cref{fig:varfam-gen-large}). We see that models trained on clustering tasks with data from the \texttt{cauchy} distributions (first column in \cref{fig:varfam-gen}) always diverge for all distributions. This aligns with our previous results in \cref{fig:dists-var}, where models trained on \texttt{cauchy} tasks are even unable to generalize well to single-step \texttt{cauchy} clustering problems. Models trained on other distributions are also unable to perform well with \texttt{cauchy} tasks (first row in \cref{fig:varfam-gen}); the models with \oh and \bn embeddings do not diverge, but are not competitive to Lloyd's here. The remaining pairwise evaluations (the lower-right $4 \times 4$ in \cref{fig:varfam-gen}) all show that {\em the learned models can generalize well to evaluation tasks from other distributions, and continue to significantly outperform Lloyd's iterations} (for \oh/\bn). While generalization to \texttt{cauchy} is an open question, and we leave it for future work, we provide a preliminary probe in \cref{fig:task-sim} (more tasks per family in \cref{fig:task-sim-larger}) where we show the pairwise distributional dissimilarities between the per-task initial-center to point squared Euclidean distances. The \texttt{cauchy} tasks are clearly only similar to other \texttt{cauchy} tasks, and are quite dissimilar to the rest; all other task families are quite similar to each other. These task (dis)similarities align quite well with the generalization (or lack thereof) we observe in \cref{fig:varfam-gen}.

\begin{wrapfigure}[23]{R}{0.4\textwidth}
\vskip -0.1in
\centering
\includegraphics[clip, trim=4in 0.5in 4in 0.32in, width=0.4\textwidth]{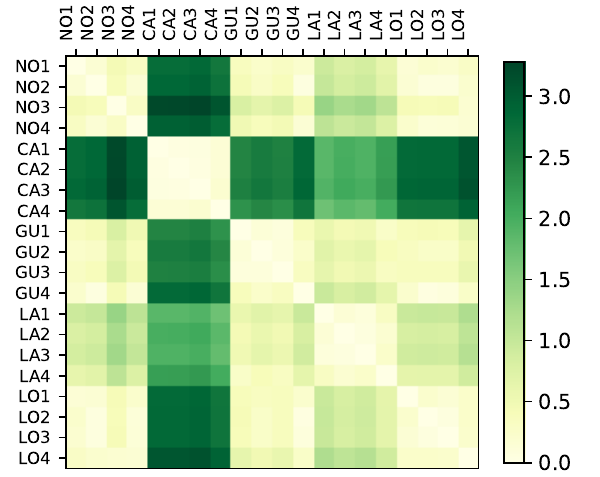}
\caption{
{\bf Pairwise task dissimilarities.}
We use 4 tasks from each of the families and compute the pairwise Earth-mover distance (EMD) between their respective distribution of the distances between the points $\mX$ and initial centers $\mC$, giving us a $20 \times 20$ matrix. Each row and column corresponds to a single task, and the tasks are arranged into family-based blocks in the following order: \texttt{normal}, \texttt{cauchy}, \texttt{gumbel}, \texttt{laplace}, \texttt{lognormal}.
{\em Light shades show low EMD.}
}
\label{fig:task-sim}
\end{wrapfigure}

\section{Discussion} \label{sec:related}
Building upon the initial findings of \citet{garg2022can} and \citet{akyurek2023what}, the area of in-context learning with transformers is very active, usually focusing on supervised learning tasks where some labeled examples $\{ (\vx_i, y_i ) \}_{i \in \iset{n}}$ are provided in the context alongside a single unlabeled query $\vx_q$, and the predicted output $\hat{y}_q$ for the query is generated in-context during the forward pass through the transformer.
\citet{von2023transformers} demonstrate that transformers with {\em linear attention} can execute gradient descent for a linear model in-context, while subsequent works study learnability with simplified linear attention~\citep{zhang2023trained} and the standard softmax attention~\citep{huang2023context}; note that \citet{huang2023context} study a very stylized transformer with structured weight matrices. These results have been extended to demonstrate preconditioned gradient descent for a linear model~\citep{ahn2023transformers} and functional gradient descent~\citep{cheng2024transformers}. Usually, a differentiable supervised learning objective (such as mean squared error for regression) is considered.
Our focus is on discrete clustering problems, where the learning objective is not natively differentiable, and we carefully study the effect of the necessary smoothing of the clustering objective on learnability.

While most existing literature focuses on linear attention transformers with (causal) self-attention, we build upon \citet{clarkson2026transformer} and study softmax attention with an encoder-decoder model that uses both self-attention and cross-attention, though the algorithmic expressivity results rely on a {\em limiting softmax attention} where the attention softmax temperature goes to zero. However, this form of softmax attention is practically infeasible, and we use standard softmax attention for our learnability analysis and experiments.

Another line of research connecting transformers and clustering is the study of the token dynamics of self-attention where the tokens are iteratively updated via self-attention, and the goal is to understand the conditions under which these dynamics stabilize and converge~\citep{geshkovski2023emergence, geshkovski2023mathematical}. The results show that, under certain conditions, the token representations tend to self-cluster, first into a small number of clusters, and then eventually, to a single cluster.
In contrast, we focus on learning clustering algorithms that are competitive to existing hand-crafted schemes, and train a single transformer layer instead of the problem setup of the aforementioned papers with
multiple (possibly infinite) layers. We learn the transformer parameters from data (in this case, a collection of clustering tasks) and study their training convergence and algorithmic generalization.

\paragraph{Conclusion.}
We study the ability of transformers to express and learn clustering algorithms. First, we establish expressivity of Lloyd's algorithm for $k$-means clustering with a leaner transformer (\Cref{thm:lloyds-exp}) than existing constructions~\citep{clarkson2026transformer}. Then, we study the factors affecting the ability of transformers to learn a clustering algorithm (\Cref{thm:loss-lip-short}), and empirically validate their importance (\Cref{fig:tau-omega-demb-dim-analysis}, \ref{fig:lsurfs:normal:k6L2}, \ref{fig:dists-var}, \ref{fig:ubdiff}). We show that these learned algorithms -- in the form of transformers -- can progressively improve clustering with iterative application (\Cref{fig:multi-step}), though naive iteration fails, and can successfully execute clustering for tasks that are different than those used for training (\Cref{fig:vardist-gen}, \ref{fig:varfam-gen}), though exceptions exist.
This work comes with various limitations, with the main ones being (i)~the restriction that transformers learn an algorithm for a specific data dimensionality $d$ and number of clusters $k$, while hand-crafted algorithms are more general, and (ii)~the lack of understanding of the failure to train and generalize with \texttt{cauchy} tasks. We discuss these and various others in \Cref{asec:lims}.

\clearpage
\bibliography{refs}
\bibliographystyle{unsrtnat}

%%%%%%%%%%%%%%%%%%%%%%%%%%%%%%%%%%%%%%%%%%%%%%%%%%%%%%%%%%%%

\clearpage
\appendix
\crefalias{section}{appendix}
\allowdisplaybreaks

\newpage
\addcontentsline{toc}{section}{Appendix}
\part{Appendix}
\parttoc
\clearpage

\section{Limitations} \label{asec:lims}

As discussed in our concluding remarks, one major limitation of our work is that we currently lack a clear understanding as to why the learning and generalization fail in certain situations. For example, it is not clear why learning and generalization is hard with \texttt{cauchy} distributions. We hypothesize that this is related to the inability of the smoothed clustering objective (via a regularized $\min$) to provide a tight upper-bound for the actual $k$-means clustering loss (as we show in \Cref{fig:ubdiff}). However, it is not completely clear why, but we speculate that the following is a potential path to understanding:
The \texttt{cauchy} distribution is fat-tailed, and hence, many points in $\mX$ can be quite far from the true cluster, and act like outliers in a clustering problem.
The presence of outlier values -- the large pairwise squared Euclidean distances to the outliers compared to the relatively smaller pairwise squared distances to the inliers -- in the regularized $\min$ can severely degrade the upper-bounds obtained via the smoothed objective.
This still does not give us a clear way to mitigate the issues while training with \texttt{cauchy} tasks seen in \Cref{fig:dists-var}. Similarly, the presence of outliers in $\mX$ also results in the significant pairwise distributional distances between the tasks shown in \Cref{fig:task-sim} and \Cref{fig:task-sim-larger}, leading to the failures of the trained transformers on \texttt{cauchy} validation tasks.

Beyond that, there are architectural limitations we inherit from \citet{clarkson2026transformer}, where the transformer learns to cluster for a specific choice of data dimensionality $d$ and number of clusters $k$, though it is able to perform clustering for varying number of points $n$ in $d$ dimensions into $k$ clusters (as shown in \Cref{fig:lengen-nsamples}). It is important to extend this architecture to be able to generalize to different values of $k$ or $d$, and be able to train on clustering tasks with a certain $k$ value and generalize to larger $k$ values. One possible way to generalize in terms of $d$ is to train the transformer with a large $d'$. Then any clustering task with $d < d'$ can be input to the transformer by zero-padding the necessary extra $(d'-d)$ dimensions to the original data in $d$ dimensions. As Euclidean distances are unaffected by zero-padding, this idea should conceptually work. However, we pay the computational cost of training a model for a large $d'$ as the model embedding dimensionality $\demb$ would scale linearly in $d'$, while the number of parameters would scale quadratically in $d'$. Whether such a setup would show strong generalization performance across various values of $d$ is unclear, and would depend on the training strategies. A similar scheme could also be considered to generalize to different values of $k$, where we would train a model on a large value of $k'$.

Additionally, there are various open questions we do not address in this paper, such as
\begin{itemize}
\item It is not clear if the additional dimensions -- $k$ dimensions with \eqref{eq:tk-emb-v1} and $\ceil{\log_2 k}$ with \eqref{eq:tk-emb-v2} -- are necessary to learn a clustering algorithm that is competitive to Lloyd's algorithm. Empirically, we see that these additional dimensions do allow the models to surpass Lloyd's algorithm, and the models without these additional dimensions (the \na models) are empirically unable to match Lloyd's. However, we do not have theory codifying this need for the additional dimensions.
\item Our proposed token embedding in \eqref{eq:tk-emb-v2} can express Lloyd's algorithm for $k$-means clustering with the limiting softmax attention (with $\gamma \to \infty$), but can introduce an arbitrary topological order among the cluster centers with a finite $\gamma$. We already see that models with the \bn token embedding have slower training convergence compared to the models with the \oh token embeddings (especially for small $d$), but it is possible that this arbitrary order can introduce other unexpected behaviours.
\item The standard transformer has an attention block and a feedforward block (which we do not consider in this work). It is not clear how the inclusion of the feedforward block would affect learning and generalization. The theoretical results in \citet{clarkson2026transformer} utilize the feedforward block to perform Euclidean clustering with dot-product attention (instead of the negative squared Euclidean distance based attention in \Cref{def:attn-op}), but it is not clear what role the feedforward block will play in the practical learning setup.
\item While we are able to train the model efficiently for single-step clustering tasks, and utilize it for multi-step clustering by treating multi-step clustering as multiple single-step clustering tasks, it is not clear if we can learn a better model by training directly for multi-step clustering with appropriately designed objectives, and do so efficiently.
\item Our discussion in \Cref{sec:learning} around learning convergence and generalization consider training losses that are $\beta$-smooth. We see that the L2-regularized $\min$ (using sparsemax~\citep{martins2016softmax}) provides a tighter upper-bound and thus generalizes better in many cases over the NE-regularized $\min$ (using softmax). However, the sparsemax operation is only differentiable almost everywhere, and is not generally $\beta$-smooth, thus leading to a mismatch between our theoretical discussion and empirical results. One possible explanation for the strong empirical performance of the L2-regularized $\min$ is that the loss surface is actually more favourable than the usual nonconvex loss landscapes of neural networks, as can be seen in \Cref{fig:lsurfs:normal:k6L2}, where the loss landscapes appear convex for the larger value of $d$.
\item We treat the softmax attention inverse temperature $\gamma$ and the regularized $\min$ penalty $\tau$ as fixed hyperparameters in our analysis in \Cref{thm:loss-lip-short} and in our experiments. But in practice, these parameters can be adaptive or annealed through the training and inference, and this is not something we consider here.
\item Finally, while we see that the clustering algorithm learned by the transformer is able to outperform Lloyd's algorithm for $k$-means clustering, it is not clear what the precise clustering algorithm is that the learned model is executing. While we can consider the tools used in \citet{akyurek2023what} such as the notion of ``algorithmic distance'', we leave this for future work.
\end{itemize}

\section{Leaner Transformers Can Emulate Lloyd's Iteration} \label{asec:expressivity}

\begin{algorithm}[tb]
\SetAlgoCaptionSeparator{}% no separator, default colon
\caption{Lloyd's algorithm to cluster $n$ points into $k$ clusters}
\label{alg:lloyds}
\DontPrintSemicolon
Input: Instances $\{\vx_i \in \R^d, i \in \iset{ n } \}$\;
Input: Number of iterations $T$\;
Input: Initial centers $\{\vc_j \in \R^d, j \in \iset{k} \}$\;
$ \forall j \in \iset{ k }, \vc_j^{(0)} \gets \vc_j$ \;
\For{$t = 1, \ldots, T$} {
  $\forall i \in \iset{ n }, \ a_{i}^{(t)} \gets \argmin_{j \in \iset{ k }} \eudi{\vx_i}{\vc_j^{(t-1)}}$ \;
  $\forall j \in \iset{ k }, \ \vc_j^{(t)} \gets \argmin_{\vc \in \R^d} \sum_{ i \in \iset{ n }: a_i^{(t)} = j}  \eudi{ \vx_i }{ \vc }$
}
\Return{$\{ \vc_j^{(T)}, j \in [k] \}$}
\end{algorithm}

We use $\mI_d$ to denote the $(d \times d)$ identity matrix. Further, we specify our matrices for attention in a compact form, using the following notation. For integers $1 \le i \le j\le d$, let $\mI_d^{i:j}$ be the matrix such that for any $\vx\in\R^{d}$ such that
\begin{equation}
\vv = \mI_d^{i:j} \vx,
\quad \text{ where }
v_m = \begin{cases}
  x_m & i \leq m \leq j \\
  0 & o.w.
\end{cases}.
\end{equation}
The $\vv = \mI_d^{i:j} \vx$ zeroes out the $m$-th entries of $\vx$ with indices $m \in [1,i) \cup (j, d]$, leaving the rest unchanged. When $j=d$, we use $\mI_d^{i:}$, and when $i=1$, we use $\mI_d^{:j}$.

We prove \cref{thm:lloyds-exp} based on the construction in \citet[Theorem 2]{clarkson2026transformer} with minor modification. We use the selective property of attention as described in \citet[Appendix D]{clarkson2026transformer} in the form of {\em averaging hard attention} or AHAT as the inverse temperature of the softmax goes to infinity; this behaves as a {\em hardmax} when there is a unique maximum. We are restating \cref{thm:lloyds-exp} as follows:
\begin{theorem}\label{athm:lloyds-exp}
Given as input a set of points $\mX \in \R^{d \times n}$ and initial centers $\mC \in \R^{d\times k}$, assuming that no iterations of Lloyd's algorithm lead to empty clusters, there exists parameters $\{( \mQ_\mu, \mK_\mu, \mV_\mu), \mu = 1, \ldots, 4\}$ for the transformer in \Cref{def:updates}  with $\gamma = \infty$ and initial token embeddings $( \bar{\mX}^{(0)}, \bar{\mC}^{(0)} )$ as in \Cref{def:tk-emb-v2} then, for any $t \in \mathbb{N}$, the output of the $t$-th transformer layer exactly matches the output of Lloyd's algorithm (\Cref{alg:lloyds}) with the same input after $t$ iterations.
\end{theorem}
\begin{proof}
With $\demb = d + \ceil{\log_2 k}$, let $\mK_1 = \mQ_1 = \mK_2 = \mQ_2 = \mI_{\demb}^{:d}$, and $\mV_1 = -\mV_2=\mI_{\demb}^{d+1:}$. Note that $\mI_{\demb}^{:d}$ and $\mI_{\demb}^{d+1:}$ are $(d+ \ceil{\log_2 k}) \times (d+ \ceil{\log_2 k})$ matrices.

At any step $t \geq 0$, let us assume that $\bar{\mX}^{(t)} = \qa{\mX}{\mY^{(t)}}$, and $\bar{\mC}^{(t)} = \qa{\mC^{(t)}}{\mB_k}$. This is true at $t = 0$ by the definition of the token embeddings in \cref{def:tk-emb-v2}. We will show that, at any $t \geq 0$ the update to $(\bar{\mX}^{(t+1)}, \bar{\mC}^{(t+1)})$ maintains this form where $\bar{\mX}^{(t+1)} = \qa{\mX}{\mY^{(t+1)}}$, and $\bar{\mC}^{(t+1)} = \qa{\mC^{(t+1)}}{\mB_k}$, where
\begin{itemize}
\item $\mY^{(t+1)}$ is the matrix with the $i$-th column corresponding to the binary representation of $(j^\star(i) - 1)$ in $\ceil{\log_2 k}$ bits, with $j^\star(i) \in \iset{k}$ denoting the index of the cluster to which the $i$-th point is assigned to in the $(t+1)$-th Lloyd's iteration in line 6 of \Cref{alg:lloyds}.
\item $\mC^{(t+1)}$ corresponds to the updated cluster centers as in line 7 of \Cref{alg:lloyds}.
\end{itemize}
To update cluster assignments for Lloyd's algorithm, that is, to compute $\bar{\mX}^{(t+1)}$ from $\bar{\mX}^{(t)}$ and $\bar{\mC}^{(t)}$, we compute
\begin{align}\label{eq:l2-clus-assign-eu}
&
\attned\gamma (\bar{\mX}^{(t)}, \bar{\mC}^{(t)}; ( \mI_{\demb}^{:d}, \mI_{\demb}^{:d}, \mI_{\demb}^{d+1:} ))
+ \attned\gamma(\bar{\mX}^{(t)}, \bar{\mX}^{(t)}; ( \mI_{\demb}^{:d}, \mI_{\demb}^{:d}, -\mI_{\demb}^{d+1:} ))
\nonumber
\\ &
\qquad
= \mI_{\demb}^{d+1:} \bar{\mC}^{(t)} \softmax(
 - (\nicefrac{\gamma}{\sqrt{\demb}}) \sqdist( \mI_{\demb}^{:d} \bar{\mX}^{(t)}, \mI_{\demb}^{:d} \bar{\mC}^{(t)})
)
\nonumber
\\ & \qquad \qquad
- \mI_{\demb}^{d+1:}\bar{\mX}^{(t)} \softmax(
 - (\nicefrac{\gamma}{\sqrt{\demb}}) \sqdist(\mI_{\demb}^{:d} \bar{\mX}^{(t)}, \mI_{\demb}^{:d} \bar{\mX}^{(t)})
)
\\ &
\qquad
= \qa{\mathbf{0}_{d\times k}}{\mB_k} \softmax \left(
 - (\nicefrac{\gamma}{\sqrt{\demb}})  \sqdist\left(
   \qa{\mX}{0_{\ceil{\log_2 k} \times n}}, \qa{\mC^{(t)}}{0_{\ceil{\log_2 k} \times k}}
\right)\right)
\nonumber
\\ & \qquad \qquad
- \qa{\mathbf{0}_{d\times n}}{\mY^{(t)}} \softmax\left(
 - (\nicefrac{\gamma}{\sqrt{\demb}})  \sqdist\left(
   \qa{\mX}{\mathbf{0}_{\ceil{\log_2 k} \times n}}, \qa{\mX}{\mathbf{0}_{\ceil{\log_2 k}\times n}}
\right) \right)
\\ &
\qquad
= \qa{\mathbf{0}_{d\times k}}{\mB_k} \hardmax \left(- \sqdist\left(\mX, \mC^{(t)} \right)\right)
- \qa{\mathbf{0}_{d\times n}}{\mY^{(t)}} \hardmax \left( - \sqdist\left(\mX, \mX \right) \right),
\end{align}
as $\softmax$ with inverse temperature $\gamma = \infty$ gives us $\hardmax$, where for any matrix $\mA \in \R^{a \times b}$, $\hardmax(\mA) \in \{0, 1\}^{a \times b}$ is binary matrix of the same size as $\mA$ with each column having exactly a single nonzero value corresponding to the unique maximum of the column, assuming each column has an unique maximum.

First, assuming $\vx_i \not= \vx_{i'}$ for any $i \not= i'$, $\hardmax(-\sqdist(\mX, \mX)) = \mI_n$. Next $\hardmax \left(- \sqdist\left(\mX, \mC^{(t)} \right)\right)$ applies a columnwise argmin to the $(k \times n)$ matrix $\sqdist(\mX, \mC^{(t)})$, thereby, each column having only the $j^\star(i)$-th entry as 1 where $j^\star(i) = \argmin_{j \in \iset{k}} \eudi{\mX_{:,i}}{\mC^{(t)}_{:,j}}$. Thus $\mY^{(t+1)} = \mB_k \hardmax \left(- \sqdist\left(\mX, \mC^{(t)} \right)\right)$ is a $(\ceil{\log_2 k} \times n)$ binary matrix, where the $i$-th column $\mY^{(t)}_{:,i}$ is the $j^\star(i)$-th column of $\mB_k$, thereby denoting the updated cluster assignments as per Lloyd's algorithm (\Cref{alg:lloyds}, line 5). All together, we have
\begin{equation}\label{eq:cross-self-xupdate}\begin{split}
& \attned\gamma (\bar{\mX}^{(t)}, \bar{\mC}^{(t)};  ( \mI_{\demb}^{:d}, \mI_{\demb}^{:d}, \mI_{\demb}^{d+1:} ))
+ \attned\gamma(\bar{\mX}^{(t)}, \bar{\mX}^{(t)}; ( \mI_{\demb}^{:d}, \mI_{\demb}^{:d}, -\mI_{\demb}^{d+1:} ))
\\
& \quad
= \qa{\mathbf{0}_{d \times n}}{\mY^{(t+1)} -\mY^{(t)}}.
\end{split}\end{equation}
Adding this to $\bar{\mX}^{(t)} = \qa{\mX}{\mY^{(t)}}$ via the residual connection in \cref{eq:updates} gives us the updated $\bar{\mX}^{(t+1)} = \qa{\mX}{\mY^{(t+1)}}$ with the updated cluster assignments $\mY^{(t+1)}$ matching Lloyd's algorithm (\Cref{alg:lloyds}, line 5).

Let $\mQ_3 = \mK_3 = \mQ_4 = \mK_4 = \mI_{\demb}^{d+1:}$, and $ \mV_3= - \mV_4 = \mI_{\demb}^{:d}$. To update cluster centers as per Lloyd's algorithm (\Cref{alg:lloyds}, line 6), we compute
\begin{align}\label{eq:l2-clus-center-eu}
&  \attned\gamma(\bar{\mC}^{(t)}, \bar{\mX}^{(t+1)}; ( \mI_{\demb}^{d+1:}, \mI_{\demb}^{d+1:}, \mI_{\demb}^{:d} ) )
  + \attned\gamma(\bar{\mC}^{(t)}, \bar{\mC}^{(t)}; ( \mI_{\demb}^{d+1:}, \mI_{\demb}^{d+1:}, -\mI_{\demb}^{:d} ) )
\nonumber
\\ &
\qquad
=  \mI_{\demb}^{:d} \bar{\mX}^{(t+1)} \softmax\left(
 -(\nicefrac{\gamma}{\sqrt{\demb}}) \sqdist\left(
  \mI_{\demb}^{d+1:} \bar{\mC}^{(t)}, \mI_{\demb}^{d+1:} \bar{\mX}^{(t+1)}
\right)\right)
\nonumber
\\ & \qquad \qquad
- \mI_{\demb}^{:d} \bar{\mC}^{(t)} \softmax \left(
 -(\nicefrac{\gamma}{\sqrt{\demb}}) \sqdist\left(
  \mI_{\demb}^{d+1:} \bar{\mC}^{(t)}, \mI_{\demb}^{d+1:} \bar{\mC}^{(t)}
\right) \right)
\\ &
\qquad
=  \qa{\mX}{\mathbf{0}_{\ceil{\log_2 k}\times n}} \softmax\left(
 -(\nicefrac{\gamma}{\sqrt{\demb}}) \sqdist\left(
    \qa{\mathbf{0}_{d\times k}}{\mB_k}, \qa{\mathbf{0}_{d\times n}}{\mY^{(t+1)}}
\right) \right)
\nonumber
\\ & \qquad \qquad
- \qa{\mC^{(t)}}{\mathbf{0}_{\ceil{\log_2 k} \times k}} \softmax\left(
 -(\nicefrac{\gamma}{\sqrt{\demb}}) \sqdist\left(
   \qa{\mathbf{0}_{d\times k}}{\mB_k}, \qa{\mathbf{0}_{d\times k}}{\mB_k}
\right) \right)
\\ &
\qquad
=  \qa{\mX}{\mathbf{0}_{\ceil{\log_2 k}\times n}}
\ahmax\left(-\sqdist\left(\mB_k, \mY^{(t+1)}\right) \right)
\nonumber
\\ & \qquad \qquad
- \qa{\mC^{(t)}}{\mathbf{0}_{\ceil{\log_2 k} \times k}}
\hardmax\left(- \sqdist\left(\mB_k, \mB_k \right) \right)
\end{align}
where $\gamma = \infty$ gives us the columnwise $\hardmax$ and the columnwise averaging hardmax operation $\ahmax$. The reason for the averaging hardmax is because each column in $\sqdist\left(\mB_k, \mY^{(t+1)}\right)$ can have a non-singleton set of minima (or multiple maxima). As before, $\hardmax\left(- \sqdist\left(\mB_k, \mB_k \right) \right) = \mI_k$.

Next, $\ahmax\left(-\sqdist\left(\mB_k, \mY^{(t+1)}\right) \right)$ will be a $(n \times k)$ matrix, where in the $j$-th column, the $i$-th entry will be zero if $\mX_{:,i}$ is not assigned to the $j$-th cluster, that is the $i$-th column of $\mY^{(t+1)}$ does not match the $j$-th column of $\mB_k$; otherwise, the averaging hardmax sets it to $1 / m_j$ where $m_j$ is the number of points assigned to the $j$-th cluster, and thus, the squared Euclidean distance between the $i$-th column of $\mY^{(t+1)}$ and the $j$-th column of $\mB_k$ is zero. Note that $m_j > 0 \forall j \in \iset{k}$ as we assume that there are no empty clusters during Lloyd's iterations. Thus,  $\mX \ahmax\left(-\sqdist\left(\mB_k, \mY^{(t+1)}\right) \right) $ will be a $(d \times k)$ matrix corresponding to the centers $\mC^{(t+1)}$ computed with Lloyd's algorithm (\Cref{alg:lloyds}, line 6). Thus, we have
\begin{equation} \label{eq:cross-self-cupdate} \begin{split}
&
\attned\gamma(\bar{\mC}^{(t)}, \bar{\mX}^{(t+1)}; ( \mI_{\demb}^{d+1:}, \mI_{\demb}^{d+1:}, \mI_{\demb}^{:d} ) )
  + \attned\gamma(\bar{\mC}^{(t)}, \bar{\mC}^{(t)}; ( \mI_{\demb}^{d+1:}, \mI_{\demb}^{d+1:}, -\mI_{\demb}^{:d} ) )
\\
& \quad
= \qa{\mC^{(t+1)} - \mC^{(t)}}{\mathbf{0}_{\ceil{\log_2 k} \times k}},
\end{split}\end{equation}
and adding this to $\bar{\mC}^{(t)} = \qa{\mC^{(t)}}{\mB_k}$ gives us $\bar{\mC}^{(t+1)} = \qa{\mC^{(t+1)}}{\mB_k}$, thus tracking the same cluster centers as in Lloyd's algorithm.

Thus we have established that both the cluster assignments and the cluster center updates of Lloyd's algorithm can be maintained in $(\bar{\mX}^{(t)}, \bar{\mC}^{(t)})$ for any $t$ with the leaner token embedding with $\demb = d + \ceil{\log_2 k}$.
\end{proof}

Note that this construction for the $\bar{\mX}^{(t)} \to \bar{\mX}^{(t+1)}$ is the same as in the proof of \citet[Theorem 2]{clarkson2026transformer}. The construction for the $\bar{\mC}^{(t)} \to \bar{\mC}^{(t+1)}$ update is different from the one in \citet{clarkson2026transformer} as we make use of negative squared Euclidean distance based attention with $\gamma = \infty$ instead of a linear attention as in \citet{clarkson2026transformer}.

\section{Single-Step and Multi-Step Clustering} \label{asec:multi-step}

We set up the learning problem in \cref{eq:learn-obj} as a single-step clustering task where the transformer model generates cluster prototypes in one-shot (with a single forward pass through the model) such that the $k$-means objective (actually a surrogate of it) is reduced. However, almost all clustering algorithms are iterative in nature, where the cluster prototypes are iteratively refined in a way that usually progressively improves the clustering objective; Lloyd's algorithm for $k$-means clustering~\citep{lloyd1982least} and  partitioning around medoids or PAM~\citep{kaufman1990partitioning} are examples of such algorithms. Thus, it is natural to explore whether we can utilize this trained transformer for multi-step clustering, with the cluster prototypes being iteratively refined. Given the theoretical results in \Cref{thm:lloyds-exp} and \citet[Theorem 1]{clarkson2026transformer}, it would seem natural to perform $T$-step clustering refinement by applying the attention-based updates in \cref{def:updates} iteratively for $T$ steps with the learned model parameters, starting from $(\bar{\mX}^{(0)}, \bar{\mC}^{(0)}) \gets ( E_X(\mX), E_C(\mC) )$, and use the first $d$ rows $\bar{\mC}_{:d,:}^{(T)}$ of the $T$-step output $(\bar{\mX}^{(T)}, \bar{\mC}^{(T)})$ as the final set of prototypes. Note that the token embedding scheme can be either the $(\eohx, \eohc)$ scheme in \eqref{eq:tk-emb-v1} or the proposed $(\ebnx, \ebnc)$ scheme in \eqref{eq:tk-emb-v2} in \Cref{def:tk-emb-v2}.

However, it is important to note that this is valid only if the transformer model parameters are as in the constructive proofs of these theoretical results, and if we are using the softmax attention with the inverse temperature $\gamma \to \infty$. We have trained the transformer model with finite $\gamma$.

Thus, in the following, we precisely explain what we mean by {\em single-step clustering} (used during training) vs {\em multi-step clustering} (used for evaluating the model as a clustering algorithm). First, we will demonstrate how standard algorithms such as Lloyd's perform multi-step clustering by applying the same operation iteratively on evolving inputs (data and current centers). Then we will use the same intuition to motivate how we can similarly use the model iteratively on evolving inputs.

For a clustering task with data $\mX$ and initial centers $\mC^{(0)}$, $T$ iterations of Lloyd's algorithm (\cref{alg:lloyds}) does the following:
\begin{itemize}
\item $[\text{Step 1}] \mX, \mC^{(0)} \to \text{assign/center-update} \to \mC^{(1)}$,
\item $[\text{Step 2}] \mX, \mC^{(1)} \to \text{assign/center-update} \to \mC^{(2)}$,
\item $\cdots$
\item $[\text{Step $T$}] \mX, \mC^{(T-1)} \to \text{assign/center-update} \to \mC^{(T)}$.
\end{itemize}

With a trained model $\cmodel(\cdot, \cdot; \Theta)$ which uses token embeddings (\oh/\bn) and attention based updates in \eqref{eq:updates}, we can perform $T$ iterations of multi-step clustering in the same spirit as the above iterative Lloyd's algorithm by applying the same model $\cmodel$ on evolving inputs (data and current centers):
\begin{itemize}
\item $[\text{Step 1}] \mX, \mC^{(0)} \to \cmodel(\mX, \mC^{(0)}; \Theta) \to \mC^{(1)}$,
\item $[\text{Step 2}] \mX, \mC^{(1)} \to \cmodel(\mX, \mC^{(1)}; \Theta) \to \mC^{(2)}$,
\item $\cdots$
\item $[\text{Step $T$}] \mX, \mC^{(T-1)} \to  \cmodel(\mX, \mC^{(T-1)}; \Theta) \to \mC^{(T)}$.
\end{itemize}

This follows the iterative nature of Lloyd's algorithm, and this is what is evaluated in \cref{fig:multi-step} [right], \cref{fig:vardist-gen}, \cref{fig:varfam-gen}, and \cref{fig:lengen-nsamples}. With single-step clustering, we refer to the fact that, during training, this model is executed on each training clustering task for only 1 step, and the training loss is computed with $\mC^{(1)}$, the output of the [Step 1] update. This makes the training extremely efficient as we are training a single transformer layer without any recursion. However, we see that this model learns some clustering mechanisms as we can reapply it for $T$ and progressively improve clustering quality, much like applying the "assign/center-update" step in Lloyd's algorithm iteratively.

\Cref{fig:multi-step} [left] is ablating the role of the token embedding in $\cmodel$ by evaluating an alternate form of multi-step clustering that only applies the attention updates in \eqref{eq:updates} iteratively as follows:
\begin{itemize}
\item $[\text{Step 0}] \mX, \mC^{(0)} \to \text{\eqref{eq:tk-emb-v1}/\eqref{eq:tk-emb-v2}} \to \bar{\mX}^{(0)}, \bar{\mC}^{(0)}$,
\item $[\text{Step 1}] \bar{\mX}^{(0)}, \bar{\mC}^{(0)} \to \text{eq. \eqref{eq:updates}} \to  \bar{\mX}^{(1)}, \bar{\mC}^{(1)}$,
\item $[\text{Step 2}] \bar{\mX}^{(1)}, \bar{\mC}^{(1)} \to \text{eq. \eqref{eq:updates}} \to  \bar{\mX}^{(2)}, \bar{\mC}^{(2)}$,
\item $\cdots$,
\item $[\text{Step T}] \bar{\mX}^{(T-1)}, \bar{\mC}^{(T-1)} \to  \text{eq. \eqref{eq:updates}} \to  \bar{\mX}^{(T)}, \bar{\mathbf C}^{(T)}$,
\item $[\text{final step}] \bar{\mathbf C}^{(T)} \to \mathbf C^{(T)}$.
\end{itemize}

The results in \cref{fig:multi-step} [left] shows that just applying the attention updates recursively does not improve the clustering quality, and both the token embeddings \eqref{eq:tk-emb-v1}/\eqref{eq:tk-emb-v2} and the attention updates \eqref{eq:updates} in the trained model $\cmodel(\cdot, \cdot; \Theta)$ are essential for the trained transformer to successfully cluster iteratively.

\section{Clustering Task Generation} \label{asec:ctasks}

To train the transformer, we generate clustering tasks $(\mX, \mC) \sim \gD$ from a distribution $\gD$ where $\mX$ denotes the points to be clustered, and $\mC$ denotes the initial cluster centers. We consider mixture distributions (such as mixture of \texttt{normal} distributions or mixture of \texttt{cauchy} distributions, etc) to generate $\mX \in [0,1]^{d \times n}$ in the following manner:
\begin{itemize}
\item Randomly sample $k$ centers $\{\vl_j \sim \gU(0, 1)^d, j \in \iset{k} \}$ in $d$ dimensions uniformly from the unit cube $[0,1]^d$
\item Randomly sample $k$ $d$-dimensional scales $\{\vs_j \sim \gU(0,1)^d, j \in \iset{k} \}$ if using anisotropic distributions else $\vs_j = \mathbf{1}_d \forall j \in \iset{k}$.
\item Randomly sample $k$ weights $\{w_j, j \in \iset{k}\}$ from the $k$-dimensional simplex $\Delta^k$ if using unequally weighted mixtures else $w_j = 1/k \forall i \in \iset{k}$
\item Create a mixture distribution $\gT \triangleq \{(\vl_j, \vs_j, w_j)_{j \in \iset{k}} \}$ and sample $n$ samples $\vx_i \sim \gT$ in $d$ dimensions to get $\hat{\mX} \in \R^{d \times n}$.
\item $\mX \gets \texttt{MinMaxScaler}(\hat{\mX})$
\end{itemize}

Note that, as the data dimensionality $d$ increases, the diameter of the set of samples increases. Given the set of samples, we randomly sample $k$ points from $\mX$ to define the set of initial centers $\mC \in [0,1]^{d \times k}$.

The motivation behind generating clustering tasks where $\mX \in [0,1]^{d \times n}, \mC \in [0,1]^{d \times k}$ is that this aligns with our theoretical analyses, and any general $k$-means clustering problem with $\hat\mX \in \R^{d \times n}, \hat\mC \in \R^{d \times k}$ can be mapped accordingly given the translation and scaling invariance of the $k$-means clustering problem. Let $\vv_{\textsf{lb}} \in \R^d$ be the vector of the per-dimension minimum values for all the points in $\hat\mX$ and $\hat\mC$ combined. Then setting $\bar\mX \gets \hat\mX - \vv_{\textsf{lb}} \mathbf{1}_n^\top, \bar{\mC} \gets \hat\mC - \vv_{\textsf{lb}} \mathbf{1}_k^\top$ moves all the points and centers to the nonnegative quadrant $\R_{\geq 0}^d$. Let $v_{\textsf{ub}} \in \R_+$ be the single largest scalar value in $\bar\mX, \bar\mC$. We assume that this $v_{\textsf{ub}} > 0$, otherwise the problem is trivial with $\hat\mX = \mathbf{0}_{d \times n} , \mC = \mathbf{0}_{d \times k}$. Then setting $\mX \gets \bar\mX / v_{\textsf{ub}}, \mC \gets \bar\mC / v_{\textsf{ub}}$ ensures that $\mX \in [0,1]^{d \times n}, \mC \in [0,1]^{d \times k}$ while the $k$-means problem remains the same due to the translation and scaling invariance.

\section{Computational Resources and Implementation} \label{asec:compute}

All experiments were executed on a Intel i7 Core CPU (16 threads, 64GB memory), and a NVIDIA V100 GPU (32GB memory). Each experiment was executed with 10 random seeds and all results are aggregated across these 10 trials. The implementation is in Pytorch and available at this \href{https://github.com/rithram/kmeans-trf}{GitHub repository} (\url{https://github.com/rithram/kmeans-trf}). We train models with 300 different configurations --- 50 for each of the 6 different combinations of regularized $\min$, softmax vs sparsemax, and initial token embedding, one-hot (\oh) vs binary (\bn) vs none (\na). With 10 repetitions for each, we train a total of 3000 models, each taking around 15-20 minutes to train, for a total training compute time of 1000 hours. The loss surface generation for \cref{fig:lsurfs:normal:k6L2} took about 15 minutes per model for a total of 1 hour of compute. Each of the inference evaluations in \cref{sec:emp} took around 5 minutes each for total of $(25+2+4)\times 5 \approx 150$ minutes. The timed runs in \cref{fig:enc-profile} and \cref{tab:enc-profile} in \cref{sec:expressivity} and \cref{fig:train-time} in \cref{asec:ablations} are performed on a NVIDIA RTX PRO 3000 Blackwell Generation Laptop GPU.

\section{Ablative Experiments} \label{asec:ablations}

In this section, we present ablations across various axes, presenting the training loss trajectory as the training progresses, and the relative validation loss. Given a transformer parameterized with $\Theta$, for an input $(\mX, \mC) \sim \gD$ sampled from the distribution $\gD$ of clustering tasks, we will denote the updated centers after a single transformer layer as $\bar{\mC}(\Theta)$. Then, the training loss $\E_{(\mX, \mC) \sim \gD} \tilde{\Ls}_\Omega^\tau(\mX, \bar{\mC}(\Theta) )$ (defined in \cref{eq:learn-obj}) is the smoothed upper-bound of the desired $k$-means loss $\E_{(\mX, \mC) \sim \gD} \Ls(\mX, \bar{\mC}(\Theta) )$ (defined in \cref{eq:kmeans}) with $\Ls(\mX, \bar{\mC}(\Theta) ) \leq \tilde{\Ls}_\Omega^\tau (\mX, \bar{\mC}(\Theta) )$, and the inequality approaching equality as the temperature $\tau \to 0$. We define the {\em relative validation loss} as the ratio of the $k$-means objective with the centers $\bar{\mC}(\Theta)$ obtained with a single transformer layer to the $k$-means objective with the centers $\mC^\star$ generated from a single iteration of Lloyd's algorithm (\Cref{alg:lloyds}) with the same set of points $\mX$ and initial centers $\mC$. Thus the relative validation error is given by
\begin{equation} \label{eq:rel-val-loss}
\E_{(\mX, \mC) \sim \gD} \frac{\Ls( \mX, \bar{\mC}(\Theta) ) }{ \Ls (\mX, \mC^\star ) }.
\end{equation}
Note that, for estimating the training loss and the relative validation loss, we use different sets of clustering tasks, and the clustering tasks in the validation set are not used for training. A relative validation error of less than 1 denotes that the single transformer layer outperforms Lloyd's algorithm (on average) when considered for a single iteration.

Each figure in these ablations has 6 curves corresponding to (i)~the $\min$-regularizer $\Omega$ --- the negative Shannon entropy or NE-regularizer (denoted by {\bf NE}) vs the L2-regularizer (denoted by {\bf L2}), and (ii)~the use of the one-hot initial token embeddings $(\eohx, \eohc)$ defined in \eqref{eq:tk-emb-v1} with $\demb = d + k$ and denoted by {\bf OH} vs the binary initial token embeddings $(\ebnx, \ebnc)$ defined in \eqref{eq:tk-emb-v2} with $\demb = d + \ceil{\log_2 k }$ and denoted by {\bf BN} vs no initial token embedding with $\demb = d$ denoted by {\bf NA}. Thus we have 6 lines corresponding to labels \textcolor{xkcdAppleGreen}{BN:NE}, \textcolor{xkcdOliveGreen}{BN:L2}, \textcolor{xkcdAquaMarine}{NA:NE}, \textcolor{xkcdPeacockBlue}{NA:L2}, \textcolor{xkcdSalmon}{OH:NE}, \textcolor{xkcdDarkOrange}{OH:L2}. Each experiment is executed for 10 trials with different random seeds, and each curve reports the median performance metrics, with a ribbon around it denoting the inter-quartile range of the performance metrics.

\begin{figure}[htb]
\centering
\begin{subfigure}{0.49\textwidth}
\centering
\includegraphics[width=\textwidth]{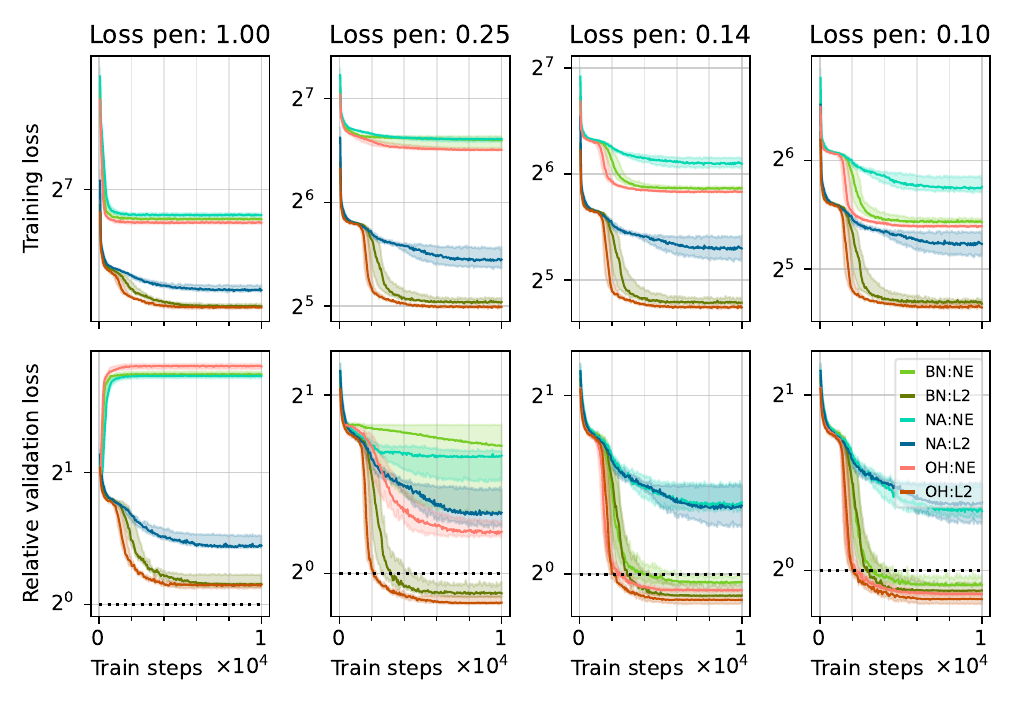}
\caption{Clustering tasks in $d = 4$ dimensions.}
\label{fig:ablations:loss-itemp-d4}
\end{subfigure}
~
\begin{subfigure}{0.49\textwidth}
\centering
\includegraphics[width=\textwidth]{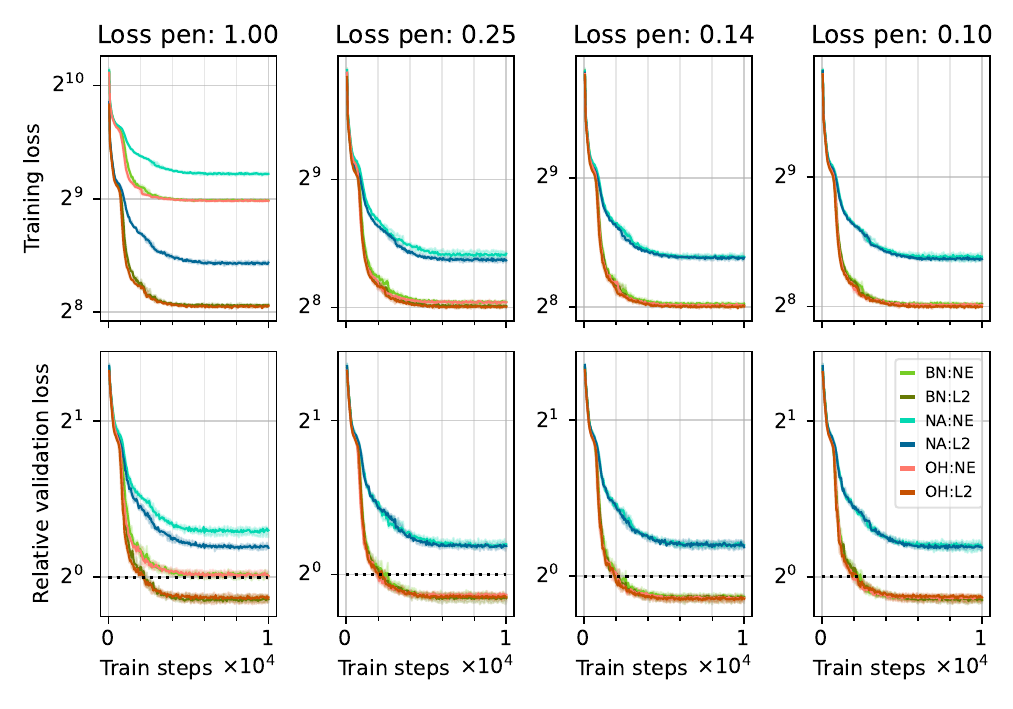}
\caption{Clustering tasks in $d = 32$ dimensions.}
\label{fig:ablations:loss-itemp-d32}
\end{subfigure}
\caption{Training loss convergence (top row) and relative validation loss from \cref{eq:rel-val-loss} (bottom row) for varying values of the $\min$-regularization penalty $\tau$ in the training loss $\tilde{\Ls}_\Omega^\tau$ (the smoothed upper-bound of the discrete non-differentiable $k$-means objective $\Ls$). This is an extended version of \Cref{fig:tau-var}. We consider $\tau \in \{ 1, 0.25, 0.14, 0.1 \}$. These set of results correspond to a training with clustering tasks involving $n = 512$ points and $k = 6$ clusters with $d = 4$ (\cref{fig:ablations:loss-itemp-d4}) and $d = 32$ (\cref{fig:ablations:loss-itemp-d32}) dimensions. Here the initial learning rate $\eta=0.01$ for the Adam optimizer with $M = 10000$ training steps with a task batch-size $B = 32$.}
\label{fig:ablations:loss-itemp}
\end{figure}

\paragraph{Varying the $\min$-regularization penalty $\tau$ in the smoothed loss $\tilde{\Ls}_\Omega^\tau$.}
In \Cref{fig:ablations:loss-itemp}, we present the effect of varying the $\min$-regularization penalty $\tau$ in the smoothed upper-bound $\tilde{\Ls}_\Omega^\tau$, highlighting that, as the penalty $\tau$ decreases, the training loss (the upper-bound) becomes tighter (and thus has lower values), and the performance on the validation tasks actually starts matching or even outperforming 1-step Lloyd's algorithm. Note that, for the larger value of $d$, the number of dimensions, the $\min$-regularization penalty $\tau$ does not need to be too low for strong performance. This result also highlights that the L2-regularized $\min$ provides tighter upper-bounds, which translate to better validation performance, with the difference most apparent at large values of $\tau$ and the smaller data dimensionality of $d=4$, when the penalization is stronger in general, and the upper-bounds are looser but smoother. The training loss converges faster at the higher values of $\tau$, though the validation performance highlights that the faster training loss convergence does not necessary guarantee strong validation performance because of the gap between the smoothed upper-bound and the actual discrete clustering objective we care about.

\begin{figure}[htb]
\centering
\begin{subfigure}{0.49\textwidth}
\centering
\includegraphics[width=\textwidth]{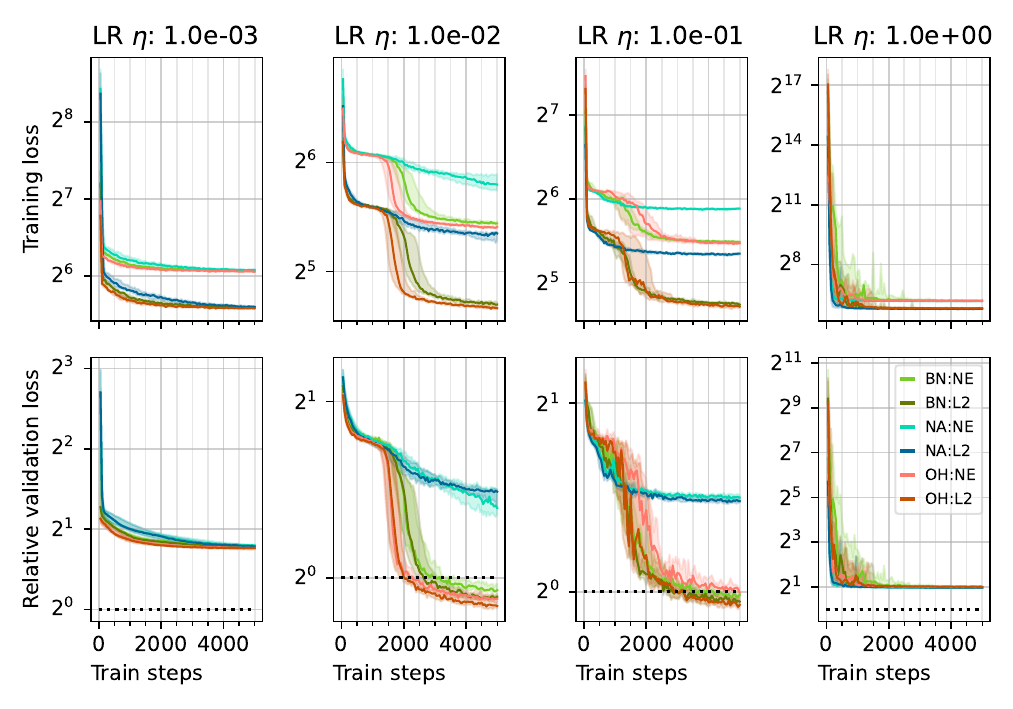}
\caption{Clustering tasks in $d = 4$ dimensions.}
\label{fig:ablations:lrate-d4}
\end{subfigure}
~
\begin{subfigure}{0.49\textwidth}
\centering
\includegraphics[width=\textwidth]{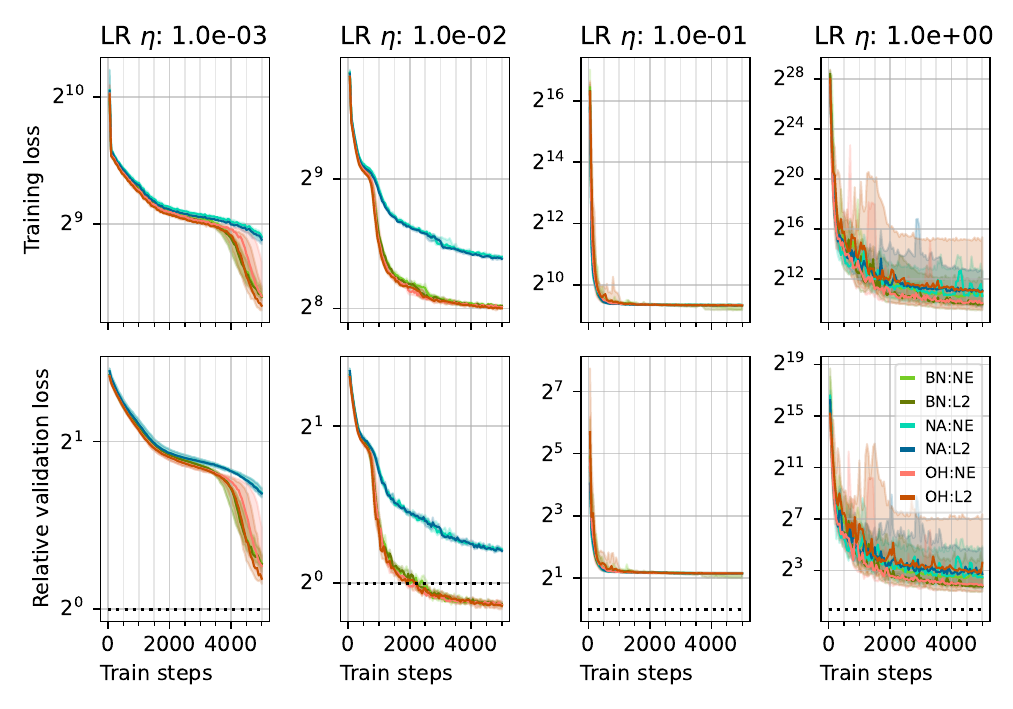}
\caption{Clustering tasks in $d = 32$ dimensions.}
\label{fig:ablations:lrate-d32}
\end{subfigure}
\caption{Training loss convergence (top row) and relative validation loss from \cref{eq:rel-val-loss} (bottom row) for varying values of the initial learning rate $\eta$ for the Adam optimizer selected from $\{0.001, 0.01, 0.1, 1.0\}$ with $M=5000$ training steps with a task batch-size $B=32$. These set of results correspond to a training with clustering tasks involving $n = 512$ points and $k = 6$ clusters with $d = 4$ (\cref{fig:ablations:lrate-d4}) and $d = 32$ (\cref{fig:ablations:lrate-d32}) dimensions with a $\min$-regularization penalty $\tau = 0.1$.}
\label{fig:ablations:lrate}
\end{figure}

\paragraph{Varying the learning rate $\eta$ in \Cref{alg:training}.}
In \Cref{fig:ablations:lrate}, we present the effect of varying the initial learning rate in the Adam optimizer. The results indicate that, for low and moderately high dimensions $d$, a low learning rate of 0.001 shows monotonic convergence but is slower than a learning rate of 0.01. On the other hand, larger values of learning rates such as 0.1 and 1.0 show that both the training and the validation losses do not decrease monotonically, especially at the higher $d=32$ dimensions. The larger learning rates ($\eta = 0.1, 1.0$ for $d=32$ and $\eta=1.0$ for $d=4$) converge to a worse local minima, and thus do not generalize well. In all cases where the training converges favourably, the L2-regularized $\min$ continues to converge to a better minima (both in terms of training loss and validation) than the NE-regularized $\min$ (controlling for all other factors).

\begin{figure}[htb]
\centering
\includegraphics[width=\textwidth]{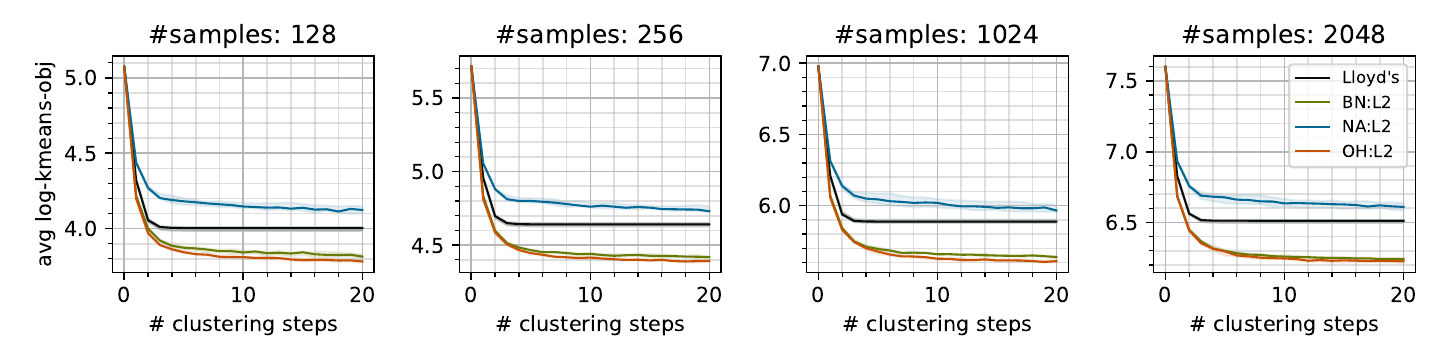}
\caption{{\bf ``Length-generalization'' of learned models.} We consider models trained with $n=512, d = 32, k=10, \tau = 0.1$ utilized for $T=20$ step clustering on tasks with $\{128, 256, 1024, 2048\}$ points. We only consider L2-regularized $\min$ and the three forms of token embeddings, \oh, \bn and \na as described in \Cref{fig:tau-omega-demb-dim-analysis}.
The performance of Lloyd's algorithm is plotted in black.}
\label{fig:lengen-nsamples}
\end{figure}

\paragraph{Evaluating the ``length generalization'' of models, learning on one $n$, and generalizing to other values of $n$}
Here we evaluate the ability of these trained models to length generalize --- the ability of models to be trained on problems of a certain size, but generalize to problems of different (often larger) sizes. In the context of clustering, we train the models on (single-step) clustering tasks with $n=512$ samples (for a mixture of \texttt{normal} distributions), and then evaluate their (multi-step) clustering performance on problem with varying number of samples (again from the same distribution as it was trained on) in \Cref{fig:lengen-nsamples}. The results show that {\em the learned model can generalize to larger clustering problems}. The models using the \eqref{eq:tk-emb-v1} and \eqref{eq:tk-emb-v2} token embeddings outperform Lloyd's iterations, both for smaller problems (with 128 and 256 samples) and larger problems (with 1024 and 2048 samples).

\begin{figure}[htb]
\centering
\begin{subfigure}{0.49\textwidth}
\centering
\includegraphics[width=\textwidth]{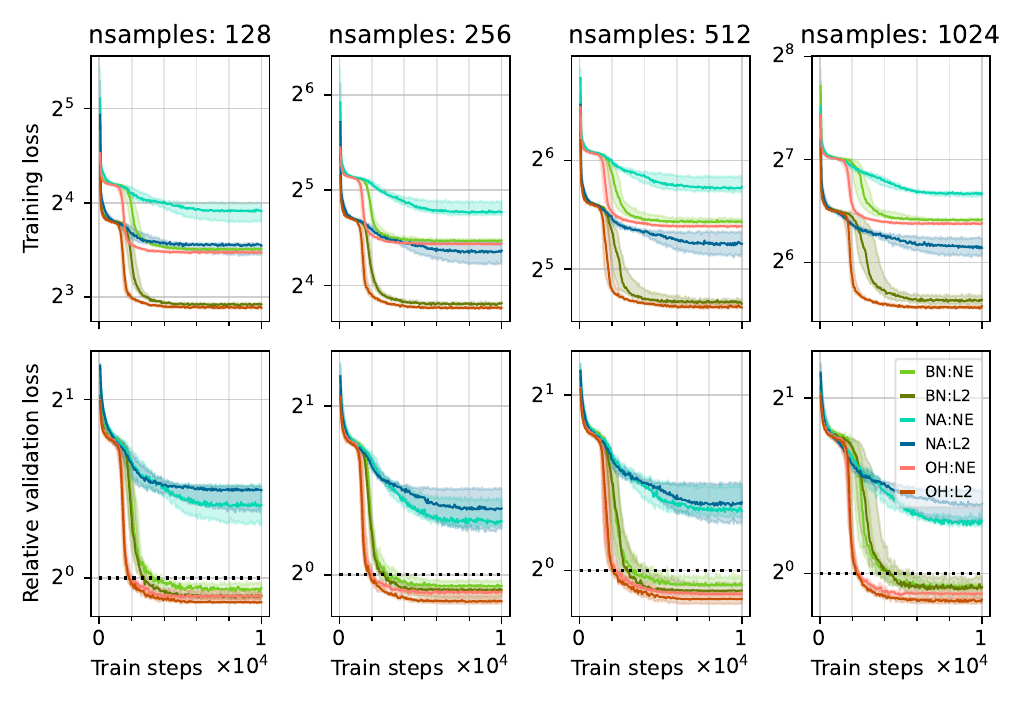}
\caption{Clustering tasks in $d = 4$ dimensions.}
\label{fig:ablations:nsamples-d4}
\end{subfigure}
~
\begin{subfigure}{0.49\textwidth}
\centering
\includegraphics[width=\textwidth]{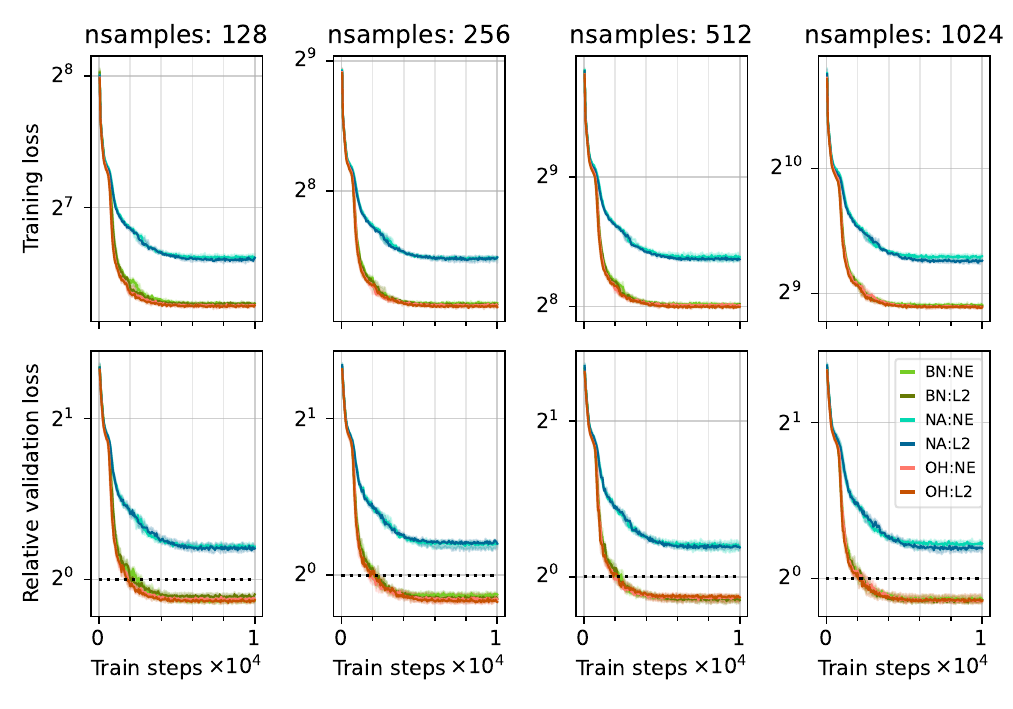}
\caption{Clustering tasks in $d = 32$ dimensions.}
\label{fig:ablations:nsamples-d32}
\end{subfigure}
\caption{Training loss convergence (top row) and relative validation loss from \cref{eq:rel-val-loss} (bottom row) for varying values of the number of points $n$ in the clustering tasks used for training the transformers, with $n \in \{ 128, 256, 512, 1024 \}$. These set of results correspond to a training with clustering tasks involving $k = 6$ clusters with $d = 4$ (\cref{fig:ablations:nsamples-d4}) and $d = 32$ (\cref{fig:ablations:nsamples-d32}) dimensions. Here we use the initial learning rate $\eta=0.01$ for the Adam optimizer with $M = 10000$ training steps, a task batch-size $B = 32$, and the $\min$-regularization penalty $\tau=0.1$.}
\label{fig:ablations:nsamples}
\end{figure}

\paragraph{Varying the number of points $n$ to be clustered in the training tasks.}
In \Cref{fig:ablations:nsamples}, we study effect of varying the number of points $n$ in the clustering tasks $(\mX, \mC)$ with $\mX \in \R^{d \times n}$ used during training, selecting $n \in \{ 128, 256, 512, 1024 \}$ for clustering tasks in $d \in \{4, 32\}$ dimensional Euclidean space. The results indicate that the number of points in the clustering tasks do not seem to have a significant effect on the training convergence or in-distribution generalization for any data dimensionality $d$, $\min$-regularizer $\Omega$ or initial token embedding scheme (\oh vs \bn vs \na). This combined with our length-generalization results in \Cref{fig:lengen-nsamples} indicate that a good strategy would be to train the transformer with clustering tasks with a small $n$ as that would be computationally efficient, and still be able to use this transformer for larger clustering tasks with larger number of points. Note that, for a fixed $d$ and $k$ and initial token embedding choice (\oh vs \bn vs \na), the number of learnable parameters remains the same as we vary $n$ in the training clustering tasks.

\begin{figure}[htb]
\centering
\begin{subfigure}{0.49\textwidth}
\centering
\includegraphics[width=\textwidth]{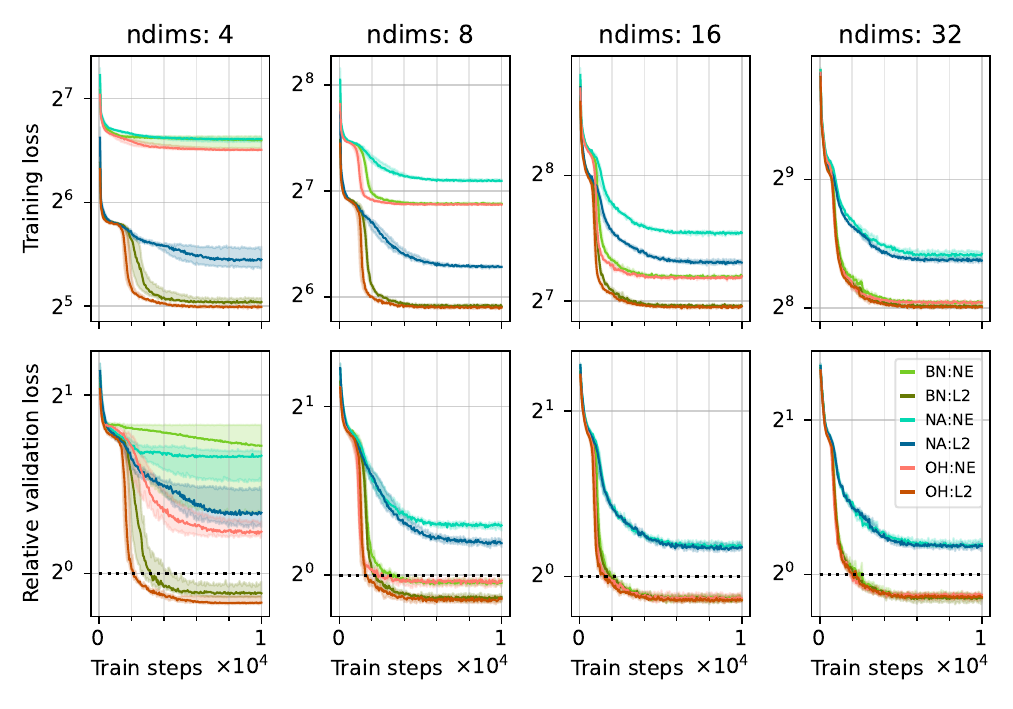}
\caption{$\min$-regularization penalty $\tau = 0.25$.}
\label{fig:ablations:ndims-litemp4}
\end{subfigure}
~
\begin{subfigure}{0.49\textwidth}
\centering
\includegraphics[width=\textwidth]{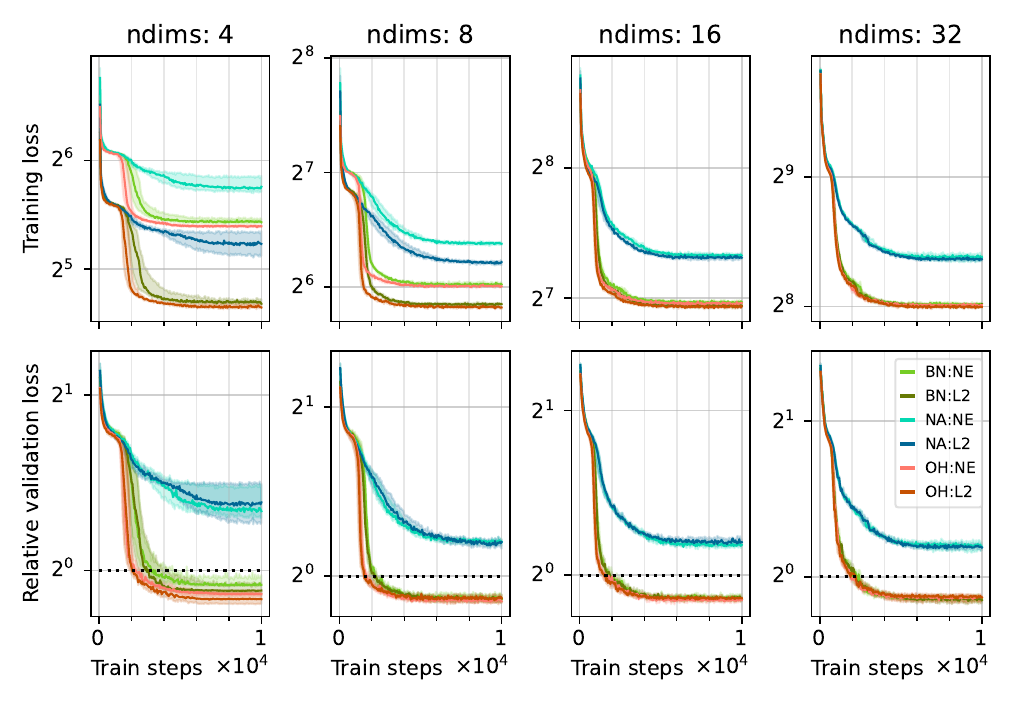}
\caption{$\min$-regularization penalty $\tau = 0.1$.}
\label{fig:ablations:ndims-litemp10}
\end{subfigure}
\caption{Training loss convergence (top row) and relative validation loss from \cref{eq:rel-val-loss} (bottom row) for varying values of the dimensionality $d$ of the clustering tasks used for training the transformers, with $d \in \{ 4, 8, 16, 32 \}$. This is an extended version of \Cref{fig:dim-var}. These set of results correspond to a training with clustering tasks involving $n=512$ points, $k = 6$ clusters, trained with the smoothed loss with a loss $\min$-regularizer penalty $\tau = 0.25$ (\cref{fig:ablations:ndims-litemp4}) and $\tau=0.1$ (\cref{fig:ablations:ndims-litemp10}). Here the initial learning rate $\eta = 0.01$ for the Adam optimizer with $M=10000$ training steps, each with a task batch-size $B=32$.}
\label{fig:ablations:ndims}
\end{figure}

\paragraph{Varying the data dimensionality $d$ of the clustering tasks.}
In \Cref{fig:ablations:ndims}, we study the effect of the data dimensionality $d$ of the clustering tasks used to train the transformer. We consider two different values for the $\min$-regularization penalty $\tau$, with smaller value of $\tau$ implying tighter upper-bounds to the discrete clustering loss. For the larger penalty value of $\tau = 0.25$ in \Cref{fig:ablations:ndims-litemp4}, we can see that the L2-regularized $\min$ based smoothed loss converges and generalizes better than the NE-regularized $\min$ especially at a low $d  = 4$; also note that the \oh token embedding converges faster than the \bn token embeddings as previously discussed (though both generalize comparably). However, as the data dimensionality $d$ increases, this difference between L2-regularized and NE-regularized $\min$ vanishes. This is because, as we consider all $\mX \in [0,1]^{d \times n}$ and $\mC \in [0,1]^{d \times k}$, the pairwise squared Euclidean distances generally increases with $d$, and thus the regularized $\min$ produces tighter upper-bounds with increasing $d$ for a fixed $\min$-regularization penalty $\tau$. We will discuss this further in \Cref{asec:reg-argmax}. Note that the number of parameters in the transformer model grows quadratically with $d$, but we do not see any clear adverse effect of this increase in the number of learnable parameters.

\begin{figure}[!!htb]
\centering
\begin{subfigure}{0.49\textwidth}
\centering
\includegraphics[width=\textwidth]{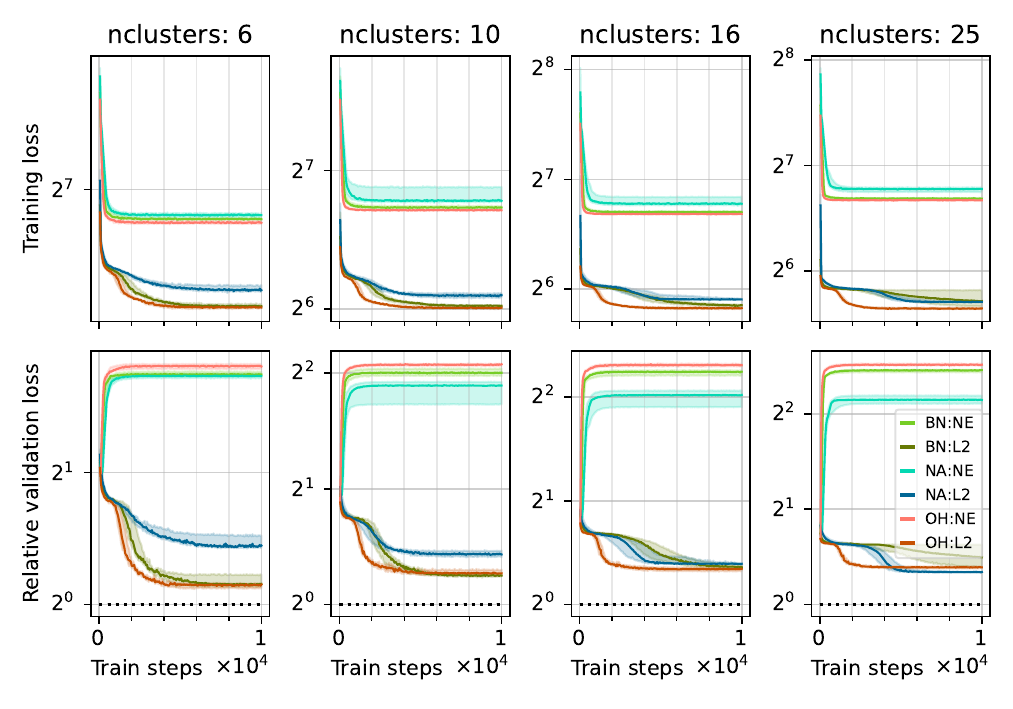}
\caption{$d = 4, \tau = 1$.}
\label{fig:ablations:nclusters-d4litemp1}
\end{subfigure}
~
\begin{subfigure}{0.49\textwidth}
\centering
\includegraphics[width=\textwidth]{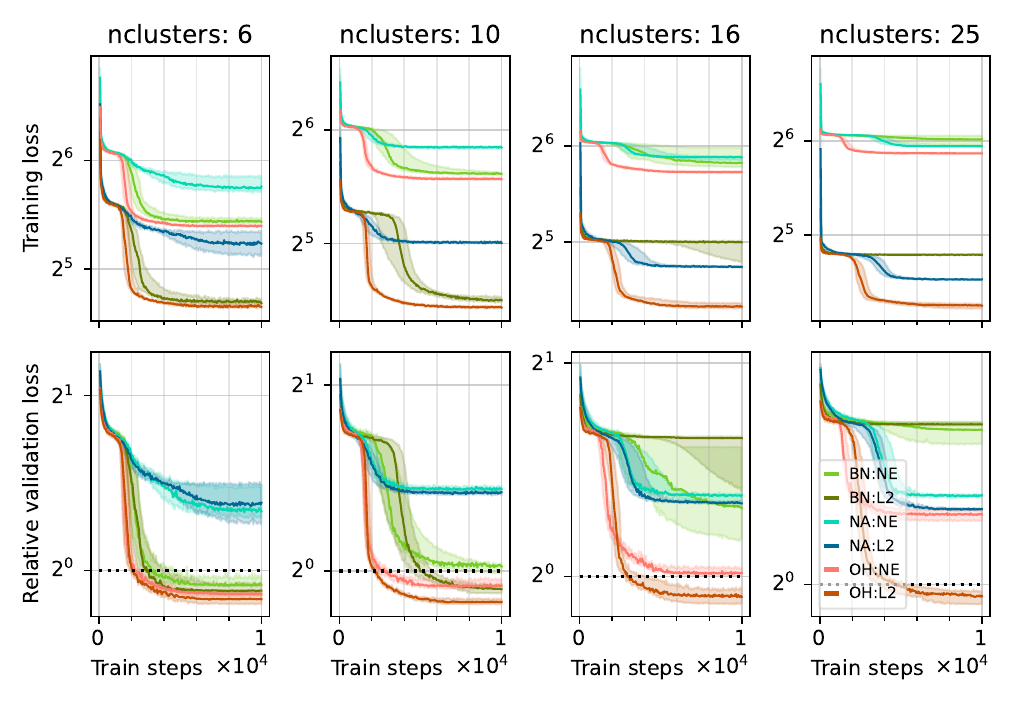}
\caption{$d = 4, \tau = 0.1$.}
\label{fig:ablations:nclusters-d4litemp10}
\end{subfigure}
~
\begin{subfigure}{0.49\textwidth}
\centering
\includegraphics[width=\textwidth]{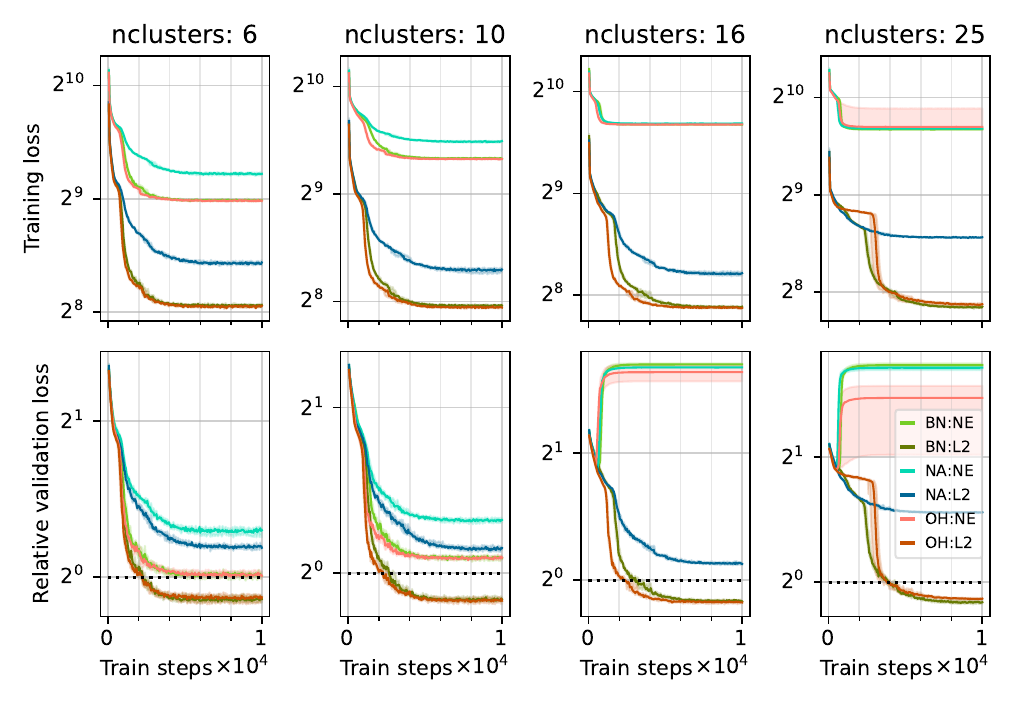}
\caption{$d = 32, \tau = 1$.}
\label{fig:ablations:nclusters-d32litemp1}
\end{subfigure}
~
\begin{subfigure}{0.49\textwidth}
\centering
\includegraphics[width=\textwidth]{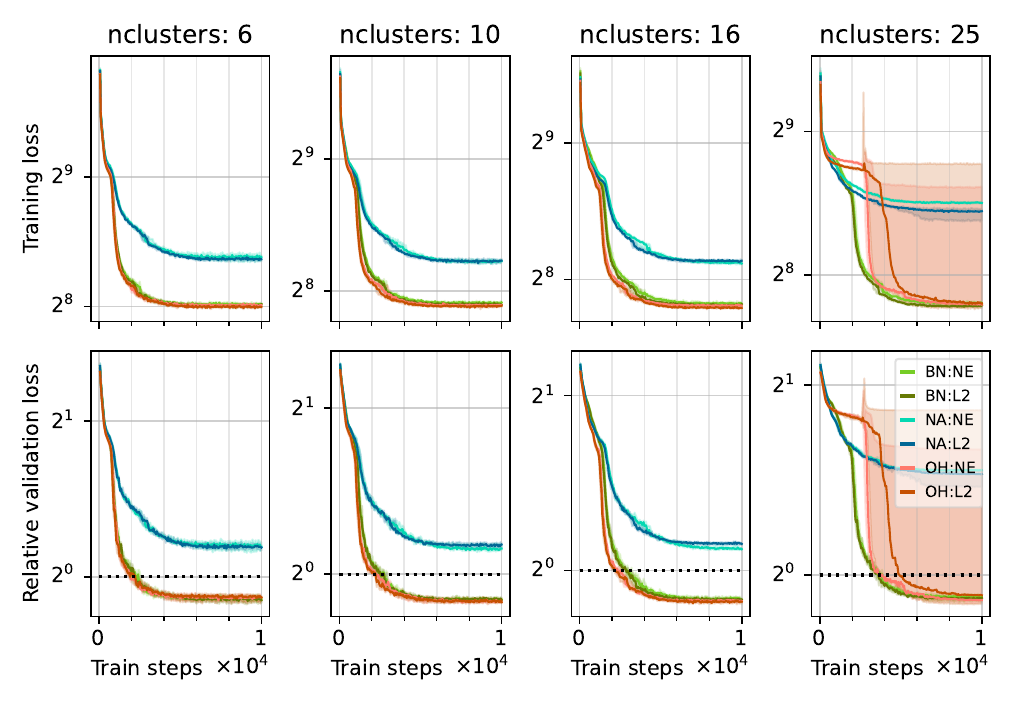}
\caption{$d = 32, \tau = 0.1$.}
\label{fig:ablations:nclusters-d32litemp10}
\end{subfigure}
\caption{Training loss convergence (top row) and relative validation loss from \cref{eq:rel-val-loss} (bottom row) for varying values of the number of clusters $k$ in the clustering tasks used for training the transformers, with $k \in \{ 6, 10, 16, 25 \}$. These set of results correspond to a training with clustering tasks with $n = 512$ points. To highlight the interplay of the dimensionality $d$ and the $\min$-regularization penalty $\tau$, we consider four cases: $d = 4, \tau = 1$ (\cref{fig:ablations:nclusters-d4litemp1}),  $d = 4, \tau = 0.1$ (\cref{fig:ablations:nclusters-d4litemp10}),  $d = 32, \tau = 1$ (\cref{fig:ablations:nclusters-d32litemp1}), and  $d = 32, \tau = 0.1$ (\cref{fig:ablations:nclusters-d32litemp10}). In all cases, the initial learning rate $\eta = 0.01$ for the Adam optimizer with $M=10000$ training steps each with a task batch-size $B = 32$.}
\label{fig:ablations:nclusters}
\end{figure}

\paragraph{Varying the number of clusters $k$ in the training tasks.}
In \Cref{fig:ablations:nclusters}, we carefully study the effect of varying the number of clusters $k$ in the training tasks to learning convergence and generalization. To understand the interplay of the number of clusters $k$, data dimensionality $d$ and the $\min$-regularization penalty $\tau$, we consider a pair of values of $d$ and for $\tau$, and for each of the 4 combinations, we vary $k \in \{6, 10, 16, 25 \}$.
First, in \Cref{fig:ablations:nclusters-d4litemp1}, we consider a low $d = 4$ and high penalty $\tau=1$, implying a looser smoother upper-bound. In this case, the L2-regularized $\min$ is able to learn and somewhat generalize, but is unable to match the results of a single Lloyd's iteration (the black dotted line) for any value of $k$; the NE-regularized $\min$ is unable to learn at all as is evidenced by the high relative loss on the single-step validation clustering tasks. One observation is that the convergence (and generalization) of \bn token embedding relatively deteriorates with increasing $k$ as one would expect from our results in \Cref{thm:loss-lip-short}, where the Lipschitz constant depends on $\sqrt{d+\ceil{\log_2 k}}$, and thus increases with $k$ (thereby hurting convergence). In contrast, the \oh token embedding does not seem to be affected by the increase in $k$, which corroborates our theoretical results in \Cref{thm:loss-lip-short} where the Lipschitz constant would depend on $(d+1)/\sqrt{d+k}$, and thus is not adversely affected by increasing $k$.
Fixing the dimension $d = 4$, and lowering the $\min$-regularization $\tau = 0.1$ in \Cref{fig:ablations:nclusters-d4litemp10}, thereby using a tighter (less smooth) upper-bound for training, we see that the NE-regularized $\min$ is able to somewhat learn and generalize for smaller values of $k$, with its performance deteriorating with increasing $k$; L2-regularized $\min$ with \oh token embedding does significantly better and outperforms Lloyd's for all values of $k$ considered. For smaller values of $k$, the \bn token embedding is able to outperform Lloyd's but converges and generalizes slower than \oh as our theory suggests; for larger values of $k$, the convergence and generalization of \bn token embedding slows down considerably, and even underperforms the model with no token embedding (\na).
As we increase the data dimensionality to $d = 32$ and utilize a larger $\min$-regularization penalty $\tau = 1$ in \Cref{fig:ablations:nclusters-d32litemp1}, we see that L2-regularized $\min$ with \oh and \bn token embeddings outperforms Lloyd's for all values of $k$. However, with NE-regularized $\min$, the smoothed upper-bound $\tilde{\Ls}_\Omega^\tau$ for a fixed $\tau$ gets progressively looser with larger values of $k$, and this is evident in the worsening validation losses.
Finally, fixing the data dimensionality to $d=32$ and lowering the $\min$-regularization penalty to $\tau = 0.1$ for a tighter clustering loss upper-bound in \Cref{fig:ablations:nclusters-d32litemp10}, we see that both L2-regularized and NE-regularized $\min$ for both \oh and \bn token embeddings outperform Lloyd's on the validation tasks. The only thing of note here is that, at the largest value of $k=25$, the model with the \oh token embeddings experiences intermediate numerical issues for 3/10 random trials in \cref{fig:ablations:nclusters-d32litemp10} which throws the model off, and the training does not recover in the allotted $10000$ steps, thereby broadening the inter-quartile range unfavourably; note that the median performance continues to outperform Lloyd's. We would like to note a couple of points regarding this:
\begin{itemize}
\item First, we did not use the standard practice of gradient clipping in any of our results, and our results indicate that the learning problem is well-structured enough for the training to easily converge in almost all cases (also highlighted by the loss landscapes in \cref{fig:lsurfs:normal:k6L2}).
\item Second, for the 3/10 random trials where the training gets stuck, adding gradient clipping mitigates the issue of the numerical instability, and these previously failing training runs converge as expected consistent with the other random trials. We presented the results in \cref{fig:ablations:nclusters-d32litemp10} without gradient clipping to be consistent across all evaluations.
\item Finally, this only occurs with the \oh token embedding for both NE (softmax) and L2 (sparsemax) regularization, where $d_\text{emb} = d+k = 57$ with $d=32, k=25$, and the training is not able to find the right basin of attraction in this relatively high $O(d_\text{emb}^2) \sim 3600$ dimensional loss landscape once the training is destabilized. In contrast, for the same setup (same random seed and clustering tasks for training), the training converges fine with the \bn token embedding, where $d_{\text{emb}} = d + \lceil \log_2 k \rceil =  37$ with $d=32,  k = 25, \lceil \log_2 k \rceil = 5$, implying that the loss landscapes are easier to navigate in a relatively lower dimensionality $O(d_\text{emb}^2) \sim 1600$.
\end{itemize}

\begin{figure}[!!!htb]
\centering
\begin{subfigure}{0.49\textwidth}
\centering
\includegraphics[width=\textwidth]{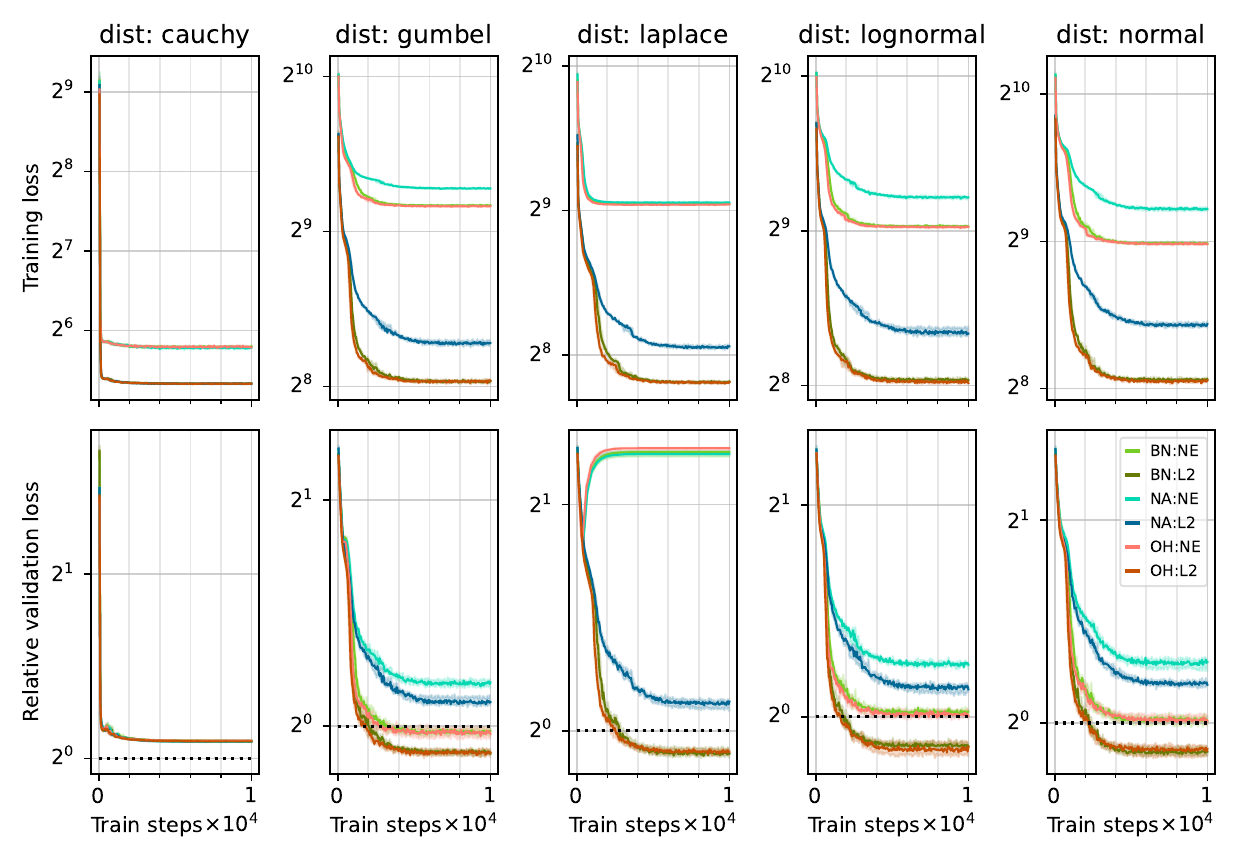}
\caption{$k = 6, \tau = 1$.}
\label{fig:ablations:dist-k6-itemp1}
\end{subfigure}
~
\begin{subfigure}{0.49\textwidth}
\centering
\includegraphics[width=\textwidth]{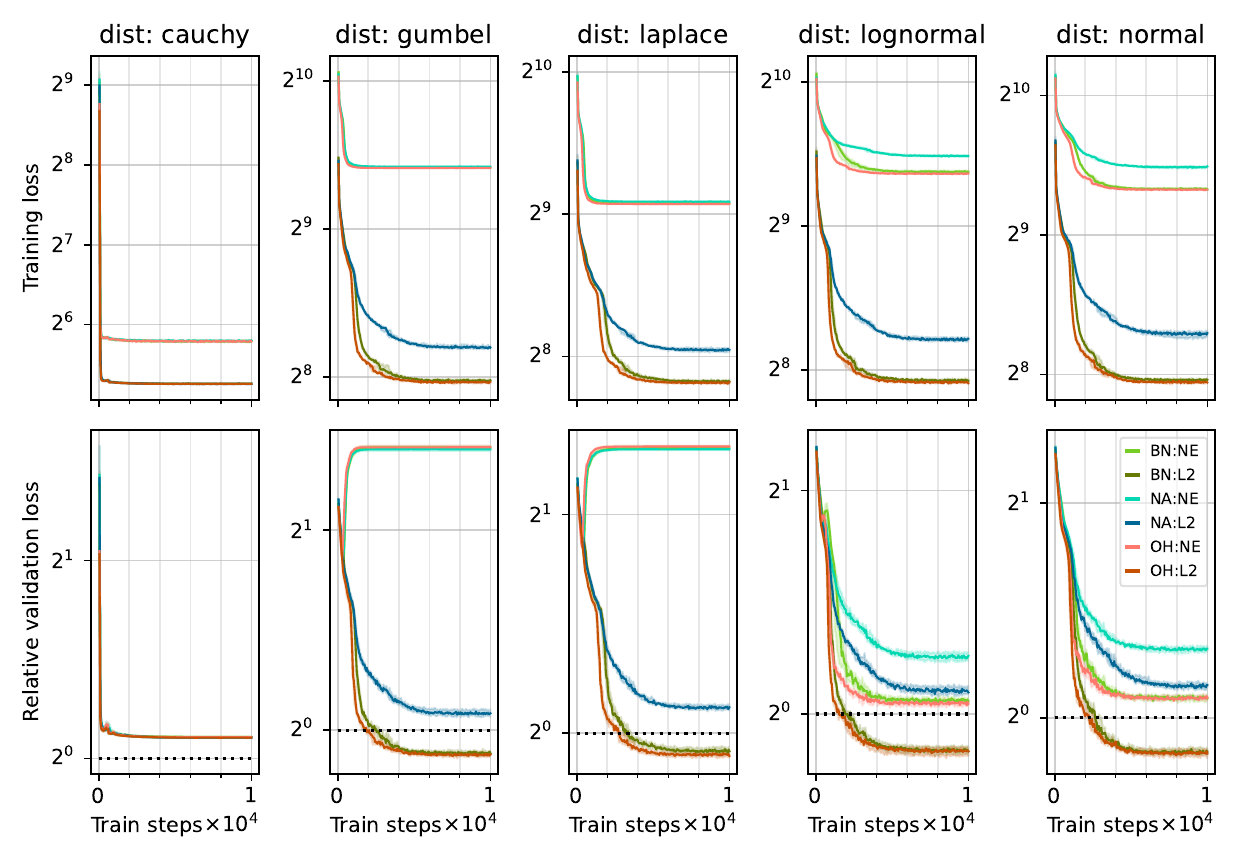}
\caption{$k = 10, \tau = 1$.}
\label{fig:ablations:dist-k10-itemp1}
\end{subfigure}
~
\begin{subfigure}{0.49\textwidth}
\centering
\includegraphics[width=\textwidth]{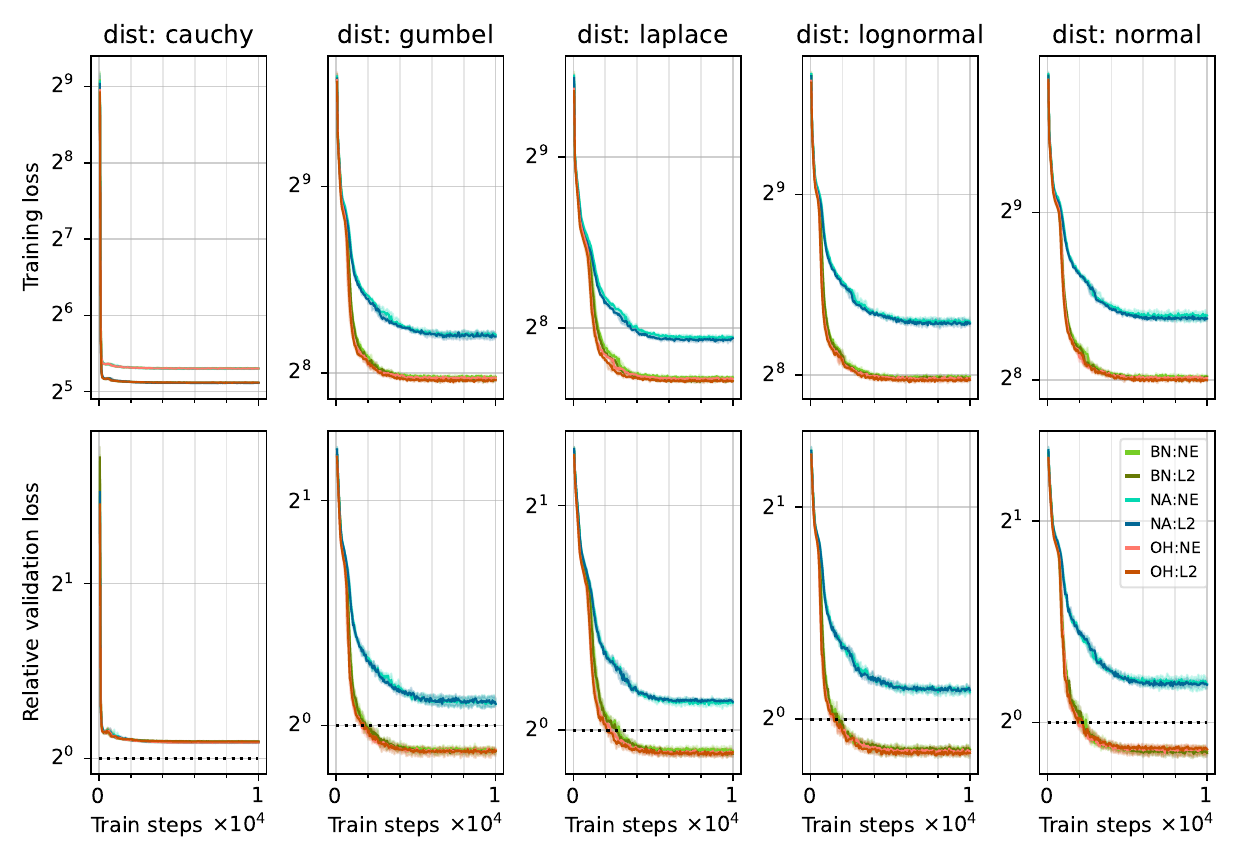}
\caption{$k = 6, \tau = 0.1$.}
\label{fig:ablations:dist-k6-itemp10}
\end{subfigure}
~
\begin{subfigure}{0.49\textwidth}
\centering
\includegraphics[width=\textwidth]{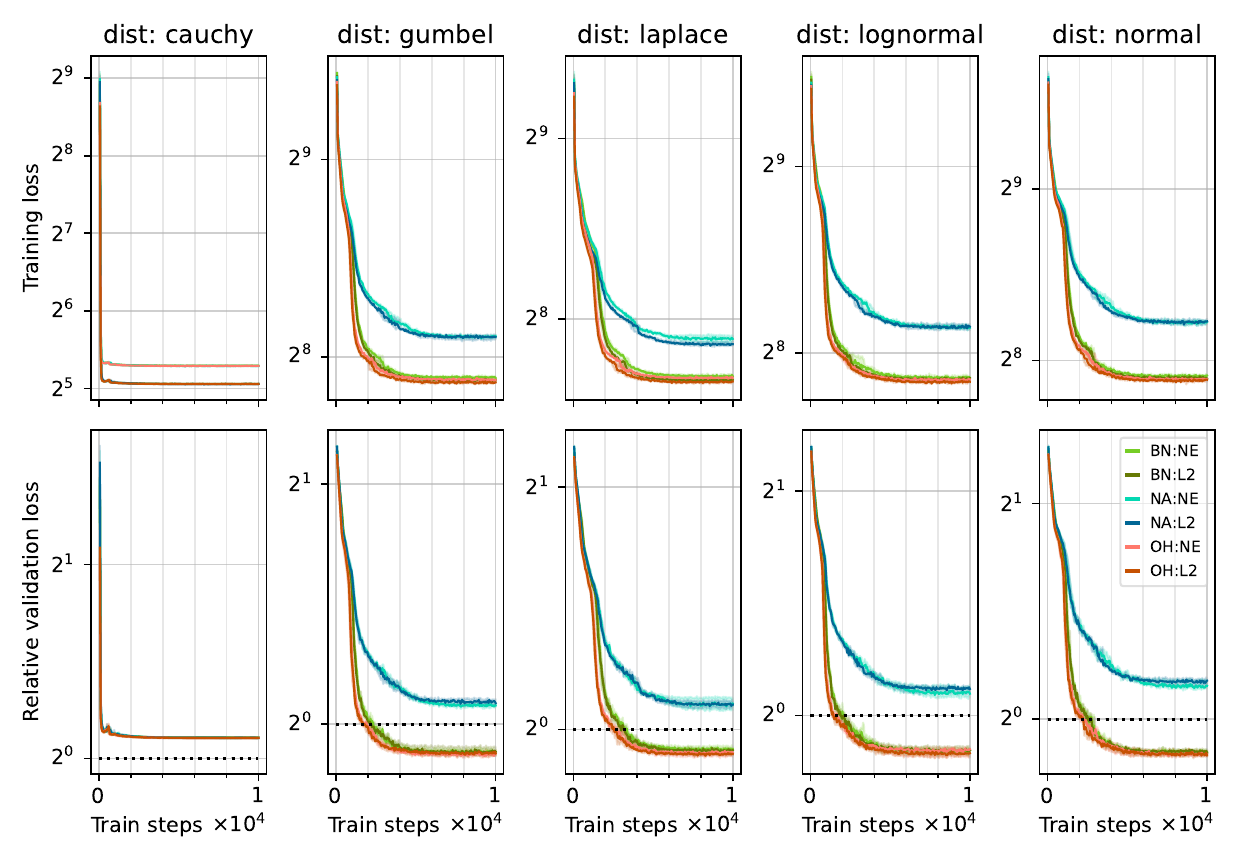}
\caption{$k = 10, \tau = 0.1$.}
\label{fig:ablations:dist-k10-itemp10}
\end{subfigure}
\caption{Training loss convergence (top row) and relative validation loss from \cref{eq:rel-val-loss} (bottom row) for varying the distribution of the points to the clustered. These set of results correspond to a training with clustering tasks with $n = 512$ points in $d=32$ dimensions, and is an extended version of \Cref{fig:dists-var}. To highlight the interplay of the number of clusters $k$ and the loss temperature $\tau$, we consider four cases: $k = 6, \tau = 1$ (\cref{fig:ablations:dist-k6-itemp1}),  $k = 10, \tau = 1$ (\cref{fig:ablations:dist-k10-itemp1}),  $k = 6, \tau = 0.1$ (\cref{fig:ablations:dist-k6-itemp10}), and  $k = 10, \tau = 0.1$ (\cref{fig:ablations:dist-k10-itemp10}). In all cases, the initial learning rate $\eta=0.01$ for the Adam optimizer with $M=10000$ training steps and a per-step task batch-size $B=32$.}
\label{fig:ablations:dist}
\end{figure}

\paragraph{Varying the distribution of the data from which the points are generated for the training tasks.}
In \Cref{fig:ablations:dist}, we build upon \Cref{fig:dists-var}, and provide an extended evaluation of the effect of the data distribution from which the set of points $\mX \in [0,1]^{d \times n}$ are generated for any clustering tasks $(\mX, \mC)$. As with \Cref{fig:dists-var}, we see that our proposed learning scheme is unable to learn from and generalize to clustering tasks generated from a mixture of \texttt{cauchy} distributions under any of the configuarations considered. Focusing on the remaining distributions, namely \texttt{gumbel}, \texttt{laplace}, \texttt{lognormal}, and \texttt{normal}, we see that, at a larger $\min$-regularization of $\tau = 1$ in \Cref{fig:ablations:dist-k6-itemp1} and \Cref{fig:ablations:dist-k10-itemp1}, the L2-regularized $\min$ provides better generalization than the NE-regularized $\min$, with NE-regularized $\min$ failing to generalize at all with the \texttt{gumbel} and \texttt{laplace} distributions; this behaviour is more evident with the larger $k=10$ value in \Cref{fig:ablations:dist-k10-itemp1}. With a smaller $\min$-regularization penalty $\tau = 0.1$ providing tighter clustering loss upper-bounds for training in \Cref{fig:ablations:dist-k6-itemp10} and \Cref{fig:ablations:dist-k10-itemp10}, we see that both NE-regularized and L2-regularized $\min$ with \oh and \bn token embeddings can outperform Lloyd's across all data distributions.

\begin{figure}[htb]
\centering
\begin{subfigure}{0.49\textwidth}
\centering
\includegraphics[width=\textwidth]{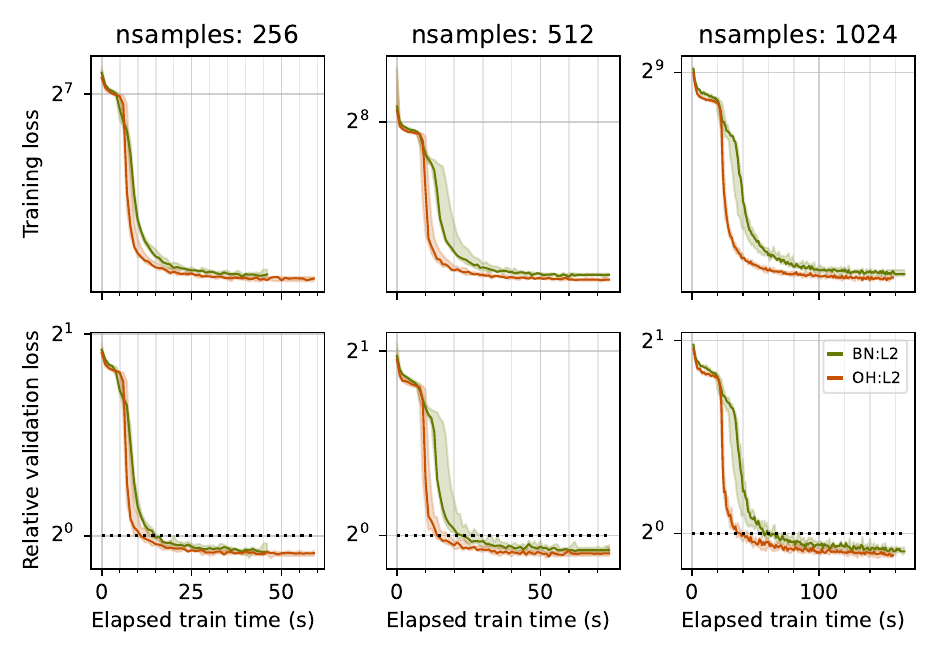}
\caption{$k = 16, d = 16, \tau = 1$.}
\label{fig:train-time:k16d16}
\end{subfigure}
~
\begin{subfigure}{0.49\textwidth}
\centering
\includegraphics[width=\textwidth]{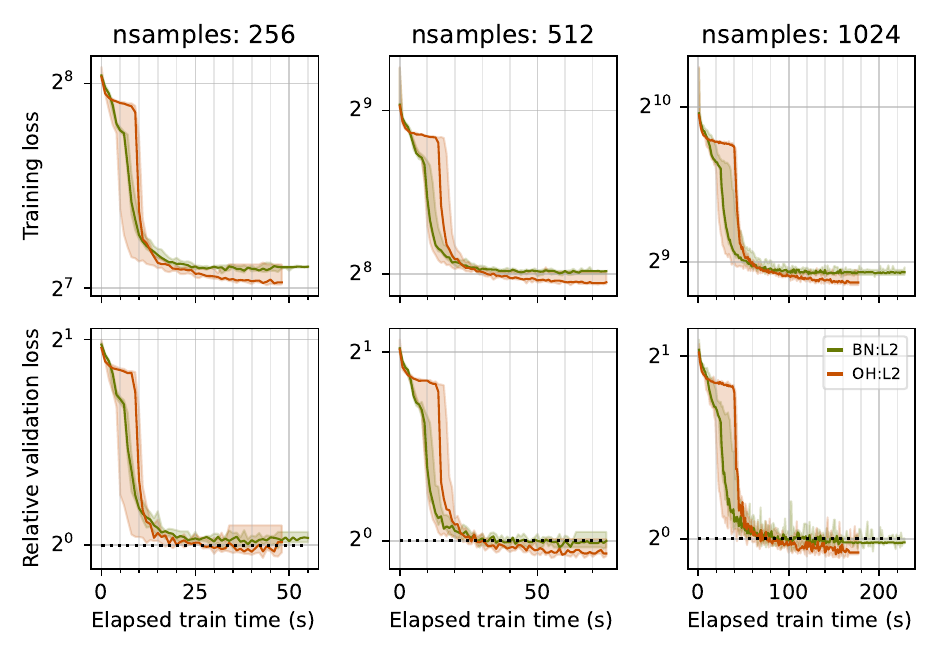}
\caption{$k = 25, d = 32, \tau = 1$.}
\label{fig:train-time:k25d32}
\end{subfigure}
\caption{Training loss convergence (top row) and relative validation loss from \cref{eq:rel-val-loss} (bottom row) in terms of actual training time for varying values of $n \in \{ 256, 512, 1024 \}$. In one case, we use $k = 16, d = 16$ with $\tau = 1$ (\cref{fig:train-time:k16d16}), while in another, we use $k = 25, d = 32$, again with $\tau = 1$ (\cref{fig:train-time:k25d32}). Here we limit ourselves to the L2-regularized $\min$, and compare the \oh and \bn token embeddings, giving us only the \textcolor{xkcdOliveGreen}{BN:L2} and the \textcolor{xkcdDarkOrange}{OH:L2} curves.}
\label{fig:train-time}
\end{figure}

\paragraph{Interplay of the lower computational overhead and slower convergence of the leaner \eqref{eq:tk-emb-v2}.}
The results in \cref{fig:enc-profile} and \cref{tab:enc-profile} indicate that the proposed \bn token embeddings can provide significant computational gains over the existing \oh token embeddings primarily on account of having smaller embeddings sizes $\demb$ (\oh with $(d+k)$ vs \bn with $(d+\ceil{\log_2 k})$). However, the results in \cref{thm:loss-lip-short} indicate that the model with the \bn token embeddigns will have slower convergence guarantees (in terms of the number of optimization steps). To understand their interplay, we consider two settings of $(d,k)$, and compare the convergence (of training and validation loss) in terms of actual elapsed time in \cref{fig:train-time}.
These training runs use gradient clipping to ensure that all the model trainings avoided any gradient explosions and numerical issues. With moderately small $d = k = 16$ in \cref{fig:train-time:k16d16}, we see that the model with the \bn token embeddings converges signficantly slower than the model with the \oh token embedding, though both appear to eventually converge close to each other, and better than single step Lloyd's update for all values of $n$. However, as the data dimensionality $d$ and the number of clusters $k$ grows to $d = 32, k = 25$ in \cref{fig:train-time:k25d32}, we see that, initially, the model with \bn token embeddings converges faster in terms of the runtime, with both methods eventually matching or outperforming single-step Lloyd's updates by the end of their 10000 training steps (which take different amounts of time for both).

\clearpage
\section{Additional Figures} \label{asec:add-figs}

\begin{figure}[htb]
\vskip -0.2in
\centering
\includegraphics[width=0.8\textwidth]{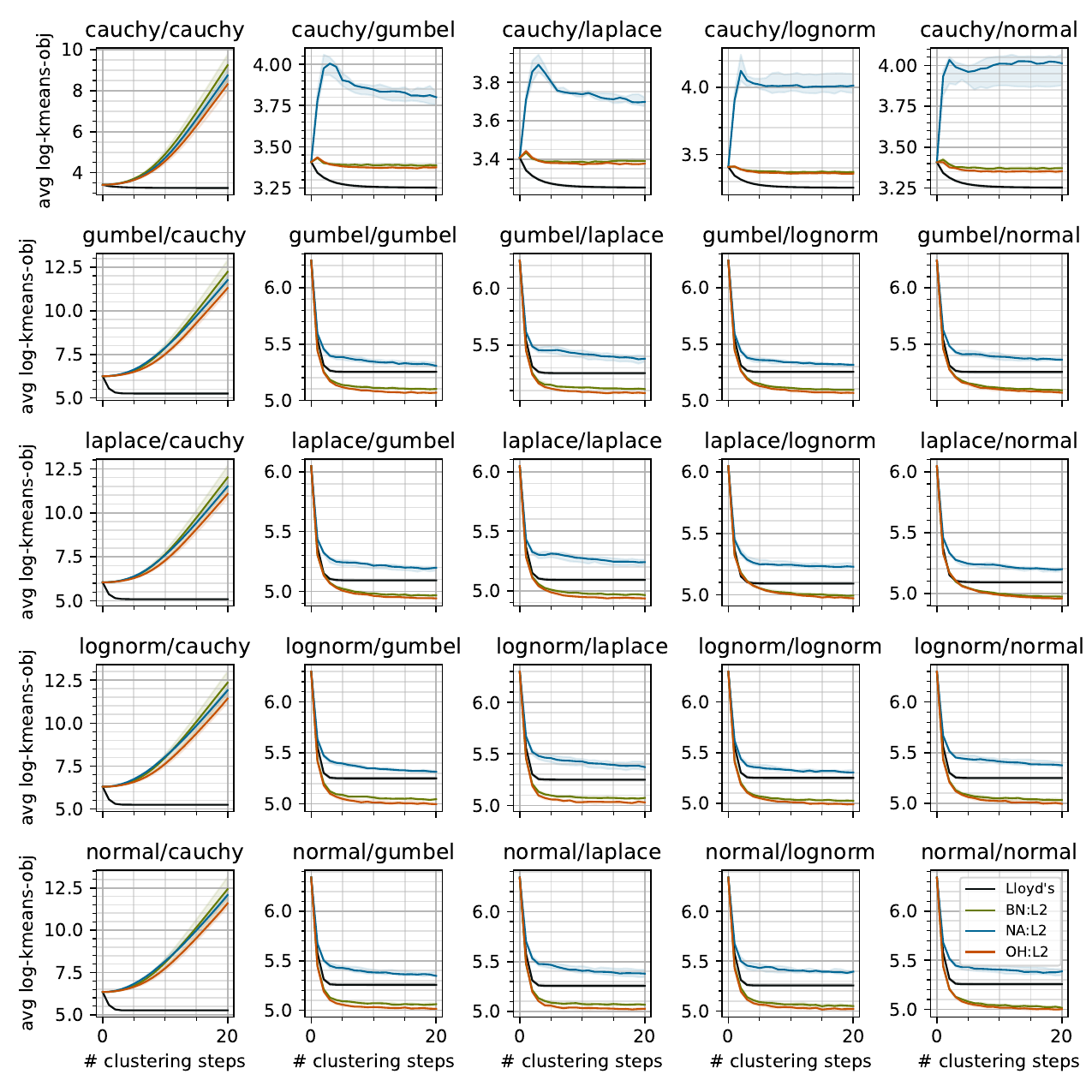}
\caption{{\bf Generalization to different distribution families.}
Larger version of \Cref{fig:varfam-gen}.
The models are trained on one distribution family (each column maps to a training family), and tested for multi-step ($T=20$) clustering from a different distribution family (each row maps to a test family).
We consider $5$ distributions --- \texttt{cauchy}, \texttt{gumbel}, \texttt{laplace}, \texttt{lognormal} and \texttt{normal} --- giving us a $5\times 5$ pairwise evaluations. Note that the pairwise evaluations on the ``diagonal'' correspond to training and evaluations tasks being sampled from the same distribution (though the tasks themselves are different, and the training only considers single-step clustering while we are evaluating multi-step clustering).
We only consider models trained with the L2-regularized $\min$ and the three forms of token embeddings, \oh, \bn and \na as described in \Cref{fig:tau-omega-demb-dim-analysis}.
Each curve corresponds to the inter-quartile range aggregated over 10 trials.
The performance of Lloyd's is in black.
}
\label{fig:varfam-gen-large}
\end{figure}

\clearpage

\begin{figure}[htb]
\centering
\includegraphics[clip, trim=15.5in 1.5in 17in 1.2in, width=\textwidth]{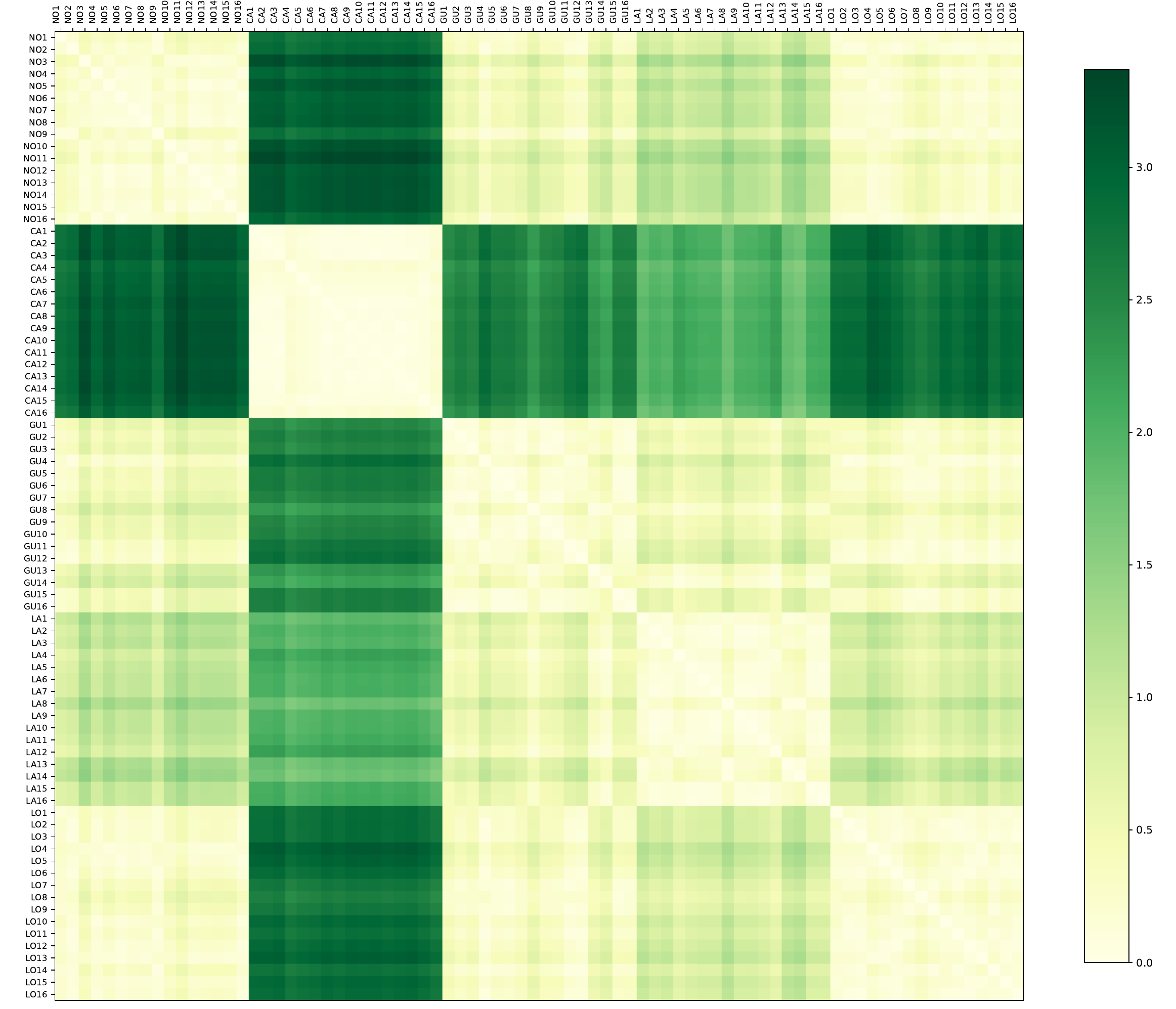}
\caption{
{\bf Pairwise task dissimilarities.}
A version of \Cref{fig:task-sim} with 16 tasks per distribution family. We use the 16 tasks from each of the families and compute the pairwise Earth-mover distance between their respective distribution of the squared Euclidean distances $\{ \eudi{\mX_{:,i}}{\mC_{:,j}}, i \in \iset{n}, j \in \iset{k} \}$ between the points $\mX$ and the initial centers $\mC$, giving us a $80 \times 80$ matrix. Each row and column corresponds to a single task, and the tasks are arranged into family-based blocks in the following order: \texttt{normal}, \texttt{cauchy}, \texttt{gumbel}, \texttt{laplace}, \texttt{lognormal}. {\em Light shades show low distances.} The ``diagonal'' corresponds to tasks distances to themselves, and thus are zero (the lightest shade). \texttt{cauchy} tasks are only similar to other \texttt{cauchy} tasks, and significantly dissimilar to tasks from other distribution families, with the \texttt{laplace} distribution appearing the least dissimilar (but still quite dissimilar).
}
\label{fig:task-sim-larger}
\end{figure}

\clearpage
\section{Smoothed Clustering Objective via Regularized $\min$} \label{asec:reg-argmax}

For the purposes of clustering, our main target is the non-differentiable per-point $k$-means loss $L: \R^d \times \R^{d \times k} \to \R$, given a sample $\vx \in \R^d$ and cluster centers $\mC = [\vc_1, \ldots, \vc_k ] \in \R^{d \times k}$:
\begin{equation} \label{eq:exact-kmeans-obj}
L(\vx, \mC) = \min_{j \in \iset{k}} \eudi{ \vx }{ \vc_j },
\end{equation}
with the complete $k$-means objective in \cref{eq:kmeans} given by $\sum_{i \in \iset{n}} L(\vx_i, \mC)$. We can obtain a differentiable, trainable surrogate for $L(\vx, \mC)$ in a few steps.
First, by defining function $R: \R^d \times \R^{d \times k} \to \R^k$  where $\vr = R(\vx, \mC)$ has $r_j \triangleq \eudi{ \vx }{ \mC_{:,j} }$ for each $j \in \iset{k}$, we have equivalently
\begin{equation} \label{eq:exact-kmeans-obj-r}
L(\vx, \mC) = \min_{j \in \iset{k}} r_j
\quad \text{where} \quad
\vr = R(\vx,\mC).
\end{equation}
Let $\Delta^k$ be the $(k-1)$-dimensional probability simplex with $\Delta^k = \{ \vp \in \R^k_{\geq 0}: \norm{\vp}_1 = 1 \}$.
Since $\min_{\vlambda \in \Delta^k} \inpr{\vlambda}{\vr} = \min_{j\in\iset{k}} r_j$, we could \emph{also} equivalently say, again with $\vr = R(\vx,\mC)$,
\begin{equation} \label{eq:exact-kmeans-obj-ip}
L(\vx, \mC) \triangleq \texttt{L}(\vr)
= \inpr{P(\vr)}{\vr},
\quad \text{where }
P(\vr) \triangleq \argmin_{\vlambda \in \Delta^k} \inpr{\vlambda}{\vr}
\end{equation}
We haven't changed $L(\vx, \mC)$, which remains non-differentiable.
However, we can regularize the minimization of $\vlambda$, obtaining a surrogate $\tilde{L}^\tau_\Omega(\vx, \mC)$, defined as
\begin{equation} \label{eq:reg-argmin-kmeans-obj}
\tilde{L}^\tau_\Omega(\vx, \mC) \triangleq \tilde{\texttt{L}}^\tau_\Omega (\vr)
= \inpr{P^\tau_\Omega(\vr)}{\vr}, \quad \text{where }
P^\tau_\Omega(\vr) \triangleq \argmin_{\vlambda \in \Delta^k} \inpr{\vlambda}{\vr} + \tau\Omega(\vlambda),
\end{equation}
where $\tau > 0$ is the penalty on the regularization $\Omega: \Delta^k \to \R$, with $\tau=0$ yields the unregularized $\argmin$.
If $\Omega$ is strongly convex, $\tilde{\texttt{L}}^\tau_\Omega(\vr)$ is differentiable with respect to $\vr$. Note that, for any $\tau > 0$, $P^\tau_\Omega(\vr) = P^1_\Omega(\vr/\tau)$, and thus we can just operate with $P^1_\Omega(\vr)$.
We will write $P^1_\Omega$ as $P_\Omega$ for the ease of exposition with:
\begin{equation} \label{eq:reg-argmin-generic}
P_\Omega(\vr) \triangleq \argmin_{\vlambda \in \Delta^k} \inpr{\vlambda}{\vr} + \Omega(\vlambda)
\end{equation}
For $\vlambda \in \Delta^k$, if $\Omega(\vlambda) = \sum_{j \in \iset{k}} \lambda_j \log \lambda_j$, the negative Shannon entropy, then the corresponding $P_{\text{\tiny NE}} (\vr / \tau)$ are the softmax weights (actually ``softmin'' weights) with the temperature $\tau$, giving us the following ``soft-min'' $k$-means objective $\tilde{L}_{\text{\tiny NE}}^\tau$:
\begin{equation} \label{eq:soft-min-kmeans-obj} \begin{split}
&
P_{\text{\tiny NE}} (\vr/\tau)_j
= \frac{\exp(- r_j /\tau)}{\sum_{j' \in \iset{k}} \exp(-r_{j'} / \tau)},
\\
&
\quad
\tilde{L}_{\text{\tiny NE}}^\tau(\vx, \mC)
\triangleq \tilde{\texttt{L}}_{\text{\tiny NE}}^\tau ( \vr )
= \inpr{P_{\text{\tiny NE}}(\vr/\tau)}{\vr}
= \sum_{j \in \iset{k}}
\frac{\exp(- r_j /\tau) r_j }{\sum_{{j'} \in \iset{k}} \exp(- r_{j'} /\tau)}.
\end{split}\end{equation}
Note that $L(\vx, \mC) \leq \tilde{L}_{\text{\tiny NE}}^\tau(\vx, \mC)$ with the equality achieved when $\tau \to 0$. In the learning procedure, we use \cref{eq:soft-min-kmeans-obj} instead of \cref{eq:exact-kmeans-obj} (or \cref{eq:exact-kmeans-obj-ip}) to have a differentiable objective.

An alternative choice for the regularization for $\vlambda \in \Delta^k$ is $\Omega(\vlambda) = \nicefrac{1}{2} \norm{ \vlambda }_2^2$, yielding
\[
P_{\text{\tiny L2}}(\vr) = \argmin_{\vlambda \in \Delta^k} \inpr{\vlambda}{\vr} + \nicefrac{1}{2} \norm{ \vlambda }_2^2
    = \argmin_{\vlambda \in \Delta^k} \nicefrac{1}{2}\norm{\vr}^2 + \inpr{\vlambda}{\vr} + \nicefrac{1}{2} \norm{ \vlambda }_2^2
    = \argmin_{\vlambda \in \Delta^k} \nicefrac{1}{2}\norm{\vr + \vlambda}^2,
\]
the closest point to $(-\vr)$ in $\Delta^k$.
For $\tau > 0$ this yields the sparsemax weights $P_{\text{\tiny L2}}(\vr / \tau)$, which can be computed as:
\begin{equation}
P_{\text{\tiny L2}}(\vr/\tau)_j = \left[ - (r_j / \tau) - \kappa(-\vr/\tau) \right]_+,
\quad \text{ with }
[a]_+ = \max\{ a, 0 \}, a \in \R,
\end{equation}
where $\kappa: \R^k \to \R$ is the appropriately defined threshold function such that, for any $\vz \in \R^k$, $\sum_{ j \in \iset{ k } } [z_j - \kappa(\vz) ]_+ = 1$; see \citet{wang2013projection} and \citet[Proposition 1]{martins2016softmax}.
For any $\vz \in \R^k$, let its entries be ordered as $z_{(1)} \geq z_{(2)} \geq \cdots \geq z_{(k)}$.
Then $\kappa(\vz)$ is computed as follows:
\begin{equation}
N(\vz) = \max\left\{ m \in \iset{ k }: 1 + m z_{(m)} > \sum_{j \leq m} z_{(j)} \right\},
\quad
\kappa(\vz) = \frac{ \left(\sum_{j \leq N(\vz)} z_{(j)}  \right) - 1 }{ N(\vz) }.
\end{equation}
This results in sparser weights, where some of the $P_{\text{\tiny L2}}(\vr/\tau)_j$ values can be exactly zero. This gives us the following differentiable upper-bound of the $k$-means objective:
\begin{equation} \label{eq:sparse-min-kmeans-obj}
\tilde{L}_{\text{\tiny L2}}^\tau(\vx, \mC)
\triangleq \tilde{\texttt{L}}_{\text{\tiny L2}}^\tau ( \vr )
= \inpr{P_{\text{\tiny L2}}(\vr/\tau)}{\vr}
= \sum_{j \in \iset{k}} \left[ - (r_j /\tau) - \kappa(-\vr/\tau) \right]_+ r_j.
\end{equation}
Again $L(\vx, \mC) \leq \tilde{L}_{\text{\tiny L2}}^\tau(\vx, \mC)$ with equality now achieveable at $\tau > 0$ (as opposed to at $\tau \to 0$ for NE-regularization) and depends on this notion of ``margin''~\citep{blondel2020learning}. We will also consider this as our learning loss. Using \citet[Proposition 1, Approximation Error]{blondel2020learning}, we can show that, for the $k$-means objective, for the same regularization penalty $\tau > 0$, the sparsemax-based upper-bound in \cref{eq:sparse-min-kmeans-obj} provides a tighter upper than the softmax-based upper-bound in \cref{eq:soft-min-kmeans-obj}.
\begin{proposition}[\citet{blondel2020learning}] \label{prop:reg-arg-max-app-err}
Suppose $\textsf{domain}(\Omega)$ contains all vectors in $\gY = \{ \ve_1, \ldots, \ve_k \}$, where $\ve_j$ is the $j^\text{th}$ characteristic vector.
If such an $\Omega$ is $\gamma$-strongly convex and bounded, with $L \leq \Omega(\vlambda) \leq U$ for all $\vlambda \in \textsf{domain}(\Omega)$, then $\nicefrac{1}{2} \eudi{P(\vz)}{P_\Omega(\vz)} \leq \frac{U - L }{\gamma}$ for any $\vz \in \R^k$.
\end{proposition}
For our purposes $\textsf{domain}(\Omega) = \Delta^k$ and thus, $\gY \subset \Delta^k$. Quantifying the tightness of the upper-bound, we have
\begin{equation} \label{eq:obj-bound-tight}
\begin{split}
0 < \tilde{L}_\Omega^\tau(\vx, \mC) - L(\vx, \mC)
& = \tilde{\texttt{L}}_\Omega^\tau(\vr) - \texttt{L}(\vr)
\\
&
= \inpr{P_\Omega(\vr / \tau) - P(\vr) }{\vr}
\leq \norm{P(\vr) - P_\Omega(\vr / \tau)}\norm{\vr}.
\end{split}
\end{equation}
Thus, a smaller value of $\norm{P(\vr) - P_\Omega(\vr / \tau)}$ implies a tighter upper-bound. Note that for $\tau > 0$, $P(\vr) = P(\vr / \tau)$, and thus $\norm{P(\vr) - P_\Omega(\vr / \tau)} = \norm{P(\vr / \tau) - P_\Omega(\vr / \tau)}$, so it suffices to bound $\norm{P(\vz) - P_\Omega(\vz)}$ for any $\vz \in \R^k$

For sparsemax with $\Omega(\vlambda) = \nicefrac{1}{2}\norm{\vlambda}_2^2$, $\Omega$ is $\gamma$-strongly convex with $\gamma = 1$, and $\Omega(\vlambda) = \nicefrac{1}{2}\norm{\vlambda}_2^2 \in [\nicefrac{1}{2k}, \nicefrac{1}{2}]$ for any $\vlambda \in \Delta^k$. Thus $\nicefrac{1}{2} \eudi{P(\vz)}{P_\Omega(\vz)} \leq \nicefrac{1}{2}(1 - \nicefrac{1}{k}) < \nicefrac{1}{2}$ for any $\vz \in \R^k$ using \cref{prop:reg-arg-max-app-err}. Thus $\norm{P(\vr / \tau) - P_\Omega(\vr / \tau) } \leq \sqrt{1 - \nicefrac{1}{k}} < 1$.

For softmax with $\Omega(\vlambda) = \sum_{j \in \iset{ k } } \lambda_j \log \lambda_j$, with $0 \leq \lambda_j \leq M \leq 1$, $\Omega$ is $\gamma$-strongly convex with $\gamma = 1/M$.
Naively $M = 1$ and thus, $\Omega$ would be 1-strongly convex.
When $\vlambda$ is one-hot, $\Omega(\vlambda) = 0 = U$, while $\lambda_j = 1/k \forall j \in \iset{k}$ gives us $\Omega(\vlambda) = - \log k = L$. Thus, we have
\begin{equation}
\nicefrac{1}{2} \eudi{P(\vz)}{P_\Omega(\vz)} \leq \frac{U - L}{\gamma} =  \log k.
\end{equation}
Thus, for softmax, $\norm{P(\vr/\tau) - P_\Omega(\vr / \tau) } \leq \sqrt{2 \log k}$, giving a looser bound than with sparsemax of $\sqrt{1 - \nicefrac{1}{k}}$.

Next, we will establish the Lipschitz continuity of the smoothed objective $\tilde{\texttt{L}}_\Omega^\tau$. First, let us assume that the regularized $\min$ operation is $\omega$-input-stable. That is:
\begin{assumption} \label{asp:reg-argmax-stable}
The regularized $\min$ operation $P_\Omega: \R^k \to \Delta^k$ in \cref{eq:reg-argmin-generic} is $\omega_\tau$-input-stable for a $\tau$-dependent $\omega_\tau > 0$, such that, for any $\vr, \vr' \in \R^k$
\begin{equation}\label{eq:reg-argmax-stable}
\norm{P_\Omega(\vr/\tau) - P_\Omega(\vr'/\tau)}_1
\leq (\omega_\tau/\tau) \norm{\vr - \vr'}_1.
\end{equation}
\end{assumption}
This assumption is already satisfied for softmax with bounded inputs $\vr, \vr'$~\citep[Lemma B.1]{li2023transformers}:
\begin{align}
& \norm{P_{\text{\tiny NE}}(\vr/\tau) - P_{\text{\tiny NE}}(\vr'/\tau)}_1
\leq \frac{2 \exp(\delta/\tau)}{k \tau} \norm{\vr - \vr'}_1,
\\
&
\quad \text{where }
\delta \geq (\max_{j \in \iset{k}} r_j) - (\min_{j \in \iset{k}} r_j) \, \forall \vr
\qquad
\text{ and }
\delta \geq (\max_{j \in \iset{k}} r_j') - (\min_{j \in \iset{k}} r_j') \, \forall \vr',
\end{align}
with the corresponding $\omega_\tau = \exp(\nicefrac{\delta}{\tau}) / k$.
Thus we can show the following result:
\begin{proposition}\label{prop:smooth-obj-lip}
The smoothed $k$-means objective $\tilde{\texttt{L}}_\Omega^\tau: \R^k \to \R$ in \cref{eq:reg-argmin-kmeans-obj} is Lipschitz continuous, such that, with any $\vr, \vr' \in [0, r^\star]^k$, we have
\begin{equation} \label{eq:smooth-obj-lip}
\left|
\tilde{\texttt{L}}_\Omega^\tau (\vr) -  \tilde{\texttt{L}}_\Omega^\tau (\vr')
\right|
= \left|
\inpr{P_\Omega(\vr/\tau)}{\vr} - \inpr{P_\Omega(\vr'/\tau)}{\vr'}
\right|
\leq \left(1 + \frac{r^\star \omega_\tau}{ \tau } \right) \norm{ \vr - \vr' }_1.
\end{equation}
\end{proposition}
\begin{proof}
First, we have the following inequalities by simply applying the triangle inequalities:
\begin{align*}
& \left|
\inpr{P_\Omega(\vr/\tau)}{\vr} - \inpr{P_\Omega(\vr'/\tau)}{\vr'}
\right|
\\
&\quad  \leq
\left| \inpr{P_\Omega(\vr/\tau)}{(\vr - \vr')} \right|
+ \left| \inpr{(P_\Omega(\vr/\tau)- P_\Omega(\vr'/\tau))}{\vr'} \right|
\\
&\quad \leq
\max_{j \in \iset{k}} | P_\Omega(\vr/\tau)_j | \norm{\vr - \vr'}_1
+ \max_{j \in \iset{k}} |r'_j| \norm{P_\Omega(\vr/\tau)- P_\Omega(\vr'/\tau))}_1.
\end{align*}
Noting that $\max_{j \in \iset{k}} P_\Omega(\vr/\tau)_j \leq 1$ as $P_\Omega: \R^k \to \Delta^k$, and $r_j, r'_j \in [0, r^\star]$, and applying \cref{eq:reg-argmax-stable} gives us the result in \cref{eq:smooth-obj-lip}.
\end{proof}
This result indicates that the Lipschitz constant of the smoothed loss usually increases as the $\min$-regularization penalty $\tau$ decreases. In general, the Lipschitz constant depends on $(1/\tau)$. Furthermore, if $\omega_\tau$ increases as $\tau$ decreases, then the Lipschitz constant also increases. With softmax, $\omega_\tau = \exp(\nicefrac{\delta}{\tau})/k $ also increases as $\tau$ decreases.

\section{Lipschitz-ness of the Smoothed Upperbound of the Clustering Objective} \label{asec:smoothed-obj-lip}

In this section, we will establish the Lipschitz constant for the smoothed upper-bound $\tilde{\Ls}_\Omega^\tau$ of the original discrete non-differentiable clustering objective $\Ls$. We begin with some standard assumptions:

\begin{assumption}[bounded inputs] \label{asp:bnd-input}
For any clustering task $(\mX, \mC) \sim \gD$ with $\mX \in \R^{d \times n}$ and $\mC \in \R^{d \times k}$, we assume that any point and initial cluster center is bounded to the unit hypercube, that is, $\mX_{:,i} \in [0,1]^d$ for all $i \in \iset{ n }$ and $\mC_{:,j} \in [0,1]^d$ for all $j \in \iset{ k }$.
\end{assumption}
We explicitly enforce this boundedness in all our experiments.
\begin{assumption}[bounded model parameters] \label{asp:bnd-params}
We assume that the query, key and value projection matrices $(\mQ_\mu, \mK_\mu, \mV_\mu), \mu = 1, \ldots, 4$ have a bounded spectral norm throughout the learning process, that is,
\begin{equation} \label{eq:bnd-params}
\norm{\mQ_\mu} \leq \Gamma_1, \quad
\norm{\mK_\mu} \leq \Gamma_1, \quad
\norm{\mV_\mu} \leq \Gamma_2, \quad
\mu = 1, \ldots, 4.
\end{equation}
\end{assumption}
For the ease of exposition, we assume that the query and key projection matrices share the norm bound.
\begin{lemma}[bounded model outputs] \label{lem:out-bound}
Under Assumptions~\ref{asp:bnd-input}~and~\ref{asp:bnd-params}, and the attention based updates in \Cref{def:updates}, the initial token embeddings, the intermediate token embeddings, and the cluster centers $\mC(\Theta) \triangleq \cmodel(\mX, \mC; \Theta)$ output by the model $\cmodel$ with initial token embedding scheme $(E_C, E_X)$ for any clustering task $(\mX, \mC)$ have bounded norms as follows:
\begin{equation} \label{eq:out-bound}
\begin{split}
\forall i \in \iset{ n },
&
\norm{ \bar{\mX}^{(0)}_{:,i} } \leq \sqrt{d},
\qquad \qquad
\norm{ \bar{\mX}^{(1)}_{:,i} } \leq (1 + \Gamma_2) \sqrt{d} + \Gamma_2 \sqrt{d_E},
\\
\forall j \in \iset{ k },
&
\norm{ \bar{\mC}^{(0)}_{:,j} } \leq \sqrt{d_E},
\qquad
\norm{ \mC(\Theta)_{:,j} } \leq ( 1 + 2 \Gamma_2 + 2 \Gamma_2^2 ) \sqrt{ d_E },
\\
& \quad
\text{where }
d_E = \begin{cases}
d + 1 & \text{if using \eqref{eq:tk-emb-v1}} \\%  E = E^{\text{\tiny OH}} \\
d + \ceil{\log_2 k}  & \text{if using \eqref{eq:tk-emb-v2}}
\end{cases}.
\end{split}
\end{equation}
\end{lemma}
\begin{proof}
Denoting $\bar{\mX}^{(0)} \gets E_X(\mX)$ and $\bar{\mC}^{(0)} \gets E_C(\mC)$, we can see that $\norm{ \bar{\mX}^{(0)}_{:,i} } \leq \sqrt{d}$ for any $i \in \iset{ n }$ given Assumption~\ref{asp:bnd-input} and the fact that $\mY^{(0)} = \mathbf{0}_{(\demb - d) \times n}$ for any initial token embedding $E_X$ for the points. For the embedded initial cluster centers, for any $j \in \iset{ k }$, their norms $\norm{ \bar{\mC}^{(0)}_{:,j} } \leq \sqrt{ d + 1 }$ with the \oh token embedding scheme $\eohc$ and $\norm{ \bar{\mC}^{(0)}_{:,j} } \leq \sqrt{ d + \ceil{\log_2 k } }$ with the binary token embedding scheme $\ebnc$. Thus we can see that  $\norm{ \bar{\mC}^{(0)}_{:,j} } \leq \sqrt{d_E}$ with $d_E$ defined in \cref{eq:out-bound}.

Since the output of $\mC(\Theta) \triangleq \cmodel(\mX, \mC; \Theta)$ is the first $d$ rows $\bar{\mC}^{(1)}_{:d,:}$ of the token embeddings for the cluster centers $\bar{\mC}^{(1)}$ after a single transformer layer with updates defined as in \Cref{def:updates}, then we have
\begin{align}
\norm{ \mC(\Theta)_{:,j} }
= \norm{ \bar{\mC}^{(1)}_{:d,j} } \leq  \norm{ \bar{\mC}^{(1)}_{:,j} }
& =
\norm{
  \bar{\mC}^{(0)}_{:,j} + \left( \mV_3 \sum_{i = 1}^n a_{ji} \bar{\mX}^{(1)}_{:,i} \right)
+ \left( \mV_4 \sum_{j' = 1}^k b_{jj'} \bar{\mC}^{(0)}_{:,j'} \right)
}
\\
& \leq
\norm{\bar{\mC}^{(0)}_{:,j}} + \max_{i \in \iset{n}} \norm{ \mV_3 \bar{\mX}^{(1)}_{:,i} }
+ \max_{j' \in \iset{ k }} \norm{ \mV_4 \bar{\mC}^{(0)}_{:,j'} },
\label{eq:c-up-bnd}
\end{align}
where $a_{ji}$ are the cross-attention weights and $b_{jj'}$ are the self-attention weights in the cluster center updates in \cref{eq:updates}. The first inequality holds above since the norm of the first $d$ dimensions is smaller than the norm of all the $\demb$ dimensions. The last inequality applies triangle inequality, and the fact that the attention weights lie on a simplex (that is, $a_{ji} \in [0,1], \sum_i a_{ji} = 1$ and $b_{jj'} \in [0,1], \sum_{j'} b_{jj'} = 1$) thus corresponding to a convex sum of the value vectors, and a convex sum is upper-bounded by the maximum. Now, we need to bound the norm of the updated token embeddings $\bar{\mX}^{(1)}$ of the points:
\begin{align}
\forall i \in \iset{ n }, \norm{ \bar{\mX}^{(1)}_{:,i} }
& = \norm{ \bar{\mX}^{(0)}_{:,i} +  \left( \mV_1 \sum_{j' = 1}^k a'_{ij'} \bar{\mC}^{(0)}_{:,j'} \right)
+ \left( \mV_2 \sum_{i' = 1}^n b'_{ii'} \bar{\mX}^{(0)}_{:,i'} \right)
}
\\
& \leq \norm{ \bar{\mX}^{(0)}_{:,i} } + \max_{j' \in \iset{ k }}  \norm{ \mV_1 \bar{\mC}^{(0)}_{:,j'} }
+ \max_{i' \in \iset{ n } } \norm{ \mV_2 \bar{\mX}^{(0)}_{:,i'} }
\\
& \leq \sqrt{ d } + \Gamma_2 \sqrt{ d_E } + \Gamma_2 \sqrt{ d },
\end{align}
where $a'_{ij'}$ and $b'_{ii'}$ are the cross-attention and self-attention weights for the point embedding updates in \cref{eq:updates}. This gives us the bound on $\norm{ \bar{\mX}^{(1)}_{:,i} }$ in \cref{eq:out-bound} in the statement of the result. Plugging this in \cref{eq:c-up-bnd}, we see
\begin{align}
\norm{ \mC(\Theta)_{:,j} }
& \leq
\norm{\bar{\mC}^{(0)}_{:,j}} + \max_{i \in \iset{n}} \norm{ \mV_3 \bar{\mX}^{(1)}_{:,i} }
+ \max_{j' \in \iset{ k }} \norm{ \mV_4 \bar{\mC}^{(0)}_{:,j'} },
\\
& \leq
\sqrt{ d_E } + \Gamma_2 ( \sqrt{ d } + \Gamma_2 \sqrt{ d_E } + \Gamma_2 \sqrt{ d } )
+ \Gamma_2 \sqrt{ d_E }
\leq ( 1 + 2 \Gamma_2 + 2 \Gamma_2^2 ) \sqrt{ d_E },
\end{align}
thus giving us \cref{eq:out-bound} in the statement of lemma.
\end{proof}

\begin{lemma} \label{lem:loss-to-output}
Under Assumptions~\ref{asp:bnd-input},~\ref{asp:bnd-params},~\ref{asp:reg-argmax-stable}, and the conditions of \Cref{lem:out-bound}, for any clustering task $(\mX, \mC) \sim \gD$ with $n$ points in $d$ dimensions and $k$ clusters, the smoothed loss $\tilde{\Ls}_\Omega^\tau$ is input-stable with respect to the output of the model $\cmodel$ upon changes to the model parameters $\Theta$ as follows for any pair of parameters $\Theta, \Theta'$:
\begin{equation} \label{eq:loss-to-output}
\begin{split}
&
\left|
\tilde{\Ls}_\Omega^\tau( \mX, \cmodel( \mX, \mC; \Theta ) )
- \tilde{\Ls}_\Omega^\tau( \mX, \cmodel( \mX, \mC; \Theta' ) )
\right|
\\
& \quad
\leq
\kappa_1(n, \Gamma_2, d_E, \omega_\tau, \tau)
\norm{
\cmodel( \mX, \mC; \Theta ) - \cmodel( \mX, \mC; \Theta' )
}_{2,1}
\\ &
\quad
\text{where }
\kappa_1(n, \Gamma_2, d_E, \omega_\tau, \tau)
= 4 n   (1 + \Gamma_2)^2 \sqrt{d_E} \left( 1 + \frac{4 (1 + \Gamma_2)^4 d_E \omega_\tau }{\tau} \right)
\sim O \left(
\frac{n \omega_\tau
  d_E^{\nicefrac{3}{2}} }{\tau}
\right).
\end{split}
\end{equation}
\end{lemma}
\begin{proof}
Denoting the output cluster centers as $\mC(\Theta) \triangleq \cmodel( \mX, \mC; \Theta )$ and $\mC(\Theta') \triangleq \cmodel( \mX, \mC; \Theta' )$, first, we can see the expand the clustering loss $\tilde{\Ls}_\Omega^\tau$ to a per-point loss $\tilde{L}_\Omega^\tau$ as follows
\begin{align}
\left|
\tilde{\Ls}_\Omega^\tau( \mX, \mC( \Theta ) )
- \tilde{\Ls}_\Omega^\tau( \mX, \mC( \Theta' ) )
\right|
& =
\left|
\sum_{i = 1}^n \tilde{L}_\Omega^\tau( \mX_{:, i}, \mC( \Theta ) )
- \sum_{i = 1}^n \tilde{L}_\Omega^\tau( \mX_{:,i}, \mC( \Theta' ) )
\right|
\\
& \leq \sum_{i = 1}^n \left|
\tilde{L}_\Omega^\tau( \mX_{:, i}, \mC( \Theta ) )
- \tilde{L}_\Omega^\tau( \mX_{:,i}, \mC( \Theta' ) )
\right|
\\
& \leq
n \max_{i \in \iset{ n } } \left|
\tilde{L}_\Omega^\tau( \mX_{:, i}, \mC( \Theta ) )
- \tilde{L}_\Omega^\tau( \mX_{:,i}, \mC( \Theta' ) )
\right|.
\label{eq:full-per-sample-bnd}
\end{align}
We will bound the above for a generic $\vx \in [0,1]^d$. Using the notation that $\vr = R(\vx, \mC(\Theta))$ and $\vr' = R(\vx, \mC(\Theta'))$, where $R: \R^d \times \R^{d \times k} \to \R^k$ is as previously defined where the $j$-th entry of its output $r_j = \eudi{\vx}{ \mC(\Theta)_{:,j} }$ (and correspondingly $r'_j = \eudi{\vx}{ \mC(\Theta')_{:,j}}$), we have $\vr, \vr' \in [0, r^\star]^k$ with
\begin{align}
r^\star \triangleq \max_{\vx \in [0,1]^d, j \in \iset{ k }, \Theta} \eudi{\vx}{ \mC(\Theta)_{:,j} }
& \leq
\left( \max_{\vx \in [0,1]^d} \norm{ \vx } + \max_{j \in \iset{ k }, \Theta} \norm{ \mC(\Theta)_{:,j} } \right)^2
\\
& \leq
\left( \sqrt{d} +
(1 + 2\Gamma_2 + 2 \Gamma_2^2) \sqrt{d_E}
\right)^2
\\ &
\leq
4 d_E (1 + \Gamma_2 + \Gamma_2^2)^2 \leq 4(1 + \Gamma_2)^4 d_E.
\end{align}
where we bounded $\max_{j, \Theta} \norm{\mC(\Theta)}$ using \Cref{lem:out-bound}. Now we can utilize \Cref{prop:smooth-obj-lip} as follows:
\begin{align}
\left|
\tilde{L}_\Omega^\tau( \vx, \mC( \Theta ) )
- \tilde{L}_\Omega^\tau( \vx, \mC( \Theta' ) )
\right|
& =
\left|
\tilde{\texttt{L}}_\Omega^\tau( \vr )
- \tilde{\texttt{L}}_\Omega^\tau( \vr' )
\right|
\\ &
\leq
\left( 1 + \frac{r^\star \omega_\tau}{\tau} \right) \norm{ \vr - \vr' }_1
\\ &
= \left( 1 + \frac{4(1 + \Gamma_2)^4 d_E \omega_\tau}{\tau} \right)
\norm{ \vr - \vr' }_1.
\label{eq:per-sample-loss-dist-diff}
\end{align}
Finally, we can bound $\norm{\vr - \vr'}_1$ as follows:
\begin{align}
\norm{ \vr - \vr' }_1 & = \sum_{j = 1}^k | r_j - r'_j | \\
& = \sum_{j = 1}^k \left|
\inpr{
\left( \mC(\Theta)_{:,j} - \mC(\Theta')_{:,j} \right)
}{
\left( \mC(\Theta)_{:,j} + \mC(\Theta')_{:,j} - 2 \vx \right)
}
\right|
\\
& \leq
\sum_{j = 1}^k \norm{\mC(\Theta)_{:,j} - \mC(\Theta')_{:,j}}
\norm{\mC(\Theta)_{:,j} + \mC(\Theta')_{:,j} - 2 \vx}
\\
& \leq
\sum_{j = 1}^k \norm{\mC(\Theta)_{:,j} - \mC(\Theta')_{:,j}}
\left( 2 \sqrt{ d } + 2 ( 1 + 2\Gamma_2 + 2\Gamma_2^2 ) \sqrt{ d_E } \right)
\\
& \leq
\left( 4 ( 1 + \Gamma_2 )^2 \sqrt{ d_E } \right)
\sum_{j = 1}^k \norm{\mC(\Theta)_{:,j} - \mC(\Theta')_{:,j}}
\\
& =
4 ( 1 + \Gamma_2 )^2 \sqrt{ d_E }
\norm{\mC(\Theta) - \mC(\Theta')}_{2,1}.
\end{align}
Finally, combining the above with \cref{eq:per-sample-loss-dist-diff} and \cref{eq:full-per-sample-bnd} gives us the inequality in \cref{eq:loss-to-output} in the statement of the result.
\end{proof}
The above result ties the sensitivity of the loss to the model parameters $\left| \tilde{\Ls}_\Omega^\tau( \mX, \mC( \Theta ) ) - \tilde{\Ls}_\Omega^\tau( \mX, \mC( \Theta' ) ) \right|$ to the sensitivity of the model output to the model parameters $\norm{\mC(\Theta) - \mC(\Theta')}_{2,1}$. To bound this $\norm{\mC(\Theta) - \mC(\Theta')}_{2,1}$ with perturbations to the model parameters, we have to consider the attention based updates in \Cref{def:attn-op}. Since there are four different attention operations in \cref{eq:attn-op} in \Cref{def:attn-op}, we will first present bounds on the sensitivity of the output of a single attention module $\attned\gamma(\cdots)$ to changes in the specific model parameters (the query/key/value projection matrices) for generic query and key tokens $\mZ, \mZ'$, and then apply them to each of the four attention modules with the appropriate keys and queries. These results are presented in the following, and then summarized in \Cref{tab:terms-dependencies}.
\begin{definition} \label{def:dispersion}
For any $\vv \in \R^m$, let its dispersion be defined as $D(\vv)$ with a function $D: \R^m \to \R_{\geq 0}$ such that
\begin{equation}\label{eq:dispersion}
D(\vv) = \left( \max_{i \in \iset{ m } } v_i  \right) - \left( \min_{i \in \iset{ m } } v_i  \right).
\end{equation}
\end{definition}

\begin{lemma}[adapted from \citet{li2023transformers}, Lemma B.3] \label{lem:smax-in-stab}
For any $\vv, \vv' \in R^m$ such that $D(\vv), D(\vv') \leq \delta$ for some $\delta > 0$, the softmax operation is bounded and input-stable as follows:
\begin{equation} \label{eq:smax-in-stab}
\norm{ \softmax(\vv) }_\infty \leq \frac{\exp(\delta)}{m},
\quad
\norm{ \softmax(\vv) - \softmax(\vv') }_1 \leq \frac{2 \exp(\delta)}{m} \norm{ \vv - \vv' }_1.
\end{equation}
\end{lemma}
\begin{lemma}[attention score dispersion] \label{lem:smax-l2-dist-disp}
For queries $\mZ \in \R^{\demb \times m}$ and keys $\mZ' \in \R^{\demb \times m'}$ such that $ \norm{ \mZ_{:, i} } \leq \rho$ for all $i \in \iset{ m }$ and $\norm{ \mZ'_{:,i} } \leq \rho'$ for all $i \in \iset{ m' }$, under the assumption~\ref{asp:bnd-params}, for $\vv_i = -(\nicefrac{\gamma}{\sqrt{\demb}}) \sqdist( \mQ_\mu \mZ_{:,i}, \mK_\mu \mZ' )$, we have the dispersion bounded as
\begin{equation} \label{eq:smax-l2-dist-disp}
\forall i \in \iset{ m },
D (\vv_i) \leq \delta(\rho, \rho', \Gamma_1, \gamma, \demb) = (\nicefrac{\gamma}{\sqrt{\demb}}) \Gamma_1^2 (\rho + \rho')^2.
\end{equation}
\end{lemma}
\begin{proof}
We have the following for any $i \in \iset{ m }$:
\begin{align}
D(\vv_i)
& =
\max_{i' \in \iset{ m' } } \left\{ -(\nicefrac{\gamma}{\sqrt{\demb}}) \eudi{ \mQ_\mu \mZ_{:,i}}{ \mK_\mu \mZ'_{:,i'}} \right\}
- \min_{i' \in \iset{ m' } } \left\{ -(\nicefrac{\gamma}{\sqrt{\demb}}) \eudi{ \mQ_\mu \mZ_{:,i}}{ \mK_\mu \mZ'_{:,i'}} \right\}
\\
& = \nicefrac{\gamma}{\sqrt{\demb}} \left(
\max_{i' \in \iset{ m' } } \left\{ \eudi{ \mQ_\mu \mZ_{:,i}}{ \mK_\mu \mZ'_{:,i'}} \right\}
- \min_{i' \in \iset{ m' } } \left\{ \eudi{ \mQ_\mu \mZ_{:,i}}{ \mK_\mu \mZ'_{:,i'}} \right\}
\right)
\\
& \leq \nicefrac{\gamma}{\sqrt{\demb}}
\max_{i' \in \iset{ m' } } \left\{ \eudi{ \mQ_\mu \mZ_{:,i}}{ \mK_\mu \mZ'_{:,i'}} \right\}
\\
& \leq \nicefrac{\gamma}{\sqrt{\demb}} \left(
\norm{ \mQ_\mu \mZ_{:,i} } + \max_{i' \in \iset{m'} }\norm{ \mK_\mu \mZ'_{:,i'} }
\right)^2
\leq \nicefrac{\gamma}{\sqrt{\demb}} \left( \Gamma_1 (\rho + \rho') \right)^2,
\end{align}
giving us the definition of $\delta( \rho, \rho', \Gamma_1,\gamma, \demb)$ in the statement of the result.
\end{proof}
\begin{lemma} \label{lem:attn-op-vproj}
Under conditions of \Cref{lem:smax-l2-dist-disp}, the attention operation is stable with respect to the value projection matrix as follows:
\begin{equation}\label{eq:attn-op-vproj}
\norm{
\attned\gamma(\mZ, \mZ'; (\mQ_\mu, \mK_\mu, \mV_\mu))
- \attned\gamma(\mZ, \mZ'; (\mQ_\mu, \mK_\mu, \mV_\mu'))
}_{2,1}
\leq \kappa_2(m, \rho') \norm{ \mV_\mu - \mV'_\mu },
\end{equation}
with $\kappa_2(m, \rho') = m \rho'$.
\end{lemma}
\begin{proof}
Expanding out the attention operation, we can see that
\begin{align}
& \norm{
\attned\gamma(\mZ, \mZ'; (\mQ_\mu, \mK_\mu, \mV_\mu))
- \attned\gamma(\mZ, \mZ'; (\mQ_\mu, \mK_\mu, \mV_\mu'))
}_{2,1}
\nonumber
\\ & \quad
=
\sum_{i = 1}^m \norm{
\sum_{j = 1}^{m'} a_{ij} \mV_\mu \mZ'_{:,j} -
\sum_{j = 1}^{m'} a_{ij} \mV'_\mu \mZ'_{:,j}
}
=
\sum_{i = 1}^m \norm{
(\mV_\mu - \mV'_\mu) \sum_{j = 1}^{m'} a_{ij} \mZ'_{:,j}
}
\\ & \quad
\leq
\sum_{i = 1}^m \norm{(\mV_\mu - \mV'_\mu)} \norm{ \sum_{j = 1}^{m'} a_{ij} \mZ'_{:,j} }
\leq
\sum_{i = 1}^m \norm{ \mV_\mu - \mV'_\mu } \max_{j \in \iset{ m' } } \norm{ \mZ'_{:,j} }
\leq m \rho' \norm{ \mV_\mu - \mV'_\mu },
\end{align}
where $a_{ij}$ are the softmax attention weights for the $i$-th query computed by $\softmax(-\nicefrac{\gamma}{\sqrt{\demb}} \sqdist( \mQ_\mu \mZ_{:,i}, \mK_\mu \mZ' ) )$, thus giving us $\kappa_2(m, \rho')$ in the statement of the result.
\end{proof}
\begin{lemma} \label{lem:attn-op-qproj}
Under conditions of \Cref{lem:smax-l2-dist-disp}, the attention operation is stable with respect to the query projection matrix as follows:
\begin{equation}\label{eq:attn-op-qproj}
\begin{split}
&
\norm{
\attned\gamma(\mZ, \mZ'; (\mQ_\mu, \mK_\mu, \mV_\mu))
- \attned\gamma(\mZ, \mZ'; (\mQ_\mu', \mK_\mu, \mV_\mu))
}_{2,1}
\\ & \quad
\leq \kappa_3(m, \rho, \rho', \Gamma_1, \Gamma_2, \gamma, \demb)
\norm{ \mQ_\mu - \mQ'_\mu },
\\
& \quad\
\text{where} \quad
\kappa_3(m, \rho, \rho', \Gamma_1, \Gamma_2, \gamma, \demb) = \frac{4 m \gamma \Gamma_1 \Gamma_2 \rho \rho' (\rho + \rho') }{\sqrt{\demb}} \exp( \delta( \rho, \rho', \Gamma_1, \gamma, \demb) ).
\end{split}
\end{equation}
\end{lemma}
\begin{proof}
First, focusing on just the $i$-th query corresponding to $\mZ_{:,i}$ for some $i \in \iset{ m }$, let $\vv = -(\nicefrac{\gamma}{\sqrt{\demb}}) \sqdist(\mQ_\mu \mZ_{:,i}, \mK_\mu \mZ')$ and $\vv' = -(\nicefrac{\gamma}{\sqrt{\demb}} ) \sqdist(\mQ'_\mu \mZ_{:,i}, \mK_\mu \mZ')$, \Cref{lem:smax-l2-dist-disp} tells us that their dispersions are bounded as $D(\vv), D(\vv') \leq \delta( \rho, \rho', \Gamma_1,\gamma, \demb)$.
Now the $\norm{ \vv - \vv' }_1$ can be bounded as
\begin{align}
\norm{\vv - \vv'}_1 & = \sum_{j=1}^{m'} | v_j - v'_j |
= \frac{\gamma}{\sqrt{\demb}} \sum_{j=1}^{m'} \left|
\eudi{\mQ_\mu \mZ_{:,i}}{\mK_\mu \mZ'_{:,j}} - \eudi{\mQ'_\mu \mZ_{:,i}}{\mK_\mu \mZ'_{:,j}}
\right|
\\
& \leq \frac{\gamma}{\sqrt{\demb}} \sum_{j=1}^{m'}
\norm{ (\mQ_\mu - \mQ'_\mu) \mZ_{:,i} } \norm{ (\mQ_\mu + \mQ'_\mu) \mZ_{:,i} - 2 \mK_\mu \mZ'_{:,j} }
\\
& \leq \frac{\gamma}{\sqrt{\demb}} m'
\norm{ (\mQ_\mu - \mQ'_\mu) } \norm{ \mZ_{:,i} } \left( \norm{ \mQ_\mu \mZ_{:,i} } + \norm{ \mQ'_\mu \mZ_{:,i}} + 2 \norm{ \mK_\mu \mZ'_{:,j} } \right)
\\
& \leq \frac{m' \gamma}{\sqrt{\demb}}
\norm{ (\mQ_\mu - \mQ'_\mu) } \rho  \left( 2\Gamma_1 (\rho + \rho')  \right).
\end{align}
Combining this with \Cref{lem:smax-in-stab}, we can see that
\begin{align}
\norm{ \softmax(\vv) - \softmax(\vv') }_1
& \leq \frac{2 \exp(\delta(\rho, \rho', \Gamma_1, \gamma, \demb)) }{m'} \norm{ \vv - \vv' }_1
\\
& \leq
\frac{2 \exp(\delta(\rho, \rho', \Gamma_1, \gamma, \demb)) }{m'}
\frac{m' \gamma}{\sqrt{\demb}}
\rho  \left( 2 \Gamma_1 (\rho + \rho')  \right)
\norm{ \mQ_\mu - \mQ'_\mu }
\\
\label{eq:smax-stab-qproj}
& =
\frac{4 \gamma \Gamma_1 \rho (\rho + \rho') \exp(\delta(\rho, \rho', \Gamma_1, \gamma, \demb)) }{\sqrt{\demb}}
\norm{ \mQ_\mu - \mQ'_\mu }.
\end{align}
Denoting $\mA = \softmax(- (\nicefrac{\gamma}{\sqrt{\demb}}) \sqdist(\mQ_\mu \mZ, \mK_\mu \mZ'))$ and $\mA' = \softmax(- (\nicefrac{\gamma}{\sqrt{\demb}}) \sqdist(\mQ'_\mu \mZ, \mK_\mu \mZ'))$ with the softmax applied columnwise, we have $\mA, \mA' \in \R^{m' \times m}$. Now expanding the left-hand-side of \cref{eq:attn-op-qproj}, and we see that
\begin{align}
& \norm{
\attned\gamma(\mZ, \mZ'; (\mQ_\mu, \mK_\mu, \mV_\mu))
- \attned\gamma(\mZ, \mZ'; (\mQ_\mu', \mK_\mu, \mV_\mu))
}_{2,1}
\nonumber
\\ & \quad
= \sum_{i = 1}^m \norm{ \mV_\mu \sum_{j = 1}^{m'} (A_{ji} - A'_{ji}) \mZ'_{:,j}  }
\\ & \quad
\leq
\norm{\mV_\mu} \sum_{i = 1}^m \norm{ \sum_{j = 1}^{m'} (A_{ji} - A'_{ji}) \mZ'_{:,j}  }
\leq
\Gamma_2  \sum_{i = 1}^m  \sum_{j = 1}^{m'} | A_{ji} - A'_{ji} | \norm{ \mZ'_{:,j}  }
\\ & \quad
\leq
\Gamma_2 \rho'  \sum_{i = 1}^m  \norm{ \mA_{:,i} - \mA'_{:,i} }_1
\\ & \quad
\leq
\Gamma_2 \rho' m
\frac{4 \gamma \Gamma_1 \rho (\rho + \rho') \exp(\delta(\rho, \rho', \Gamma_1, \gamma, \demb)) }{\sqrt{\demb}}
\norm{ \mQ_\mu - \mQ'_\mu },
\end{align}
where we invoked \cref{eq:smax-stab-qproj}, noting that  $\mA_{:,i} = \softmax(\vv)$ and $\mA'_{:,i} = \softmax(\vv')$ and summed over $i \in \iset{ m }$. This gives us the coefficient $\kappa_3(m, \rho, \rho', \Gamma_1, \Gamma_2, \gamma, \demb)$ in \cref{eq:attn-op-qproj} in the statement of the result.
\end{proof}
\begin{lemma} \label{lem:attn-op-kproj}
Under conditions of \Cref{lem:smax-l2-dist-disp}, the attention operation is stable with respect to the key projection matrix as follows:
\begin{equation}\label{eq:attn-op-kproj}
\begin{split}
&
\norm{
\attned\gamma(\mZ, \mZ'; (\mQ_\mu, \mK_\mu, \mV_\mu))
- \attned\gamma(\mZ, \mZ'; (\mQ_\mu, \mK'_\mu, \mV_\mu))
}_{2,1}
\\ & \quad
\leq \kappa_4(m, \rho, \rho', \Gamma_1, \Gamma_2, \gamma, \demb)
 \norm{ \mK_\mu - \mK'_\mu },
\\
& \quad \text{where} \quad
\kappa_4(m, \rho, \rho', \Gamma_1, \Gamma_2, \gamma, \demb) = \frac{4 m \gamma \Gamma_1 \Gamma_2  \rho'^2 (\rho + \rho') }{\sqrt{\demb}} \exp( \delta(\rho, \rho', \Gamma_1,  \gamma, \demb) ).
\end{split}
\end{equation}
\end{lemma}
\begin{proof}
First, focusing on just the $i$-th query corresponding to $\mZ_{:,i}$ for some $i \in \iset{ m }$, let $\vv = -(\nicefrac{\gamma}{\sqrt{\demb}}) \sqdist(\mQ_\mu \mZ_{:,i}, \mK_\mu \mZ')$ and $\vv' = -(\nicefrac{\gamma}{\sqrt{\demb}}) \sqdist(\mQ_\mu \mZ_{:,i}, \mK'_\mu \mZ')$, \Cref{lem:smax-l2-dist-disp} tells us that their dispersions are bounded as $D(\vv), D(\vv') \leq \delta(\rho, \rho', \Gamma_1, \gamma, \demb)$.
Now the $\norm{ \vv - \vv' }_1$ can be bounded as
\begin{align}
\norm{\vv - \vv'}_1 & = \sum_{j=1}^{m'} | v_j - v'_j |
= \frac{\gamma}{\sqrt{\demb}} \sum_{j=1}^{m'} \left|
\eudi{\mQ_\mu \mZ_{:,i}}{\mK_\mu \mZ'_{:,j}} - \eudi{\mQ_\mu \mZ_{:,i}}{\mK'_\mu \mZ'_{:,j}}
\right|
\\
& \leq \frac{\gamma}{\sqrt{\demb}} \sum_{j=1}^{m'}
\norm{ (\mK_\mu - \mK'_\mu) \mZ'_{:,j} } \norm{ (\mK_\mu + \mK'_\mu) \mZ'_{:,j} - 2 \mQ_\mu \mZ_{:,i} }
\\
& \leq \frac{\gamma}{\sqrt{\demb}} m'
\norm{ (\mK_\mu - \mK'_\mu) } \norm{ \mZ'_{:,j} } \left( \norm{ \mK_\mu \mZ'_{:,j} } + \norm{ \mK'_\mu \mZ'_{:,j}} + 2 \norm{ \mQ_\mu \mZ_{:,i} } \right)
\\
& \leq \frac{m' \gamma}{\sqrt{\demb}}
\norm{ (\mK_\mu - \mK'_\mu) } \rho'  \left( 2\Gamma_1 (\rho' + \rho)  \right).
\end{align}
Combining this with \Cref{lem:smax-in-stab}, we can see that
\begin{align}
\norm{ \softmax(\vv) - \softmax(\vv') }_1
& \leq \frac{2 \exp(\delta(\rho, \rho', \Gamma_1, \gamma, \demb)) }{m'} \norm{ \vv - \vv' }_1
\\
& \leq
\frac{2 \exp(\delta(\rho, \rho', \Gamma_1, \gamma, \demb)) }{m'}
\frac{m' \gamma}{\sqrt{\demb}}
\rho'  \left( 2\Gamma_1 (\rho + \rho')  \right)
\norm{ \mK_\mu - \mK'_\mu }
\\
\label{eq:smax-stab-kproj}
& =
\frac{4 \gamma \Gamma_1 \rho' (\rho + \rho') \exp(\delta(\rho, \rho', \Gamma_1, \gamma, \demb)) }{\sqrt{\demb}}
\norm{ \mK_\mu - \mK'_\mu }.
\end{align}
Denoting $\mA = \softmax(- (\nicefrac{\gamma}{\sqrt{\demb}}) \sqdist(\mQ_\mu \mZ, \mK_\mu \mZ'))$ and $\mA' = \softmax(- (\nicefrac{\gamma}{\sqrt{\demb}}) \sqdist(\mQ_\mu \mZ, \mK'_\mu \mZ'))$ with the softmax applied columnwise, we have $\mA, \mA' \in \R^{m' \times m}$. Now expanding the left-hand-side of \cref{eq:attn-op-kproj}, and we see that
\begin{align}
& \norm{
\attned\gamma(\mZ, \mZ'; (\mQ_\mu, \mK_\mu, \mV_\mu))
- \attned\gamma(\mZ, \mZ'; (\mQ_\mu, \mK'_\mu, \mV_\mu))
}_{2,1}
\nonumber
\\ & \quad
= \sum_{i = 1}^m \norm{ \mV_\mu \sum_{j = 1}^{m'} (A_{ji} - A'_{ji}) \mZ'_{:,j}  }
\\ & \quad
\leq
\norm{\mV_\mu} \sum_{i = 1}^m \norm{ \sum_{j = 1}^{m'} (A_{ji} - A'_{ji}) \mZ'_{:,j}  }
\leq
\Gamma_2  \sum_{i = 1}^m  \sum_{j = 1}^{m'} | A_{ji} - A'_{ji} | \norm{ \mZ'_{:,j}  }
\\ & \quad
\leq
\Gamma_2 \rho'  \sum_{i = 1}^m  \norm{ \mA_{:,i} - \mA'_{:,i} }_1
\\ & \quad
\leq
\Gamma_2 \rho' m
\frac{4 \gamma \Gamma_1 \rho' (\rho + \rho') \exp(\delta(\rho, \rho', \Gamma_1, \gamma, \demb)) }{\sqrt{\demb}}
\norm{ \mK_\mu - \mK'_\mu },
\end{align}
where we invoked \cref{eq:smax-stab-kproj}, noting that  $\mA_{:,i} = \softmax(\vv)$ and $\mA'_{:,i} = \softmax(\vv')$ and summed over $i \in \iset{ m }$. This gives us the coefficient $\kappa_4(m, \rho, \rho', \Gamma_1, \Gamma_2, \gamma, \demb)$ in \cref{eq:attn-op-kproj} in the statement of the result.
\end{proof}
\begin{lemma} \label{lem:attn-op-keys}
Under conditions of \Cref{lem:smax-l2-dist-disp}, the attention operation is stable with respect to the keys as follows:
\begin{equation}\label{eq:attn-op-keys}
\begin{split}
& \forall \mZ',  \mZ'' \in \R^{\demb \times m'}, \forall j \in \iset{ m' }, \norm{\mZ'_{:,j}} \leq \rho', \norm{\mZ''_{:,j}} \leq \rho',
\\
& \norm{
\attned\gamma(\mZ, \mZ'; (\mQ_\mu, \mK_\mu, \mV_\mu))
- \attned\gamma(\mZ, \mZ''; (\mQ_\mu, \mK_\mu, \mV_\mu))
}_{2,1}
\\ & \quad
\leq \kappa_5(m, m', \rho, \rho', \Gamma_1, \Gamma_2, \gamma, \demb)
 \norm{ \mZ' - \mZ'' }_{2,1},
\\
& \quad \text{where} \quad
\kappa_5(m, m', \rho, \rho', \Gamma_1, \Gamma_2, \gamma, \demb)
\\ & \qquad
= \left( 1 + 4 \left(\frac{\gamma}{\sqrt{\demb}}\right) \Gamma_1^2 \rho' (\rho + \rho') \right) \left(\frac{m \Gamma_2}{m'}\right) \exp( \delta(\rho, \rho', \Gamma_1, \gamma, \demb) ).
\end{split}
\end{equation}
\end{lemma}
\begin{proof}
We can denote with $\mA = \softmax(- (\nicefrac{\gamma}{\sqrt{\demb}}) \sqdist(\mQ_\mu \mZ, \mK_\mu \mZ'))$ using  the keys $\mZ'$, and the corresponding $\mA' = \softmax(- (\nicefrac{\gamma}{\sqrt{\demb}}) \sqdist(\mQ_\mu \mZ, \mK_\mu \mZ''))$ with the keys $\mZ''$ (the softmax applied columnwise) we have $\mA, \mA' \in \R^{m' \times m}$. Now expanding the left-hand-side of \cref{eq:attn-op-keys}, and we see that
\begin{align}
& \nonumber
\norm{
\attned\gamma(\mZ, \mZ'; (\mQ_\mu, \mK_\mu, \mV_\mu))
- \attned\gamma(\mZ, \mZ''; (\mQ_\mu, \mK_\mu, \mV_\mu))
}_{2,1}
\\
& \quad
= \norm{ \mV_\mu \mZ' \mA - \mV_\mu \mZ'' \mA' }_{2,1}
\leq
\norm{ \mV_\mu (\mZ' - \mZ'') \mA }_{2,1}
+ \norm{ \mV_\mu \mZ'' (\mA - \mA') }_{2,1}.
\label{eq:attn-stab-input-terms}
\end{align}
We can bound each of these terms separately. First, noting that $A_{ji} \leq \exp(\delta(\rho, \rho', \Gamma_1, \gamma, \demb) ) / m'$ for any $(i, j)$ from \Cref{lem:smax-in-stab}, we have
\begin{align}
\norm{ \mV_\mu (\mZ' - \mZ'') \mA }_{2,1}
&
= \sum_{i=1}^m \norm{ \mV_\mu (\mZ' - \mZ'') \mA_{:,i} }
\leq \norm{\mV_\mu} \sum_{i = 1}^m \norm{ (\mZ' - \mZ'') \mA_{:,i} }
\\
&
= \norm{\mV_\mu} \sum_{i = 1}^m \norm{ \sum_{j = 1}^{m'} (\mZ'_{:,j} - \mZ''_{:,j}) A_{ji} }
\leq \norm{\mV_\mu} \sum_{i = 1}^m \sum_{j = 1}^{m'} \norm{ \mZ'_{:,j} - \mZ''_{:,j} } | A_{ji} |
\\
&
\leq \Gamma_2 \frac{\exp(\delta(\rho, \rho', \Gamma_1, \gamma, \demb) ) }{m'} \sum_{i = 1}^m \sum_{j = 1}^{m'}  \norm{ \mZ'_{:,j} - \mZ''_{:,j} }
\\ &
= \Gamma_2 \frac{\exp(\delta(\rho, \rho', \Gamma_1, \gamma, \demb) ) }{m'} m  \norm{ \mZ' - \mZ'' }_{2,1}.
\label{eq:attn-stab-input-term-1}
\end{align}
Next, focusing on just the $i$-th query corresponding to $\mZ_{:,i}$ for some $i \in \iset{ m }$, let $\vv =  -(\nicefrac{\gamma}{\sqrt{\demb}}) \sqdist(\mQ_\mu \mZ_{:,i}, \mK_\mu \mZ')$ and $\vv' = -(\nicefrac{\gamma}{\sqrt{\demb}}) \sqdist(\mQ_\mu \mZ_{:,i}, \mK_\mu \mZ'')$, \Cref{lem:smax-l2-dist-disp} tells us that their dispersions are bounded as $D(\vv), D(\vv') \leq \delta(\rho, \rho', \Gamma_1, \gamma, \demb)$. Note that $\mA_{:, i} = \softmax(\vv), \mA'_{:,i} = \softmax(\vv')$.
Now the $\norm{ \vv - \vv' }_1$ can be bounded as
\begin{align}
\norm{\vv - \vv'}_1 & = \sum_{j=1}^{m'} | v_j - v'_j |
= \frac{\gamma}{\sqrt{\demb}} \sum_{j=1}^{m'} \left|
\eudi{\mQ_\mu \mZ_{:,i}}{\mK_\mu \mZ'_{:,j}} - \eudi{\mQ_\mu \mZ_{:,i}}{\mK_\mu \mZ''_{:,j}}
\right|
\\
& \leq \frac{\gamma}{\sqrt{\demb}} \sum_{j=1}^{m'}
\norm{ \mK_\mu (\mZ'_{:,j} - \mZ''_{:,j} ) } \norm{ \mK_\mu ( \mZ'_{:,j} + \mZ''_{:,j} ) - 2 \mQ_\mu \mZ_{:,i} }
\\
& \leq \frac{\gamma}{\sqrt{\demb}} \sum_{j=1}^{m'}
\norm{ \mK_\mu } \norm{ \mZ'_{:,j} - \mZ''_{:,j} } \left( \norm{ \mK_\mu \mZ'_{:,j} } + \norm{ \mK_\mu \mZ''_{:,j}} + 2 \norm{ \mQ_\mu \mZ_{:,i} } \right)
\\
& \leq \frac{\gamma}{\sqrt{\demb}} \sum_{j=1}^{m'}
\Gamma_1 \norm{ \mZ'_{:,j} - \mZ''_{:,j} }  \left( 2\Gamma_1 (\rho' + \rho)  \right)
\\ &
= \frac{\gamma}{\sqrt{\demb}} \Gamma_1 \left( 2\Gamma_1 (\rho' + \rho)  \right)
\sum_{j=1}^{m'} \norm{ \mZ'_{:,j} - \mZ''_{:,j} }
\\ &
= \frac{\gamma}{\sqrt{\demb}} \Gamma_1 \left( 2\Gamma_1 (\rho' + \rho)  \right) \norm{ \mZ' - \mZ'' }_{2,1}.
\end{align}
Combining this with \Cref{lem:smax-in-stab}, we can see that
\begin{align}
\norm{ \softmax(\vv) - \softmax(\vv') }_1
& \leq \frac{2 \exp(\delta(\rho, \rho', \Gamma_1, \gamma, \demb)) }{m'} \norm{ \vv - \vv' }_1
\\
& \leq
\frac{2 \exp(\delta(\rho, \rho', \Gamma_1, \gamma, \demb)) }{m'}
\frac{\gamma}{\sqrt{\demb}} 2 \Gamma_1^2  (\rho' + \rho)  \norm{ \mZ' - \mZ'' }_{2,1}
\label{eq:smax-stab-keys}
\end{align}
Considering the second term on the right-hand-side of \cref{eq:attn-stab-input-terms}, we see that
\begin{align}
\norm{ \mV_\mu \mZ'' (\mA - \mA') }_{2,1}
&
= \sum_{i = 1}^m \norm{ \mV_\mu \mZ'' (\mA_{:,i} - \mA'_{:,i} ) }
= \sum_{i = 1}^m \norm{ \sum_{j = 1}^{m'} \mV_\mu \mZ''_{:,j} ( A_{ji} - A'_{ji} ) }
\\ &
\leq \sum_{i = 1}^m \sum_{j = 1}^{m'} \norm{ \mV_\mu \mZ''_{:,j} } | A_{ji} - A'_{ji} |
\leq \sum_{i = 1}^m \Gamma_2 \rho'  \sum_{j = 1}^{m'} | A_{ji} - A'_{ji} |
\\ &
=  \Gamma_2 \rho' \sum_{i = 1}^m  \norm{ \mA_{:,i} - \mA'_{:,i} }_1
\\ &
\leq \Gamma_2 \rho' \sum_{i = 1}^m
\frac{2 \exp(\delta(\rho, \rho', \Gamma_1, \gamma, \demb)) }{m'}
\frac{\gamma}{\sqrt{\demb}} 2 \Gamma_1^2  (\rho' + \rho)  \norm{ \mZ' - \mZ'' }_{2,1},
\label{eq:attn-stab-input-term-2}
\end{align}
where we applied \cref{eq:smax-stab-keys} noting that $\mA_{:,i} = \softmax(\vv)$ and $\mA'_{:,i} = \softmax(\vv')$.  Now combining \cref{eq:attn-stab-input-term-1} and \cref{eq:attn-stab-input-term-2} in \cref{eq:attn-stab-input-terms}, we have
\begin{align}
&
\norm{
\attned\gamma(\mZ, \mZ'; (\mQ_\mu, \mK_\mu, \mV_\mu))
- \attned\gamma(\mZ, \mZ''; (\mQ_\mu, \mK_\mu, \mV_\mu))
}_{2,1}
\nonumber
\\ & \quad
\leq
\norm{ \mZ' - \mZ'' }_{2,1}
\frac{m \Gamma_2}{m'}  \exp(\delta(\rho, \rho', \Gamma_1, \gamma, \demb) ) \left(
1 +
4 \frac{ \gamma  }{\sqrt{\demb}} \Gamma_1^2 \rho' (\rho' + \rho)
\right)
\\ & \quad
=  \norm{ \mZ' - \mZ'' }_{2,1} \kappa_5(m, m', \rho, \rho', \Gamma_1, \Gamma_2, \gamma, \demb),
\end{align}
as in \cref{eq:attn-op-keys} in the statement of the result.
\end{proof}
We summarize the above results --- the sensitivity of the attention module output to perturbations to its parameters (value projection matrix in \Cref{lem:attn-op-vproj}, query projection matrix in \Cref{lem:attn-op-qproj}, key projection matrix in \Cref{lem:attn-op-kproj}) and its input (keys in \Cref{lem:attn-op-keys}) --- in the following \Cref{tab:terms-dependencies} for convenience, noting the factors that affect the bounds in the first sub-table (such as the number of queries or key, their respective norm upper-bounds, parameter matrix upper-bounds and such), and how the factors affect the bounds in the second and third sub-tables (in terms of whether they worsen or improve with an increase in the factor).

\begin{table}[!!ht]
\caption{Different quantities affecting the various bounds in \Cref{lem:smax-l2-dist-disp}, \Cref{lem:smax-in-stab}, \Cref{lem:attn-op-vproj}, \Cref{lem:attn-op-qproj}, \Cref{lem:attn-op-kproj} and \Cref{lem:attn-op-keys}. First, we provide a description of the terms, and then we present the effect on the bound when the terms of interest are increased. In almost all cases, the bounds worsen with increase in the term. The only quantities that lead to a decrease in the terms are the embedding dimension $\demb$ and the number of key tokens $m'$. We}
\label{tab:terms-dependencies}
\centering
{\small
\begin{tabular}{ll}
\toprule
Term & Description \\
\midrule
$m$ & Number of query tokens \\
$m'$ & Number of key tokens \\
$\rho$ & Upper-bound on the query token norms \\
$\rho'$ & Upper-bound on the key token norms \\
$\Gamma_1$ & Upper-bound on the spectral norm of the query/key projection matrix \\
$\Gamma_2$ & Upper-bound on the spectral norm of the value projection matrix \\
$\gamma$ & Inverse temperature for the attention operation in \cref{eq:attn-op} \\
$\demb$ & Token embedding dimension for the transformer architecture \\
$\delta(\rho, \rho', \Gamma_1, \gamma, \demb)$ & Maximum dispersion in the attention softmax (defined in \Cref{lem:smax-l2-dist-disp}) \\
$\kappa_2(m, \rho')$ & Stability constant w.r.t. the value projection matrix in \Cref{lem:attn-op-vproj} \\
$\kappa_3(m, \rho, \rho', \Gamma_1, \Gamma_2, \gamma, \demb)$ & Stability constant w.r.t. the query projection matrix in \Cref{lem:attn-op-qproj} \\
$\kappa_4(m, \rho, \rho', \Gamma_1, \Gamma_2, \gamma, \demb)$ & Stability constant w.r.t. the key projection matrix in \Cref{lem:attn-op-kproj} \\
$\kappa_5(m, m', \rho, \rho', \Gamma_1, \Gamma_2, \gamma, \demb)$ & Stability constant w.r.t. the key tokens in \Cref{lem:attn-op-keys} \\
\bottomrule
\end{tabular}
{\footnotesize\begin{tabular}{lccc}
\toprule
&
$\delta(\rho, \rho', \Gamma_1, \gamma, \demb)$ (\cref{lem:smax-l2-dist-disp}) &
$\kappa_2(m, \rho')$ (\cref{lem:attn-op-vproj}) &
$\kappa_3(m, \rho, \rho', \Gamma_1, \Gamma_2, \gamma, \demb)$ (\cref{lem:attn-op-qproj}) \\
\midrule
&
$(\nicefrac{\gamma}{\sqrt{\demb}}) \Gamma_1^2 (\rho + \rho')^2$ &
$m \rho'$ &
$4 (\nicefrac{\gamma}{\sqrt{\demb}}) m \rho \rho' (\rho+\rho') \Gamma_1 \Gamma_2 \exp(\delta(\cdots))$
\\
\midrule
$m$ &  & \up & \up \\
$m'$ & \\
$\rho$ & \up &  & \up \\
$\rho'$ & \up & \up & \up \\
$\Gamma_1$ & \up &  & \up \\
$\Gamma_2$ & & & \up \\
$\gamma$ & \up &  & \up \\
$\demb$ & \down & & \down \\
\bottomrule
\end{tabular}}
{\footnotesize \begin{tabular}{lcc}
\toprule
&
$\kappa_4(m, \rho, \rho', \Gamma_1, \Gamma_2, \gamma, \demb)$ (\cref{lem:attn-op-kproj}) &
$\kappa_5(m, m', \rho, \rho', \Gamma_1, \Gamma_2, \gamma, \demb)$ (\cref{lem:attn-op-keys})\\
\midrule
&
$4 (\nicefrac{\gamma}{\sqrt{\demb}}) m \rho'^2 (\rho+\rho') \Gamma_1 \Gamma_2 \exp(\delta(\cdots))$ &
$(\nicefrac{m\Gamma_2}{m'})(1 + 4 (\nicefrac{\gamma}{\sqrt{\demb}}) \Gamma_1^2 \rho'(\rho+\rho')) \exp(\delta(\cdots))$ \\
\midrule
$m$ &  \up & \up \\
$m'$ &  & \down \\
$\rho$ & \up & \up \\
$\rho'$ & \up & \up \\
$\Gamma_1$ & \up & \up \\
$\Gamma_2$ & \up & \up \\
$\gamma$ & \up & \up \\
$\demb$ & \down & \down \\
\bottomrule
\end{tabular}}
}
\end{table}
\begin{table}[!!h]
\caption{The value of the relevant terms from \Cref{tab:terms-dependencies} in the different attention operations in \Cref{def:attn-op}. We write the attention operations compactly as $\attned\gamma(\mZ, \mZ', (\cdot, \mu))$ for query tokens $\mZ \in \R^{\demb \times m}$, key tokens $\mZ' \in \R^{\demb \times m'}$, where (i)~the query token norms are bounded as $\norm{\mZ_{:,i} } \leq \rho$ for all $i \in \iset{ m }$, (ii)~the key token norms are $\norm{\mZ'_{:,j}} \leq \rho'$ for all $j \in \iset{ m' }$, and (iii)~the $(\cdots, \mu)$ corresponds to the appropriate $(\mQ_\mu, \mK_\mu, \mV_\mu)$ query/key/value projection matrices. The bounds on the norms of the key and query tokens are obtained from \Cref{lem:out-bound}.}
\label{tab:term-value-attn-op}
\centering
{\footnotesize\begin{tabular}{lcccc}
\toprule
&
$\attned\gamma(\bar{\mX}^{(0)}, \bar{\mC}^{(0)}; (\cdot, 1))$ &
$\attned\gamma(\bar{\mX}^{(0)}, \bar{\mX}^{(0)}; (\cdot, 2))$ &
$\attned\gamma(\bar{\mC}^{(0)}, \bar{\mX}^{(1)}; (\cdot, 3))$ &
$\attned\gamma(\bar{\mC}^{(0)}, \bar{\mC}^{(0)}; (\cdot, 4))$ \\
\midrule
$m$    & $n$ & $n$ & $k$ & $k$ \\
$m'$   & $k$ & $n$ & $n$ & $k$ \\
$\rho$ & $\sqrt{d}$ & $\sqrt{d}$ & $\sqrt{d_E}$ & $\sqrt{d_E}$ \\
$\rho'$ & $\sqrt{d_E}$ & $\sqrt{d}$ & $(1+\Gamma_2)\sqrt{d} + \Gamma_2 \sqrt{d_E}$ & $\sqrt{d_E}$ \\
\bottomrule
\end{tabular}}
\end{table}

We can now apply these results to the attention modules in \Cref{def:attn-op} to bound the stability of the model output $\norm{ \mC(\Theta) - \mC(\Theta') }_{2,1}$. In \Cref{tab:term-value-attn-op}, we note the relevant quantities necessary to apply our general results in \Cref{tab:terms-dependencies} to the specific attention operations in \Cref{def:attn-op}. This presentation is for the ease of exposition, and we will formally present the bounds in the following.

\begin{theorem} \label{thm:out-stab-params}
Assuming bounded inputs (assumption~\ref{asp:bnd-input}) and bounded model parameters (assumption~\ref{asp:bnd-params}), we can bound the stability of the model output $\mC(\Theta)$ with respect to perturbations to the model parameters $\Theta$ as follows for any $\Theta, \Theta'$ given $(\mX, \mC)$ as input:
\begin{equation} \label{eq:out-stab-params}
\begin{split}
& \norm{ \mC(\Theta) - \mC(\Theta') }_{2,1}
\leq \phi \cdot \norm{ \Theta - \Theta' },
\\
&
\quad \text{where }
\phi \sim O \left(
\max\{ 1, \gamma^2 \} \cdot
\left( k \sqrt{d_E} \right) \cdot
\max\left\{
 1, \left(\frac{d_E}{\sqrt{\demb}} \right)^2
\right\}
\cdot
\textsf{poly}\left( \exp \left( \gamma \frac{d_E}{\sqrt{\demb}} \right) \right)
\right).
\end{split}
\end{equation}
\end{theorem}

\begin{proof}
Given that $\Theta = \{ (\mQ_\mu, \mK_\mu, \mV_\mu), \mu = 1, \ldots, 4 \}$ and $\Theta' = \{ (\mQ'_\mu, \mK'_\mu, \mV'_\mu), \mu = 1, \ldots, 4 \}$, we will study the effect of modifying a single weight matrix at a time for ease of exposition (for example, effect of modifying $\mQ_1$ to $\mQ'_1$). In terms of notation, when we modify a single weight matrix (say $\mQ_1$ to $\mQ'_1$), and compare the outputs, we denote the outputs of the respective models as $\mC(\mQ_1)$ and $\mC(\mQ'_1)$ to make it explicit that we are studying the effect of just modifying $\mQ_1$ to $\mQ'_1$. To make explicit the effect of the model parameters, we will also write the final embeddings of points and clusters centers as $\bar{\mX}^{(1)}(\mQ_\mu)$ vs $\bar{\mX}^{(1)}(\mQ'_\mu)$ and $\bar{\mC}^{(1)}(\mQ_\mu)$ vs $\bar{\mC}^{(1)}(\mQ'_\mu)$ only if the change in the considered weight matrix affects the token embeddings. Note that modifying any of the $(\mQ_\mu, \mK_\mu, \mV_\mu), \mu = 3,4$ does not affect the final point embeddings $\bar{\mX}^{(1)}$, and in those situations, we will just denote the point embeddings as $\bar{\mX}^{(1)}$.

With that in mind, we first study the effect of modifying $(\mQ_4, \mK_4, \mV_4)$ one at a time, applying \Cref{lem:attn-op-qproj}, \Cref{lem:attn-op-kproj}, \Cref{lem:attn-op-vproj}, respectively. We see that
\begin{align}
\highlightmath{Gray}{ \norm{ \mC(\mQ_4) - \mC(\mQ'_4) }_{2,1} }
& \leq  \norm{ \bar{\mC}^{(1)}(\mQ_4) - \bar{\mC}^{(1)}(\mQ'_4) }_{2,1}
\\
& = \norm{ \attned\gamma(\bar{\mC}^{(0)}, \bar{\mC}^{(0)}; (\mQ_4, \mK_4, \mV_4))
- \attned\gamma(\bar{\mC}^{(0)}, \bar{\mC}^{(0)}; (\mQ'_4, \mK_4, \mV_4))  }_{2,1}
\\
& \leq \highlightmath{Gray}{
\kappa_3(k, \sqrt{d_E}, \sqrt{d_E}, \Gamma_1, \Gamma_2, \gamma, \demb) \norm{ \mQ_4 - \mQ'_4 }
}
\label{eq:q4-stab}
\\
\highlightmath{Gray}{ \norm{ \mC(\mK_4) - \mC(\mK'_4) }_{2,1} }
& \leq  \norm{ \bar{\mC}^{(1)}(\mK_4) - \bar{\mC}^{(1)}(\mK'_4) }_{2,1}
\\
& = \norm{ \attned\gamma(\bar{\mC}^{(0)}, \bar{\mC}^{(0)}; (\mQ_4, \mK_4, \mV_4))
- \attned\gamma(\bar{\mC}^{(0)}, \bar{\mC}^{(0)}; (\mQ_4, \mK'_4, \mV_4))  }_{2,1}
\\
& \leq \highlightmath{Gray}{
\kappa_4(k, \sqrt{d_E}, \sqrt{d_E}, \Gamma_1, \Gamma_2, \gamma, \demb) \norm{ \mK_4 - \mK'_4 }
}
\label{eq:k4-stab}
\\
\highlightmath{Gray}{ \norm{ \mC(\mV_4) - \mC(\mV'_4) }_{2,1} }
& \leq  \norm{ \bar{\mC}^{(1)}(\mV_4) - \bar{\mC}^{(1)}(\mV'_4) }_{2,1}
\\
& = \norm{ \attned\gamma(\bar{\mC}^{(0)}, \bar{\mC}^{(0)}; (\mQ_4, \mK_4, \mV_4))
- \attned\gamma(\bar{\mC}^{(0)}, \bar{\mC}^{(0)}; (\mQ_4, \mK_4, \mV'_4))  }_{2,1}
\\
& \leq \highlightmath{Gray}{\kappa_2(k, \sqrt{d_E}) \norm{ \mV_4 - \mV'_4 } }.
\label{eq:v4-stab}
\end{align}
Next, we consider the effect of modifying $(\mQ_3, \mK_3, \mV_3)$ one at a time, applying \Cref{lem:attn-op-qproj}, \Cref{lem:attn-op-kproj}, \Cref{lem:attn-op-vproj}, respectively:
\begin{align}
\highlightmath{Gray}{ \norm{ \mC(\mQ_3) - \mC(\mQ'_3) }_{2,1} }
& = \norm{ \attned\gamma(\bar{\mC}^{(0)}, \bar{\mX}^{(1)}; (\mQ_3, \mK_3, \mV_3))
- \attned\gamma(\bar{\mC}^{(0)}, \bar{\mX}^{(1)}; (\mQ'_3, \mK_3, \mV_3))  }_{2,1}
\\
& \leq \highlightmath{Gray}{
\kappa_3(k, \sqrt{d_E}, \sqrt{d_E} (1+2\Gamma_2), \Gamma_1, \Gamma_2, \gamma, \demb)
\norm{ \mQ_3 - \mQ'_3 }
} \label{eq:q3-stab}
\\
\highlightmath{Gray}{ \norm{ \mC(\mK_3) - \mC(\mK'_3) }_{2,1} }
& = \norm{ \attned\gamma(\bar{\mC}^{(0)}, \bar{\mX}^{(1)}; (\mQ_3, \mK_3, \mV_3))
- \attned\gamma(\bar{\mC}^{(0)}, \bar{\mX}^{(1)}; (\mQ_3, \mK'_3, \mV_3))  }_{2,1}
\\
& \leq \highlightmath{Gray}{
\kappa_4(k, \sqrt{d_E}, \sqrt{d_E} (1+2\Gamma_2), \Gamma_1, \Gamma_2, \gamma, \demb)
\norm{ \mK_3 - \mK'_3 }
} \label{eq:k3-stab}
\\
\highlightmath{Gray}{ \norm{ \mC(\mV_3) - \mC(\mV'_3) }_{2,1} }
& = \norm{ \attned\gamma(\bar{\mC}^{(0)}, \bar{\mX}^{(1)}; (\mQ_3, \mK_3, \mV_3))
- \attned\gamma(\bar{\mC}^{(0)}, \bar{\mX}^{(1)}; (\mQ_3, \mK_3, \mV'_3))  }_{2,1}
\\
& \leq \highlightmath{Gray}{
\kappa_2(k, \sqrt{d_E} (1+2\Gamma_2)) \norm{ \mV_3 - \mV'_3 }
} \label{eq:v3-stab}
\end{align}
Studying the effect of modifying $(\mQ_1, \mK_1, \mV_1)$ one at a time, we need to incorporate the effect of the modified $\bar{\mX}^{(1)}$ using \Cref{lem:attn-op-keys} along with the use of \Cref{lem:attn-op-qproj}, \Cref{lem:attn-op-kproj}, \Cref{lem:attn-op-vproj}:
\begin{align}
&
\highlightmath{Gray}{ \norm{ \mC(\mQ_1) - \mC(\mQ'_1) }_{2,1} }
\\ & \quad
= \norm{ \attned\gamma(\bar{\mC}^{(0)}, \bar{\mX}^{(1)} (\mQ_1); (\mQ_3, \mK_3, \mV_3))
- \attned\gamma(\bar{\mC}^{(0)}, \bar{\mX}^{(1)}(\mQ'_1); (\mQ_3, \mK_3, \mV_3))  }_{2,1}
\\ & \quad
\leq \kappa_5(k, n, \sqrt{d_E}, \sqrt{d_E} (1+2\Gamma_2), \Gamma_1, \Gamma_2, \gamma, \demb)
\norm{ \bar{\mX}^{(1)}(\mQ_1) - \bar{\mX}^{(1)}(\mQ'_1) }_{2,1}
\\ & \quad
\leq \kappa_5(k, n, \sqrt{d_E}, \sqrt{d_E} (1+2\Gamma_2), \Gamma_1, \Gamma_2, \gamma, \demb)
\nonumber
\\ & \qquad\qquad
\norm{ \attned\gamma(\bar{\mX}^{(0)}, \bar{\mC}^{(0)}; (\mQ_1, \mK_1, \mV_1))
- \attned\gamma(\bar{\mX}^{(0)}, \bar{\mC}^{(0)}; (\mQ'_1, \mK_1, \mV_1))  }_{2,1}
\\ & \quad
\leq \highlightmath{Gray}{
\kappa_5(k, n, \sqrt{d_E}, \sqrt{d_E} (1+2\Gamma_2), \Gamma_1, \Gamma_2, \gamma, \demb)
\cdot \kappa_3(n, \sqrt{d}, \sqrt{d_E}, \Gamma_1, \Gamma_2, \gamma, \demb)
\norm{ \mQ_1 - \mQ'_1 }
} \label{eq:q1-stab}
\\
&
\highlightmath{Gray}{ \norm{ \mC(\mK_1) - \mC(\mK'_1) }_{2,1} }
\\ & \quad
= \norm{ \attned\gamma(\bar{\mC}^{(0)}, \bar{\mX}^{(1)} (\mK_1); (\mQ_3, \mK_3, \mV_3))
- \attned\gamma(\bar{\mC}^{(0)}, \bar{\mX}^{(1)}(\mK'_1); (\mQ_3, \mK_3, \mV_3))  }_{2,1}
\\ & \quad
\leq \kappa_5(k, n, \sqrt{d_E}, \sqrt{d_E} (1+2\Gamma_2), \Gamma_1, \Gamma_2, \gamma, \demb)
\norm{ \bar{\mX}^{(1)}(\mK_1) - \bar{\mX}^{(1)}(\mK'_1) }_{2,1}
\\ & \quad
\leq \kappa_5(k, n, \sqrt{d_E}, \sqrt{d_E} (1+2\Gamma_2), \Gamma_1, \Gamma_2, \gamma, \demb)
\nonumber
\\ & \qquad\qquad
\norm{ \attned\gamma(\bar{\mX}^{(0)}, \bar{\mC}^{(0)}; (\mQ_1, \mK_1, \mV_1))
- \attned\gamma(\bar{\mX}^{(0)}, \bar{\mC}^{(0)}; (\mQ_1, \mK'_1, \mV_1))  }_{2,1}
\\ & \quad
\leq \highlightmath{Gray}{
\kappa_5(k, n, \sqrt{d_E}, \sqrt{d_E} (1+2\Gamma_2), \Gamma_1, \Gamma_2, \gamma, \demb)
\cdot \kappa_4(n, \sqrt{d}, \sqrt{d_E}, \Gamma_1, \Gamma_2, \gamma, \demb)
\norm{ \mK_1 - \mK'_1 }
} \label{eq:k1-stab}
\\
&
\highlightmath{Gray}{ \norm{ \mC(\mV_1) - \mC(\mV'_1) }_{2,1} }
\\ & \quad
= \norm{ \attned\gamma(\bar{\mC}^{(0)}, \bar{\mX}^{(1)} (\mV_1); (\mQ_3, \mK_3, \mV_3))
- \attned\gamma(\bar{\mC}^{(0)}, \bar{\mX}^{(1)}(\mV'_1); (\mQ_3, \mK_3, \mV_3))  }_{2,1}
\\ & \quad
\leq \kappa_5(k, n, \sqrt{d_E}, \sqrt{d_E} (1+2\Gamma_2), \Gamma_1, \Gamma_2, \gamma, \demb)
\norm{ \bar{\mX}^{(1)}(\mV_1) - \bar{\mX}^{(1)}(\mV'_1) }_{2,1}
\\ & \quad
\leq \kappa_5(k, n, \sqrt{d_E}, \sqrt{d_E} (1+2\Gamma_2), \Gamma_1, \Gamma_2, \gamma, \demb)
\nonumber
\\ & \qquad\qquad
\norm{ \attned\gamma(\bar{\mX}^{(0)}, \bar{\mC}^{(0)}; (\mQ_1, \mK_1, \mV_1))
- \attned\gamma(\bar{\mX}^{(0)}, \bar{\mC}^{(0)}; (\mQ_1, \mK_1, \mV'_1))  }_{2,1}
\\ & \quad
\leq \highlightmath{Gray}{
\kappa_5(k, n, \sqrt{d_E}, \sqrt{d_E} (1+2\Gamma_2), \Gamma_1, \Gamma_2, \gamma, \demb)
\cdot \kappa_2(n, \sqrt{d_E}) \norm{ \mV_1 - \mV'_1 }
} \label{eq:v1-stab}
\end{align}
Finally, we consider $(\mQ_2, \mK_2, \mV_2)$ employing the previously used results:
\begin{align}
&
\highlightmath{Gray}{ \norm{ \mC(\mQ_2) - \mC(\mQ'_2) }_{2,1} }
\\ & \quad
= \norm{ \attned\gamma(\bar{\mC}^{(0)}, \bar{\mX}^{(1)} (\mQ_2); (\mQ_3, \mK_3, \mV_3))
- \attned\gamma(\bar{\mC}^{(0)}, \bar{\mX}^{(1)}(\mQ'_2); (\mQ_3, \mK_3, \mV_3))  }_{2,1}
\\ & \quad
\leq \kappa_5(k, n, \sqrt{d_E}, \sqrt{d_E} (1+2\Gamma_2), \Gamma_1, \Gamma_2, \gamma, \demb)
\norm{ \bar{\mX}^{(1)}(\mQ_2) - \bar{\mX}^{(1)}(\mQ'_2) }_{2,1}
\\
& \quad
\leq \kappa_5(k, n, \sqrt{d_E}, \sqrt{d_E} (1+2\Gamma_2), \Gamma_1, \Gamma_2, \gamma, \demb)
\nonumber
\\ & \qquad\qquad
\norm{ \attned\gamma(\bar{\mX}^{(0)}, \bar{\mX}^{(0)}; (\mQ_2, \mK_2, \mV_2))
- \attned\gamma(\bar{\mX}^{(0)}, \bar{\mX}^{(0)}; (\mQ'_2, \mK_2, \mV_2))  }_{2,1}
\\
& \quad
\leq \highlightmath{Gray}{
\kappa_5(k, n, \sqrt{d_E}, \sqrt{d_E} (1+2\Gamma_2), \Gamma_1, \Gamma_2, \gamma, \demb)
\kappa_3(n, \sqrt{d}, \sqrt{d}, \Gamma_1, \Gamma_2, \gamma, \demb)
\norm{ \mQ_2 - \mQ'_2 }
} \label{eq:q2-stab}
\\
&
\highlightmath{Gray}{ \norm{ \mC(\mK_2) - \mC(\mK'_2) }_{2,1} }
\\ & \quad
= \norm{ \attned\gamma(\bar{\mC}^{(0)}, \bar{\mX}^{(1)} (\mK_2); (\mQ_3, \mK_3, \mV_3))
- \attned\gamma(\bar{\mC}^{(0)}, \bar{\mX}^{(1)}(\mK'_2); (\mQ_3, \mK_3, \mV_3))  }_{2,1}
\\
& \quad
\leq \kappa_5(k, n, \sqrt{d_E}, \sqrt{d_E} (1+2\Gamma_2), \Gamma_1, \Gamma_2, \gamma, \demb)
\norm{ \bar{\mX}^{(1)}(\mK_2) - \bar{\mX}^{(1)}(\mK'_2) }_{2,1}
\\
& \quad
\leq \kappa_5(k, n, \sqrt{d_E}, \sqrt{d_E} (1+2\Gamma_2), \Gamma_1, \Gamma_2, \gamma, \demb)
\nonumber
\\ & \qquad\qquad
\norm{ \attned\gamma(\bar{\mX}^{(0)}, \bar{\mX}^{(0)}; (\mQ_2, \mK_2, \mV_2))
- \attned\gamma(\bar{\mX}^{(0)}, \bar{\mX}^{(0)}; (\mQ_2, \mK'_2, \mV_2))  }_{2,1}
\\
& \quad
\leq \highlightmath{Gray}{
\kappa_5(k, n, \sqrt{d_E}, \sqrt{d_E} (1+2\Gamma_2), \Gamma_1, \Gamma_2, \gamma, \demb)
\kappa_4(n, \sqrt{d}, \sqrt{d}, \Gamma_1, \Gamma_2, \gamma, \demb)
\norm{ \mK_2 - \mK'_2 }
} \label{eq:k2-stab}
\\
&
\highlightmath{Gray}{ \norm{ \mC(\mV_2) - \mC(\mV'_2) }_{2,1} }
\\ & \quad
= \norm{ \attned\gamma(\bar{\mC}^{(0)}, \bar{\mX}^{(1)} (\mV_2); (\mQ_3, \mK_3, \mV_3))
- \attned\gamma(\bar{\mC}^{(0)}, \bar{\mX}^{(1)}(\mV'_2); (\mQ_3, \mK_3, \mV_3))  }_{2,1}
\\
& \quad
\leq \kappa_5(k, n, \sqrt{d_E}, \sqrt{d_E} (1+2\Gamma_2), \Gamma_1, \Gamma_2, \gamma, \demb)
\norm{ \bar{\mX}^{(1)}(\mV_2) - \bar{\mX}^{(1)}(\mV'_2) }_{2,1}
\\
& \quad
\leq \kappa_5(k, n, \sqrt{d_E}, \sqrt{d_E} (1+2\Gamma_2), \Gamma_1, \Gamma_2, \gamma, \demb)
\nonumber
\\ & \qquad\qquad
\norm{ \attned\gamma(\bar{\mX}^{(0)}, \bar{\mX}^{(0)}; (\mQ_2, \mK_2, \mV_2))
- \attned\gamma(\bar{\mX}^{(0)}, \bar{\mX}^{(0)}; (\mQ_2, \mK_2, \mV'_2))  }_{2,1}
\\
& \quad
\leq \highlightmath{Gray}{
\kappa_5(k, n, \sqrt{d_E}, \sqrt{d_E} (1+2\Gamma_2), \Gamma_1, \Gamma_2, \gamma, \demb)
\kappa_2(n, \sqrt{d}) \norm{ \mV_2 - \mV'_2 }
} \label{eq:v2-stab}
\end{align}
Putting all these together, and noting the definition of $\norm{ \Theta - \Theta' }$ as the maximum of the norms of the differences of the individual weight matrices, we can see that:
\begin{equation}
\norm{ \mC(\Theta) - \mC(\Theta') }_{2,1}
\leq
\phi \cdot \norm{ \Theta - \Theta' },
\end{equation}
where the quantity $\phi$ is defined as (also noting the equations from which the quantity is derived):
\begin{align}
\phi &=
\max \left\{
\begin{array}{ll}
\kappa_5(k, n, \sqrt{d_E}, \sqrt{d_E} (1+2\Gamma_2), \Gamma_1, \Gamma_2, \gamma, \demb)
\cdot
\kappa_3(n, \sqrt{d}, \sqrt{d_E}, \Gamma_1, \Gamma_2, \gamma, \demb)
& \eqref{eq:q1-stab}
\\
\kappa_5(k, n, \sqrt{d_E}, \sqrt{d_E} (1+2\Gamma_2), \Gamma_1, \Gamma_2, \gamma, \demb)
\cdot
\kappa_4(n, \sqrt{d}, \sqrt{d_E}, \Gamma_1, \Gamma_2, \gamma, \demb)
& \eqref{eq:k1-stab}
\\
\kappa_5(k, n, \sqrt{d_E}, \sqrt{d_E} (1+2\Gamma_2), \Gamma_1, \Gamma_2, \gamma, \demb)
\cdot
\kappa_2(n, \sqrt{d_E})
& \eqref{eq:v1-stab}
\\
\kappa_5(k, n, \sqrt{d_E}, \sqrt{d_E} (1+2\Gamma_2), \Gamma_1, \Gamma_2, \gamma, \demb)
\cdot
\kappa_3(n, \sqrt{d}, \sqrt{d}, \Gamma_1, \Gamma_2, \gamma, \demb)
& \eqref{eq:q2-stab}
\\
\kappa_5(k, n, \sqrt{d_E}, \sqrt{d_E} (1+2\Gamma_2), \Gamma_1, \Gamma_2, \gamma, \demb)
\cdot
\kappa_4(n, \sqrt{d}, \sqrt{d}, \Gamma_1, \Gamma_2, \gamma, \demb)
& \eqref{eq:k2-stab}
\\
\kappa_5(k, n, \sqrt{d_E}, \sqrt{d_E} (1+2\Gamma_2), \Gamma_1, \Gamma_2, \gamma, \demb)
\cdot
\kappa_2(n, \sqrt{d})
& \eqref{eq:v2-stab}
\\
\kappa_3(k, \sqrt{d_E}, \sqrt{d_E} (1+2\Gamma_2), \Gamma_1, \Gamma_2, \gamma, \demb)
& \eqref{eq:q3-stab}
\\
\kappa_4(k, \sqrt{d_E}, \sqrt{d_E} (1+2\Gamma_2), \Gamma_1, \Gamma_2, \gamma, \demb)
& \eqref{eq:k3-stab}
\\
\kappa_2(k, \sqrt{d_E} (1+2\Gamma_2))
& \eqref{eq:v3-stab}
\\
\kappa_3(k, \sqrt{d_E}, \sqrt{d_E}, \Gamma_1, \Gamma_2, \gamma, \demb)
& \eqref{eq:q4-stab}
\\
\kappa_4(k, \sqrt{d_E}, \sqrt{d_E}, \Gamma_1, \Gamma_2, \gamma, \demb)
& \eqref{eq:k4-stab}
\\
\kappa_2(k, \sqrt{d_E})
& \eqref{eq:v4-stab}
\end{array}
\right\}
\end{align}
To simplify the above definition of $\phi$, we note that, for a fixed $(m, \rho, \Gamma_1, \Gamma_2, \gamma, \demb)$, a larger $\rho'$ implies larger constant. Thus, we can see that
\begin{align}
\kappa_2(k, \sqrt{d_E}) & < \kappa_2(k, \sqrt{d_E}(1 + 2\Gamma_2)) \\
\kappa_2(k, \sqrt{d}) & < \kappa_2(k, \sqrt{d_E}) \\
\kappa_3(k, \sqrt{d_E}, \sqrt{d_E}, \Gamma_1, \Gamma_2, \gamma, \demb)
& <
\kappa_3(k, \sqrt{d_E}, \sqrt{d_E} (1+2\Gamma_2), \Gamma_1, \Gamma_2, \gamma, \demb)
\\
\kappa_3(n, \sqrt{d}, \sqrt{d}, \Gamma_1, \Gamma_2, \gamma, \demb)
& <
\kappa_3(n, \sqrt{d}, \sqrt{d_E}, \Gamma_1, \Gamma_2, \gamma, \demb)
\\
\kappa_4(k, \sqrt{d_E}, \sqrt{d_E}, \Gamma_1, \Gamma_2, \gamma, \demb)
& <
\kappa_4(k, \sqrt{d_E}, \sqrt{d_E} (1+2\Gamma_2), \Gamma_1, \Gamma_2, \gamma, \demb)
\\
\kappa_4(n, \sqrt{d}, \sqrt{d}, \Gamma_1, \Gamma_2, \gamma, \demb)
& <
\kappa_4(n, \sqrt{d}, \sqrt{d_E}, \Gamma_1, \Gamma_2, \gamma, \demb),
\end{align}
thus simplifying further to
\begin{align}
\phi &=
\max \left\{
\begin{array}{l}
\kappa_5(k, n, \sqrt{d_E}, \sqrt{d_E} (1+2\Gamma_2), \Gamma_1, \Gamma_2, \gamma, \demb)
\cdot
\kappa_3(n, \sqrt{d}, \sqrt{d_E}, \Gamma_1, \Gamma_2, \gamma, \demb)
\\
\kappa_5(k, n, \sqrt{d_E}, \sqrt{d_E} (1+2\Gamma_2), \Gamma_1, \Gamma_2, \gamma, \demb)
\cdot
\kappa_4(n, \sqrt{d}, \sqrt{d_E}, \Gamma_1, \Gamma_2, \gamma, \demb)
\\
\kappa_5(k, n, \sqrt{d_E}, \sqrt{d_E} (1+2\Gamma_2), \Gamma_1, \Gamma_2, \gamma, \demb)
\cdot
\kappa_2(n, \sqrt{d_E})
\\
\kappa_3(k, \sqrt{d_E}, \sqrt{d_E} (1+2\Gamma_2), \Gamma_1, \Gamma_2, \gamma, \demb)
\\
\kappa_4(k, \sqrt{d_E}, \sqrt{d_E} (1+2\Gamma_2), \Gamma_1, \Gamma_2, \gamma, \demb)
\\
\kappa_2(k, \sqrt{d_E} (1+2\Gamma_2))
\end{array}
\right\}
\end{align}
Next given the definitions of $\kappa_3$ and $\kappa_4$ in \Cref{lem:attn-op-qproj} and \Cref{lem:attn-op-kproj} respectively, we can see that if $\rho < \rho'$, $\kappa_4(\cdot)$ provides a larger constant, giving us:
\begin{align}
\phi &=
\max \left\{
\begin{array}{l}
\kappa_5(k, n, \sqrt{d_E}, \sqrt{d_E} (1+2\Gamma_2), \Gamma_1, \Gamma_2, \gamma, \demb)
\cdot
\kappa_4(n, \sqrt{d}, \sqrt{d_E}, \Gamma_1, \Gamma_2, \gamma, \demb)
\\
\kappa_5(k, n, \sqrt{d_E}, \sqrt{d_E} (1+2\Gamma_2), \Gamma_1, \Gamma_2, \gamma, \demb)
\cdot
\kappa_2(n, \sqrt{d_E})
\\
\kappa_4(k, \sqrt{d_E}, \sqrt{d_E} (1+2\Gamma_2), \Gamma_1, \Gamma_2, \gamma, \demb)
\\
\kappa_2(k, \sqrt{d_E} (1+2\Gamma_2))
\end{array}
\right\}
\end{align}
Writing things out explicitly while highlighting the dependencies on problem and architectural choices, $n,k, \demb, d_E$, we first apply the definition of $\kappa_2(\cdots)$ in \Cref{lem:attn-op-vproj} to $\kappa_2(k, \sqrt{d_E} (1+2\Gamma_2))$ to get:
\begin{equation}
\kappa_2(k, \sqrt{d_E} (1+2\Gamma_2))
= \left[ 1 + 2\Gamma_2 \right]
\highlightmath{Gray}{k d_E^{1/2}}
\end{equation}
Next, we apply the definition of $\kappa_4(\cdots)$ in \Cref{lem:attn-op-kproj} to $\kappa_4(k, \sqrt{d_E}, \sqrt{d_E} (1+2\Gamma_2), \Gamma_1, \Gamma_2, \gamma, \demb)$ to get:
\begin{align}
& \kappa_4(k, \sqrt{d_E}, \sqrt{d_E} (1+2\Gamma_2), \Gamma_1, \Gamma_2, \gamma, \demb)
\\ & \quad
= 4 (\nicefrac{\gamma}{\sqrt{\demb}}) k d_E (1+2 \Gamma_2)^2 \sqrt{d_E} (2 + 2 \Gamma_2) \Gamma_1 \Gamma_2
\exp\left((\nicefrac{\gamma}{\sqrt{\demb}}) \Gamma_1^2 d_E(2+2\Gamma_2)^2 \right)
\\ & \quad
\leq
\frac{k d_E^{3/2}}{\sqrt{\demb}}
\left[ 8 \gamma (1 + 2 \Gamma_2)^2 (1 + \Gamma_2) \Gamma_1 \Gamma_2 \right]
\exp\left(
  \frac{d_E}{\sqrt{\demb}}
  \left[ \gamma \Gamma_1^2 (2+2\Gamma_2)^2 \right]
\right)
\\ & \quad
\leq
\left[ 8 \gamma (1 + 2 \Gamma_2)^3 \Gamma_1 \Gamma_2 \right]
\cdot
\highlightmath{Gray}{\frac{k d_E^{3/2}}{\sqrt{\demb}}}
\cdot
\exp\left(
  \highlightmath{Gray}{\frac{d_E}{\sqrt{\demb}}}
  \left[ \gamma \Gamma_1^2 (2+2\Gamma_2)^2 \right]
\right)
\end{align}
Applying definitions of $\kappa_2(\cdots)$ and $\kappa_5(\cdots)$ from \Cref{lem:attn-op-vproj} and \Cref{lem:attn-op-keys} respectively to $\kappa_5(k, n, \sqrt{d_E}, \sqrt{d_E} (1+2\Gamma_2), \Gamma_1, \Gamma_2, \gamma, \demb)\,\kappa_2(n, \sqrt{d_E})$, we get:
\begin{align}
&
\kappa_5(k, n, \sqrt{d_E}, \sqrt{d_E} (1+2\Gamma_2), \Gamma_1, \Gamma_2, \gamma, \demb)
\,\kappa_2(n, \sqrt{d_E})
\\ & \quad
= \frac{k \Gamma_2}{n} \left[
  1 + 4 (\nicefrac{\gamma}{\sqrt{\demb}}) \Gamma_1^2 d_E (1 + 2\Gamma_2) (2 + 2 \Gamma_2)
\right]
\exp\left((\nicefrac{\gamma}{\sqrt{\demb}}) \Gamma_1^2 d_E(2+2\Gamma_2)^2 \right)
n \sqrt{d_E}
\\ & \quad
\leq  \Gamma_2 \left[
  \highlightmath{Gray}{k d_E^{1/2}} + \highlightmath{Gray}{\frac{k d_E^{3/2}}{\sqrt{\demb}}} \left[ 8 \gamma \Gamma_1^2 (1 + 2\Gamma_2)^2 \right]
\right]
\exp\left(
  \highlightmath{Gray}{\frac{d_E}{\sqrt{\demb}}}
  \left[ \gamma \Gamma_1^2 (2+2\Gamma_2)^2 \right]
\right)
\end{align}
Applying definitions of $\kappa_4(\cdots)$ and $\kappa_5(\cdots)$ from \Cref{lem:attn-op-kproj} and \Cref{lem:attn-op-keys} respectively to $\kappa_5(k, n, \sqrt{d_E}, \sqrt{d_E} (1+2\Gamma_2), \Gamma_1, \Gamma_2, \gamma, \demb)\,\kappa_4(n, \sqrt{d}, \sqrt{d_E}, \Gamma_1, \Gamma_2, \gamma, \demb)$, we get:
\begin{align}
&
\kappa_5(k, n, \sqrt{d_E}, \sqrt{d_E} (1+2\Gamma_2), \Gamma_1, \Gamma_2, \gamma, \demb)
\,\kappa_4(n, \sqrt{d}, \sqrt{d_E}, \Gamma_1, \Gamma_2, \gamma, \demb)
\\ & \quad
= \frac{k \Gamma_2}{n} \left[
  1 + 4 (\nicefrac{\gamma}{\sqrt{\demb}}) \Gamma_1^2 d_E (1 + 2\Gamma_2) (2 + 2 \Gamma_2)
\right]
\exp\left((\nicefrac{\gamma}{\sqrt{\demb}}) \Gamma_1^2 d_E(2+2\Gamma_2)^2 \right)
\nonumber
\\ & \qquad \quad
\cdot
4 (\nicefrac{\gamma}{\sqrt{\demb}}) n d_E (\sqrt{d} + \sqrt{d_E}) \Gamma_1 \Gamma_2
\exp\left((\nicefrac{\gamma}{\sqrt{\demb}}) \Gamma_1^2 (\sqrt{d} + \sqrt{d_E})^2 \right)
\\ & \quad
\leq
k \Gamma_2 \left[ 1 + \frac{d_E}{\sqrt{\demb}} \left[ 8 \gamma \Gamma_1^2 (1 + 2 \Gamma_2)^2 \right] \right]
\cdot \frac{d_E^{3/2}}{\sqrt{\demb}} \left[ 8 \gamma \Gamma_1 \Gamma_2 \right]
\cdot \exp \left( (\nicefrac{\gamma}{\sqrt{\demb}}) \Gamma_1^2 \left[  d_E(2+2\Gamma_2)^2 + 4 d_E \right] \right)
\\ & \quad
=
\left[ 8 \gamma \Gamma_1 \Gamma_2^2 \right]
\cdot \left[
\highlightmath{Gray}{\frac{k d_E^{3/2}}{\sqrt{\demb}}} +
\highlightmath{Gray}{\frac{k d_E^{5/2}}{\demb}}
\left[ 8 \gamma \Gamma_1^2 (1 + 2 \Gamma_2)^2 \right]
\right]
\cdot \exp \left(
\highlightmath{Gray}{\frac{d_E}{\sqrt{\demb}}}
\left[ 4 \gamma \Gamma_1^2 \left[  (1+\Gamma_2)^2 + 1 \right] \right]
\right)
\end{align}
Note that in the above terms, there are two quantities that seem to repeatedly occur: (i)~$k \sqrt{d_E}$, and (ii)~$\frac{d_E}{\sqrt{\demb}}$. Considering $\Gamma_1, \Gamma_2$ to be constants that do not depend on the architectural choices or problem setting (such as the number of clusters, points, data dimensionality or token embedding scheme), we can write $\phi$ as following:
\begin{equation}
\phi = \max \left\{ \begin{array}{l}
c_1 \left( k \sqrt{d_E} \right) \\
c_2 \cdot \gamma \cdot \left( k \sqrt{d_E} \right) \cdot \left( \frac{d_E}{\sqrt{\demb}} \right)
\cdot \exp \left( c_3 \cdot \gamma \cdot \left( \frac{d_E}{\sqrt{\demb}} \right) \right)
\\
\Gamma_2 \left[ \left( k \sqrt{d_E} \right) + c_4 \cdot \gamma \cdot \left( k \sqrt{d_E} \right) \cdot
\left( \frac{d_E}{\sqrt{\demb}} \right)  \right]
\cdot \exp \left( c_3 \cdot \gamma \cdot  \left( \frac{d_E}{\sqrt{\demb}} \right) \right)
\\
\gamma \cdot \left( k \sqrt{d_E} \right) \cdot \left( \frac{d_E}{\sqrt{\demb}} \right) \left[
c_5 + c_6 \cdot \gamma \cdot \left( \frac{d_E}{\sqrt{\demb}} \right)
\right]
\cdot \exp \left( c_7 \cdot \gamma \cdot  \left( \frac{d_E}{\sqrt{\demb}} \right) \right)
\end{array} \right\}
\end{equation}
where
\begin{align}
c_1 & = (1 + 2 \Gamma_2) \\
c_2 & = 8  \Gamma_1 \Gamma_2 c_1^3 =  4 \Gamma_1 \Gamma_2 (1 + 2 \Gamma_2)^3 \\
c_3 & = 4 \Gamma_1^2 (1 + \Gamma_2)^2 \\
c_4 & = 8 \Gamma_1^2 c_1^2 = 4 \Gamma_1^2 (1 + 2 \Gamma_2)^2 \\
c_5 & = 8 \Gamma_1 \Gamma_2^2 \\
c_6 & = c_4 \cdot c_5 = 64 \Gamma_1^3 \Gamma_2^2 (1 + 2 \Gamma_2)^2 \\
c_7 & = c_3 + 4 \Gamma_1^2 = 4 \Gamma_1^2 ((1 + \Gamma_2)^2 + 1).
\end{align}
Thus, simplifying all, we can see that
\begin{equation}
\phi \sim O \left(
\max\{ 1, \gamma^2 \} \cdot
\left( k \sqrt{d_E} \right) \cdot
\max\left\{
 1, \left(\frac{d_E}{\sqrt{\demb}} \right)^2
\right\}
\cdot
\textsf{poly}\left( \exp \left( \gamma \frac{d_E}{\sqrt{\demb}} \right) \right)
\right),
\end{equation}
which gives us the result in \cref{eq:out-stab-params} in the statement of the theorem.
\end{proof}

If $d_E > \sqrt{\demb}$, which is the case with the proposed binary initial token embeddings, with $d_E = d + \ceil{\log_2 k}$ and $\demb = d + \ceil{\log_2 k}$, thus, $ \frac{d_E}{\sqrt{\demb}} = \sqrt{d + \ceil{\log_2 k } } > 1$, we can see that
\begin{equation}
\phi \sim O \left(
\max\{ 1, \gamma^2 \} \cdot \left( k \sqrt{d_E} \right) \cdot
\left(\frac{d_E}{\sqrt{\demb}} \right)^2
\cdot
\textsf{poly}\left( \exp \left( \gamma \frac{d_E}{\sqrt{\demb}} \right) \right)
\right).
\end{equation}
However $d_E < \sqrt{\demb}$ is possible with the \oh initial token embedding of \citet{clarkson2026transformer} as $d_E = d + 1$ and $\demb = d + k$, and $k \sim O((d+1)^2)$ implies that $d + 1 < \sqrt{d + k}$. In this case, we have
\begin{equation}
\phi \sim O \left(
\max\{ 1, \gamma^2 \} \cdot
\left( k \sqrt{d_E}  \right)
\cdot
\textsf{poly}\left( \exp \left( \gamma \frac{d_E}{\sqrt{\demb}} \right) \right)
\right).
\end{equation}
In all cases, for a fixed $k$, an architecture with a larger $\demb$ and smaller $d_E$ results in a smaller $\phi$.

Combining the above \Cref{thm:out-stab-params} with \Cref{lem:loss-to-output}, we can obtain the following main result regarding the Lipschitz-ness of the smoothed clustering loss function $\tilde{\Ls}_\Omega^\tau(\mX, \cmodel(\mX, \mC; \Theta))$ with respect to the learnable model parameters $\Theta$ (this is an extended version of \Cref{thm:loss-lip-short}):
\begin{theorem} \label{thm:loss-lip-full}
Assuming bounded points and initial cluster centers $(\mX, \mC) \sim \gD$ (assumption~\ref{asp:bnd-input}), and bounded learnable model parameters $\Theta, \Theta'$ (assumption~\ref{asp:bnd-params}), we can show that
\begin{equation} \label{eq:loss-lip-varphi}
\begin{split}
& \left|
\tilde{\Ls}_\Omega^\tau( \mX, \cmodel( \mX, \mC; \Theta) )
- \tilde{\Ls}_\Omega^\tau( \mX, \cmodel( \mX, \mC; \Theta' ) )
\right| \leq
\ \varphi \, \norm{ \Theta - \Theta' },
\\
& \quad \text{where} \ \varphi \sim O\left(
\frac{k n \omega_\tau}{\tau} \cdot
\max\{ 1, \gamma^2 \} \cdot
\sqrt{d_E} \cdot
\max\left\{
 1, \left(\frac{d_E}{\sqrt{\demb}} \right)^2
\right\}
\cdot
\textsf{poly} \left( \exp \left( \gamma \frac{d_E}{\sqrt{\demb}} \right) \right)
\right)
\end{split}
\end{equation}
\end{theorem}
\begin{proof}
We can get the above result by combining \cref{eq:out-stab-params} in \Cref{thm:out-stab-params} with \cref{eq:loss-to-output} in \Cref{lem:loss-to-output}:
\begin{equation}
\begin{split}
&
\left|
\tilde{\Ls}_\Omega^\tau ( \mX, \cmodel(\mX, \mC; \Theta) )
- \tilde{\Ls}_\Omega^\tau ( \mX, \cmodel(\mX, \mC; \Theta') )
\right| \leq
\varphi \cdot \norm{ \Theta - \Theta' }
\\
& \quad \text{where} \ \varphi \sim O\left(
\frac{n \omega_\tau}{\tau} \cdot
\max\{ 1, \gamma^2 \} \cdot
\left( k \sqrt{d_E} \right) \cdot
\max\left\{
 1, \left(\frac{d_E}{\sqrt{\demb}} \right)^2
\right\}
\cdot
\textsf{poly}\left( \exp \left( \gamma \frac{d_E}{\sqrt{\demb}} \right) \right)
\right),
\end{split}
\end{equation}
thus giving us the main result.
\end{proof}
This highlights the role of the different choices we make such as the choice of the $\min$-regularizer $\Omega$, the corresponding regularization penalty $\tau$, the attention inverse temperature $\gamma$, the number of points $n$ and the number of clusters $k$, the embedding dimension $\demb$ and the maximum initial cluster center embedding norm $d_E$.

%%%%%%%%%%%%%%%%%%%%%%%%%%%%%%%%%%%%%%%%%%%%%%%%%%%%%%%%%%%%

\end{document}